\documentclass{article}

\usepackage{microtype}
\usepackage{graphicx}
\usepackage{subcaption}
\usepackage{booktabs} 

\usepackage{hyperref}

\usepackage[accepted]{icml2026}

\usepackage{amsmath}
\usepackage{amssymb}
\usepackage{mathtools}
\usepackage{amsthm}

\usepackage[capitalize,noabbrev]{cleveref}

\theoremstyle{plain}
\newtheorem{theorem}{Theorem}[section]

\newtheorem{lemma}[theorem]{Lemma}

\theoremstyle{definition}
\newtheorem{definition}[theorem]{Definition}

\theoremstyle{remark}

\usepackage[textsize=tiny]{todonotes}

\icmltitlerunning{Generalizing Multi-Scale Time-Series Modeling with a Single Operator}

\usepackage{booktabs}
\usepackage{makecell}
\usepackage{graphicx}
\usepackage{multirow}
\usepackage{xcolor}
\usepackage{subcaption}
\usepackage{enumitem}

\usepackage{wrapfig}
\usepackage{tabularx}
\usepackage{lscape}
\usepackage{array,makecell}
\usepackage{pifont}

\usepackage{amsthm}
\usepackage{xspace}
\usepackage{amsmath}
\usepackage{cleveref}
\usepackage{amssymb}

\usepackage{placeins}
\usepackage{tikz}
\usetikzlibrary{arrows.meta,positioning,fit,calc}

\usepackage{amsmath,amsfonts,bm}

\def\eqref#1{equation~\ref{#1}}
\def\Eqref#1{Equation~\ref{#1}}

\def\1{\bm{1}}

\def\vb{{\bm{b}}}
\def\vc{{\bm{c}}}

\def\vs{{\bm{s}}}

\def\vu{{\bm{u}}}
\def\vv{{\bm{v}}}

\def\vx{{\bm{x}}}
\def\vy{{\bm{y}}}

\def\mH{{\bm{H}}}
\def\mI{{\bm{I}}}

\def\mK{{\bm{K}}}

\def\mW{{\bm{W}}}
\def\mX{{\bm{X}}}

\DeclareMathAlphabet{\mathsfit}{\encodingdefault}{\sfdefault}{m}{sl}
\SetMathAlphabet{\mathsfit}{bold}{\encodingdefault}{\sfdefault}{bx}{n}

\def\gF{{\mathcal{F}}}

\def\gS{{\mathcal{S}}}

\def\gX{{\mathcal{X}}}
\def\gY{{\mathcal{Y}}}

\def\sR{{\mathbb{R}}}

\def\sZ{{\mathbb{Z}}}

\newcommand{\R}{\mathbb{R}}

\newcommand{\mypar}[1]{\textbf{#1}\hspace{4pt}}

\newcommand{\circone}{\ding{172}\xspace}
\newcommand{\circtwo}{\ding{173}\xspace}
\newcommand{\circthree}{\ding{174}\xspace}
\newcommand{\circfour}{\ding{175}\xspace}
\newcommand{\circfive}{\ding{176}\xspace}
\newcommand{\circsix}{\ding{177}\xspace}

\newcommand{\RplusM}{\mathbb{R}_{\smash{+}}^{\smash{M}}}
\newcommand{\RplusL}{\mathbb{R}_{\smash{+}}^{\smash{L}}}

\newcommand{\method}{\textsc{Sigma}\xspace}
\newcommand{\methodlong}{\underline{Si}ngle \underline{G}eneralized \underline{M}ulti-scale \underline{A}rchitecture\xspace}

\newcommand{\best}[1]{\textbf{\textcolor{red}{#1}}}
\newcommand{\second}[1]{\textcolor{blue}{\underline{#1}}}

\begin{document}

\twocolumn[

  \icmltitle{Generalizing Multi-Scale Time-Series Modeling with a Single Operator}



  \icmlsetsymbol{equal}{*}
\begin{icmlauthorlist}
\icmlauthor{Cheonwoo Lee}{kaistee}
\icmlauthor{Dooho Lee}{kaistee}
\icmlauthor{Doyun Choi}{snucs}
\icmlauthor{Jaemin Yoo}{snucs}
\end{icmlauthorlist}

\icmlaffiliation{kaistee}{School of Electrical Engineering, KAIST, Daejeon, Republic of Korea}
\icmlaffiliation{snucs}{Department of Computer Science and Engineering, Seoul National University, Seoul, Republic of Korea}

\icmlcorrespondingauthor{Jaemin Yoo}{jaeminyoo@snu.ac.kr}

  \icmlkeywords{
Time-Series Forecasting,
Multi-Scale Modeling,
Machine Learning
}

  \vskip 0.3in
]



\printAffiliationsAndNotice{}  

\begin{abstract}

Multi-scale modeling has emerged as an effective design principle for time-series forecasting by capturing temporal dynamics at multiple resolutions.
As no principled foundation has been established in the literature, we unify existing scaling methods into a \emph{scaling operator family}, revealing a fundamental limitation of existing approaches: reliance on fixed and discrete scaling.
To address this limitation, 
we propose \method (\methodlong), which enables 
distance-aware scaling via the learnable discrete Gaussian (LDG) kernel grounded in scale-space theory.
We evaluate \method comprehensively on long- and short-term forecasting benchmarks against state-of-the-art multi-scale baselines.
\method outperforms all competitors on both tasks, achieving the best performance in 13 out of 16 long-term evaluation settings.
Beyond accuracy, \method significantly improves training speed by up to 5.3 times and reduces memory consumption by up to 3.8 times over the strongest competitors.
Code is available at \url{https://github.com/cheonwoolee/SiGMA}.
    
\end{abstract}

\section{Introduction}
\label{sec:intro}

Time-series forecasting is widely applied in diverse domains, including energy systems \citep{deb2017review}, weather prediction \citep{dimri2020time}, and traffic management \citep{tedjopurnomo2020survey}. 
Despite its importance, accurate forecasting is challenging due to the complex temporal structures of real-world time series \citep{kolambe2024forecasting}.
Inspired by the success of deep learning, recent work has introduced neural forecasting methods, including MLP-based methods \citep{ekambaram2023tsmixer, zeng2023transformers, yi2023frequency, challu2023nhits} and Transformer-based methods \citep{nie2023a, liu2024itransformer, shi2025timemoe}.

Within this scope, \emph{multi-scale modeling} has emerged as a powerful design principle.
By constructing temporal dynamics at multiple resolutions, it explicitly models cross-scale interactions to enhance predictive performance.
Specifically, coarse-grained scaling captures long-term dynamics, whereas fine-grained scaling preserves high-frequency fluctuations, yielding a more comprehensive representation \citep{chen2023multi}.
Building upon this principle, recent work has introduced diverse strategies from simple multi-rate sampling to hierarchical and frequency-aware mechanisms \citep{zhao2024himtm, wang2024timemixer,wang2025timemixer, naghashi2025multiscale, hu2025adaptive,yang2025multi}.

\begin{figure*}[t]
\centering
\includegraphics[width=\textwidth,
alt={Figure 1 compares three common scaling operators for time-series data: mean pooling, moving average, and subsampling. Each operator is applied to the same original periodic signal using scale values s=1, s=2, and s=4. As the scale increases, the transformed sequences become coarser and lose fine-grained temporal variation. The figure illustrates that these operators apply a fixed discrete scale uniformly across all timesteps, which can misalign the resulting representation with local temporal structures such as dominant periods or decay rates.}]{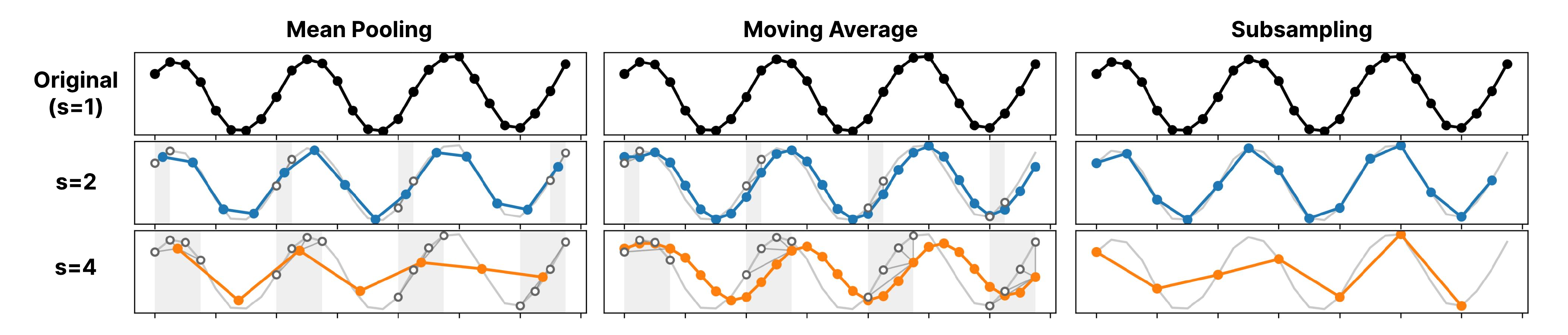}
\caption{Examples of popular scaling operators used in existing methods. 
Each operator applies a discrete scaling parameter $s$ uniformly across all timesteps to transform the input into coarser representations, often creating an abstraction that is mismatched with the dominant periods or decay rates of the time series.
Our goal is to design a learnable, dynamic scaling through a generalized framework.
}
\label{fig:scale}
\vspace{-7pt}
\end{figure*}

We visualize popular scaling operators in Figure \ref{fig:scale}.
Each scaling operator takes an integer scale parameter $s$ to transform the given sequence and applies the same scale uniformly across all timesteps within a single transformation.
This reveals a key structural limitation: scale is discretized and shared across all timesteps.
In practice, characteristic time scales, such as dominant periods or decay rates, are real-valued and may vary over time.
As a result, this design restricts the ability to represent temporal dynamics that vary smoothly across resolutions and introduces implicit boundaries between representations at different scales.
Since most existing approaches rely on these scaling operators \citep{wang2023micn,wang2024timemixer,hu2025adaptive}, such limitations are inherent to their multi-scale constructions.

To make these limitations explicit and to address them systematically, we first define a \emph{family of scaling operators} that unifies the scaling methods studied in the literature (see Table \ref{tab:scale-operators}). Within this framework, we define two fundamental requirements, \emph{non-expansiveness} and \emph{energy reduction}, that any valid scaling operator must satisfy. We further show that trivial operators on sequences, such as constant mappings or permutations, fail to meet these criteria.

Building on this unified view, we propose \method (\methodlong), a principled framework for multi-scale modeling.
Specifically, we introduce a novel concept of a \emph{generalized scaling operator family}, which generalizes existing scaling operators by allowing continuous scale parameters while preserving their structural properties.
\method adopts a learnable discrete Gaussian (LDG) kernel grounded in scale-space theory as an instance of such a generalized scaling operator family.
As a result, it allows for multiple, continuous, and learnable scale parameters within a single operator, preserving a coherent cross-scale structure directly at the representation level.

We evaluate \method extensively on long- and short-term forecasting benchmarks and observe consistent improvements over state-of-the-art multi-scale models.
In particular, on long-term tasks, \method achieves the best performance in 13 out of 16
evaluation settings with up to 7.8\% lower error than the strongest baselines.
Beyond accuracy, \method delivers clear efficiency gains, reducing memory usage by up to 3.8$\times$ and increasing training speed by 5.3$\times$.
We further support these findings with comprehensive ablation and case studies, confirming that \method provides a principled and efficient framework for multi-scale time-series modeling.

Our contributions are summarized as follows:
\begin{itemize}[topsep=0pt, itemsep=0pt,leftmargin=10pt]
        \item \textbf{Unified Framework:} We introduce a novel concept of a scaling operator family, which provides a principled foundation for multi-scale time-series modeling and clarifies the structural limitations of existing approaches.

    \item \textbf{Generalized Architecture:} 
    We propose \method, a principled architecture for multi-scale modeling. 
    It enables 
    distance-aware scaling through the learnable discrete Gaussian 
    kernel with dynamic scaling parameters.
    
    \item \textbf{Empirical Validation:} We show that \method achieves state-of-the-art performance on extensive long- and short-term forecasting benchmarks while substantially reducing computational complexity compared to prior multi-scale models.
    We further conduct comprehensive ablation and case studies to provide deeper insights on \method.
\end{itemize}
\begin{table*}[t]
\centering
\caption{Representative scaling operator families with associated scale parameters $s$.
Each family simplifies the sequence in a structured manner as the scale parameter $s$ increases, where $k \in [1, L_s]$ is the element index and $L_s$ denotes the length of the output sequence.
For wavelet decomposition, $f(\vx|s)$ yields the approximation coefficients at level $s$ using the scaling function $\phi$.}

\label{tab:scale-operators}

\begin{tabular}{lllll}
\toprule
Operator $f$ & Equation & Scale parameter $s$ & $L_s$ & Usage\\
\midrule
Average pooling & $[f(\vx|s)]_k = (1/s)\sum_{i=0}^{s-1} x_{ks+i}$ & Window length & $P/s$ &\citet{wang2024timemixer} \\
Max pooling & $[f(\vx|s)]_k = \max_{0 \le i < s} x_{ks+i}$ & Window length & $P/s$ & \citet{challu2023nhits}\\
Moving average & $[f(\vx|s)]_k = (1/s)\sum_{i=0}^{s-1} x_{k-i}$ & Window length & $P$ & \citet{wang2023micn} \\
Subsampling & $[f(\vx|s)]_k = x_{ks}$ & Stride & $P/s$ & \citet{liu2022scinet} \\
Segmentation & $[f(\vx|s)]_k = x_k$ & Length divisor & $P/s$ & \citet{wu2023timesnet} \\
Wavelet decomposition & $[f(\vx|s)]_k=\langle \vx, \phi_{s,k}\rangle$ & Decomposition level & $P/2^s$ & \citet{murad2025wpmixer} \\
\bottomrule
\end{tabular}
\vspace{-7pt}
\end{table*}

\section{Problem Definition and Related Work}
\label{sec:related-work}

\mypar{Problem Definition}
Given a time series, let $\vx \in \mathbb{R}^{L}$ denote an input sequence of length $L$, and let $\vy \in \mathbb{R}^{T}$ denote the corresponding future trajectory of length $T$. We denote the input space as $\gX \subset \mathbb{R}^{L}$ and the target space as $\gY \subset \mathbb{R}^{T}$.
Our goal is to learn a forecasting model $h:\gX \to \gY$ that maps a past window $\vx$ to a $T$-step forecast of its future $\vy$. 

 \mypar{Assumption on Data}
Throughout this work, we focus on non-trivial datasets whose forecastability $\phi$ is strictly less than $1$. The forecastability of a time series is defined as one minus the normalized spectral entropy of its Fourier decomposition~\citep{goerg2013forecastable}, formally given as
\begin{align}
\label{eqn:forecastability}
   \phi(\vx) &= 1 - \frac{H(\vx)}{\log(2\pi)} \in [0,1],\\ 
   H(\vx) &= - \int_{-\pi}^{\pi} p_{\vx}(\omega)\,\log p_{\vx}(\omega)\,d\omega,
\end{align} 
where $p_{\vx}(\omega)$ denotes the normalized spectral density of $\vx$ on $[-\pi,\pi]$, characterizing how concentrated the spectrum of $\vx$ is.
Most real-world datasets, including every benchmark used in our experiments, satisfy $\phi < 1$ (see Appendix~\ref{sec:exp-detail}).
This serves as a mild theoretical assumption and does not restrict the practical applicability of our framework.
In general, 
larger forecastability $\phi$ indicates a more concentrated spectrum and thus a stronger predictable structure~\citep{wang2024timemixer,wang2025timemixer}, which is desirable for forecasting.

\mypar{Time Series Forecasting}
The field of time series forecasting has seen a significant evolution from traditional statistical methods to modern deep learning architectures. While statistical models like ARIMA \citep{box2015time} and state-space models \citep{durbin2012time} are still used, they are often limited in their ability to capture complex nonlinear and long-range dependencies \citep{zhou2021informer}.
Among deep forecasting models that learn representations directly from raw sequences, multi-layer perceptron architectures have recently been recognized for their effectiveness and efficiency \citep{ekambaram2023tsmixer, zeng2023transformers, yi2023frequency, challu2023nhits}. Concurrently, Transformer-based approaches have demonstrated strong empirical results, particularly for long-horizon forecasting tasks \citep{nie2023a, liu2024itransformer, shi2025timemoe}.

\mypar{Multi-Scale Design in Time Series}
Time series often exhibit complex patterns across diverse temporal scales, from rapid fluctuations to long-term trends. To capture such heterogeneity, recent work has incorporated explicit multi-scale mechanisms to disentangle and integrate temporal dependencies. These methods include hierarchical decompositions that construct multi-resolution representations via downsampling \citep{liu2022pyraformer,challu2023nhits}, frequency-domain or wavelet-based analyses that emphasize periodic structure through spectral transforms \citep{wu2023timesnet,cai2024msgnet,murad2025wpmixer}, and aggregation schemes that fuse representations across scales \citep{wang2023micn,wang2024timemixer,wang2025timemixer}.
However, existing multi-scale methods rely on fixed and discrete scaling strategies to generate multi-scale representations, which limits adaptability to diverse and continuously varying temporal patterns.
This highlights the need for more flexible and principled approaches to multi-scale modeling.

\section{Foundations of Scaling in Time Series}
\label{sec:scale}

In time series analysis, \emph{scaling} denotes transformations that change the temporal resolution of time series.
As illustrated in \Cref{fig:scale}, existing methods employ common operations such as downsampling and moving averages with discrete scale parameters to produce representations at different resolutions.
However, these operations are typically introduced heuristically, offering limited insight into why particular scales are appropriate or how they relate to one another.

To address this gap, we introduce a rigorous definition of \emph{scaling operator families}, providing a unified mathematical framework for analyzing scaling operations in time series. To the best of our knowledge, this is the first work to establish such a foundation.
All proofs are given in Appendix~\ref{sec:proofs}.

\begin{definition}[Scaling operator family]
\label{def:scale-op}
    A set of transformations $\gF = \{ f(\vx|s) \mid s\in\sZ_+ \}$ such that $f(\cdot|s):\gX \to\gX_s\subset \sR^{L_s\vphantom{y}}$, conditioned on a positive integer $s \in \sZ_+$, is a scaling operator family if it satisfies
    \begin{itemize}[nosep, leftmargin=20pt]
        
        \item \emph{Non-expansiveness}: For  $s_i\in\sZ_+$ and $\vx,\vx'\in\gX$, $\|f(\vx|s_i)-f(\vx'|s_i)\|_2 \le \|\vx-\vx'\|_2.$

        \item \emph{Energy reduction}: For any $s_i, s_j\in\sZ_+$ such that $s_i=m\cdot s_j$ for some $m>1$, $\|f(\vx|s_i)\|_2 \le \|f(\vx|s_j)\|_2$ for all $\vx\in\gX$
        and $\|f(\vx'|s_i)\|_2 < \|f(\vx'|s_j)\|_2$ for some $\vx'\in\gX$.

    \end{itemize}
    The scale parameter $s$ determines the level of scaling and also the output dimension $L_s$.
\end{definition}

\begin{theorem}
\label{thm:scale-op}
    Popular sequence operations used in recent work \citep{liu2022scinet, challu2023nhits, wu2023timesnet, wang2023micn, wang2024timemixer, murad2025wpmixer}, such as max- or mean-pooling~\citep{zheng2014time},
    subsampling~\citep{hannan2009multiple}, 
    moving averages~\citep{box2015time}, segmentation~\citep{keogh2004segmenting}
    and wavelet decompositions~\citep{percival2000wavelet}, are scaling operator families.
\end{theorem}

\begin{figure}[t]
    \centering
    \includegraphics[
        width=0.9\linewidth,
        alt={Figure 2 shows two line plots evaluating the six scaling operator families on the Traffic dataset: moving average, average pooling, max pooling, subsampling, segmentation, and wavelet decomposition. The left plot measures non-expansiveness by showing the average distance between transformed sequences across scales s = 1, 2, 4, 8, and 16. For all six operators, this distance decreases as the scale increases. The right plot measures energy reduction by showing the average squared norm of each transformed sequence across the same scales. For all six operators, the energy also decreases as the scale increases. Overall, the figure empirically confirms that the six operator families satisfy the non-expansiveness and energy reduction properties defined in Definition 3.1.}
    ]{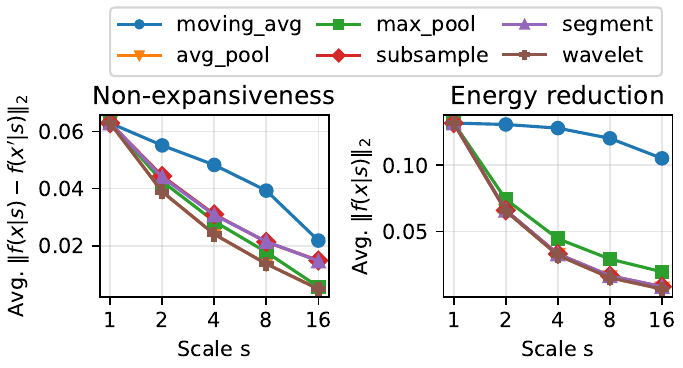}

    \vspace{-5pt}

    \caption{The non-expansiveness and energy reduction of the six scaling operator families on the Traffic dataset.
    {All these operators satisfy the two essential properties stated in Definition~\ref{def:scale-op}.}}

    \vspace{-10pt}
    \label{fig:traffic_scaling_results}
\end{figure}

Definition~\ref{def:scale-op} formalizes the expected behavior of scaling.
Scaling operators are  
\emph{non-expansive}, as they are designed to capture diverse properties from a sequence without adding external information.
At the same time, the \emph{energy reduction} property formalizes the intuition that time series with coarser scales should be simpler than finer ones along multiplicative chains of scale parameters: as the scale parameter $s$ increases, the result exhibits less overall energy, with small fluctuations being smoothed out.

\Cref{tab:scale-operators} shows the operator families stated in \Cref{thm:scale-op} with their associated scale parameters.
These operators can be grouped into two categories based on how they transform the given sequence: \emph{(i) smoothing operators}, which attenuate fine-grained variability by summarizing adjacent observations (e.g., pooling and moving average), and \emph{(ii) structural operators}, which progressively decompose or resample the signals by enforcing a specific structure (e.g., subsampling and wavelet decomposition).
Previous works have typically used the term scaling \emph{only for one group}, but our Definition~\ref{def:scale-op} suggests that both categories of operators can be understood through the same lens in terms of the non-expansiveness and energy reduction properties.

To empirically verify that the operator families in \Cref{tab:scale-operators} satisfy both properties stated in Definition~\ref{def:scale-op}, 
we evaluate their average pairwise contraction and induced average energy on the Traffic dataset. 
As shown in \Cref{fig:traffic_scaling_results}, all operator families exhibit strictly smaller output differences than input differences, confirming non-expansiveness, 
and their energies decrease monotonically over multiplicative scales, confirming energy reduction.  
Additional experiments on other datasets are provided in Appendix~\ref{sec:scale-op-prop}.

\begin{theorem}
\label{thm:not-scale-op}
    Trivial sequence operations \citep{horn2012matrix} such as constant mappings, permutations, additive shifts, scalar multiplications, and general linear transformations do not form scaling operator families.
\end{theorem}

\Cref{thm:not-scale-op} highlights that not all parameterized transformations qualify as scaling operator families. 
For example, while a scalar multiplication can change the scale of each observation, it does not guarantee energy reduction and thus falls outside our definition.
Putting \Cref{thm:scale-op} and  \Cref{thm:not-scale-op} together, we claim that our definition of scaling operator families is designed carefully to include only the ones that have been considered ``scaling'' in the literature, providing more interpretable and theoretically grounded approaches to multi-scale time series modeling.

\section{\method: Generalized Multi-Scale Modeling}
\label{sec:generalized-scale-genearator}

We introduce \method, a principled approach to multi-scale time series modeling.
We first generalize the family of scaling operators to support multiple continuous scale parameters (Sec. \ref{ssec:generalized}).
Then, we introduce the learnable discrete Gaussian (LDG) kernel grounded in the scale-space theory as an effective instantiation of the generalized scaling operator (Sec. \ref{ssec:ldg}).
Lastly, we introduce our implementation design based on these theoretical foundations (Sec. \ref{ssec:impl}).
Detailed proofs for all theorems are provided in Appendix~\ref{sec:proofs}.

\subsection{Generalized Scaling Operator Family}
\label{ssec:generalized}

The unified foundation for scaling operator families reveals their essential restriction: they are parameterized by a single, discrete scale parameter $s \in \sZ_+$ that is fixed a priori and applied uniformly across the sequence.
Such a restriction often makes scaling operators misaligned with real-world data, where characteristic time scales vary across datasets and may evolve within a single sequence.
To address this limitation, we generalize Definition \ref{def:scale-op} to allow multiple continuous scale parameters $\vs \in \RplusM$ for each $f$.

\begin{definition}[Generalized scaling operator family]
\label{def:ext-scale-op}
A set of transformations $\gF = \{ f(\vx|\vs) \mid \vs\in\RplusM\} $ such that $f(\cdot|\vs):\gX \to\gX_\vs\subset \sR^{L_\vs}$ is a generalized scaling operator family if it satisfies
\begin{itemize}[nosep, leftmargin=20pt]
    \item \emph{Consistency}: For $s\in \sZ_+$, $\{f(\vx|s\mathbf{1})\}$ is a scaling operator family.
    \item \emph{Differentiability}: For any $\vx\in\gX$, $f(\vx|\vs)$ is continuously differentiable with respect to $\vs$.
\end{itemize}
The output dimension $L_\vs$ depends on the set of scale parameters $\vs=(s_1,\cdots, s_M)$.
\end{definition}

The consistency condition ensures that Definition \ref{def:ext-scale-op} strictly generalizes discrete scaling operator families.
The differentiability condition guarantees that changes in scale parameters induce smooth changes in the transformed signal.

\begin{theorem}
\label{thm:quantization-gap}
    Let $\gS_d\subset \sZ^+$ be the discrete set of scale parameters and $\gS_c\subset \sR_+^M$ be a compact set containing $\gS_d$. For a given scaling operator family $f$ and forecastability measure $\phi$, define $g_\vx(\vs)=\phi(f(\vx|\vs))$,
    \begin{align*}
        \Phi_c(\vx)
        =\max_{\vs\in\gS_c} g_\vx(\vs),\quad and \quad \Phi_d(\vx)
        =\max_{s\in\gS_d} g_\vx(s\mathbf{1}).
    \end{align*}
    Let $\Pi_\mathbf{1}=\{\vu\in\mathbb{R}^M: \vu^\top\mathbf{1}=0,\|\vu\|_2=1\}$. Then,
    \begin{align*}
        \Phi_c(\vx)-\Phi_d(\vx)\ge C(\vx)\sup_{\vu\in\Pi_\mathbf{1}} \langle \nabla_\vs g_\vx(\vs)|_{\vs=t\mathbf{1}},\vu\rangle^2,
    \end{align*}
    for all $\vx\in\gX$, with some constant $C(\vx)>0$ and $t\in\mathbb{R}$.
    
\end{theorem}

Theorem \ref{thm:quantization-gap} formalizes a quantization gap in multi-scale forecasting. 
The set $\Pi_\mathbf{1}$ consists of directions that lie outside the representational capacity of a discrete scaling operator family. 
If the forecastability function $g$ increases along any direction
in $\Pi_\mathbf{1}$, then this improvement cannot be captured by any $s\in\gS_d$. In such cases, a generalized scaling operator family yields strictly larger forecastability.
Intuitively, restricting scale parameters to discrete values introduces an irreducible mismatch between the operator's effective scale and the signal's intrinsic temporal structure;
discrete scaling operator families incur an irreducible loss in forecastability relative to generalized families.

\begin{figure}[t]
\centering
\includegraphics[
width=\linewidth,
alt={Figure 3 empirically validates Theorem 4.2 on the Traffic dataset using two plots. The left plot shows forecastability as a function of the scale parameter s. Forecastability increases from small scales, reaches its maximum at the non-discrete scale s = 12.4, and then slightly decreases, showing that the best scale is not contained in a fixed discrete scale set. The right plot compares the expressivity gap with the theoretical lower bound for individual samples. Most points lie above the diagonal reference line, indicating that the observed expressivity gap exceeds the lower bound predicted by Theorem 4.2.}
]{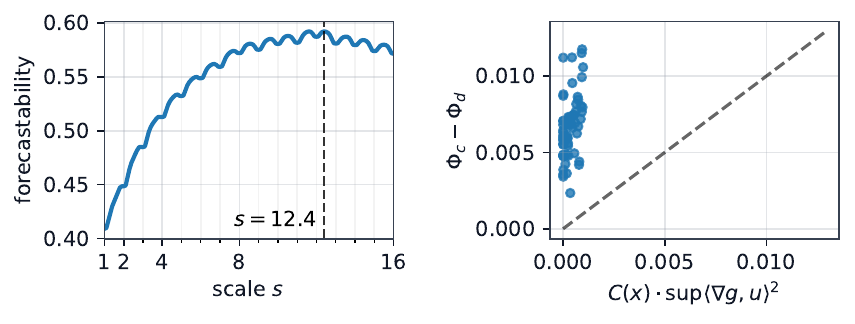}
\caption{
Empirical validation of Theorem~\ref{thm:quantization-gap} on the Traffic dataset. Forecastability is maximized at a non-discrete scale, and the expressivity gap $\Phi_c - \Phi_d$ exceeds the theoretical lower bound.
}
\label{fig:theorem42-traffic}
\vspace{-7pt}
\end{figure}

To empirically validate Theorem~\ref{thm:quantization-gap}, 
we evaluate the average forecastability and quantization gap on the Traffic dataset. 
As shown in Figure~\ref{fig:theorem42-traffic}, the left panel indicates that forecastability is maximized at a non-discrete scale, implying that the maximum forecastability over the continuous scale set, $\Phi_c$, exceeds that over the discrete scale set, $\Phi_d$. Furthermore, the right panel shows that the quantization gap, $\Phi_c-\Phi_d$, consistently exceeds the theoretical lower bound across all samples $\vx$, thereby providing empirical support for Theorem~\ref{thm:quantization-gap}. Additional results on the other datasets are provided in Appendix~\ref{sec:thm-42}.

As Definition \ref{def:ext-scale-op} is not intuitive at first glance, it may be nontrivial to design an operator that satisfies both properties.
We thus provide a guideline for designing generalized scaling operator families in Appendix~\ref{sec:design-guide-ext} (e.g., generalized mean-pooling). The key idea is to obtain generalized operators through a smooth expansion of discrete scaling operators with respect to their scale parameters. Concretely, this is realized by interpolating discrete operators through smooth functions of scale, ensuring exact recovery at integer scales and preserving the differentiability condition. 

\subsection{Learnable Discrete Gaussian as a Scaling Operator}
\label{ssec:ldg}

Among various constructions that can realize a generalized scaling operator family, we propose adopting the \emph{learnable discrete Gaussian} (LDG) kernel $k$~\citep{lindeberg2002scale}:

\begin{align}
    k(\vx | \vs) = \mK(\vs)\vx,
\qquad 
[\mK(\vs)]_{i,j} = e^{-s_d} I_{d}(s_d),
\label{eq:ldg}
\end{align}

where $\mK(\vs) \in \sR^{L \times L}$ is the kernel matrix with $L$ being the length of the sequence $\vx$, $[\cdot]_{i,j}$ denotes the $(i, j)$-th element of a matrix, 
$I_n(\cdot)$ is the modified Bessel function of the first kind of integer order $n$,
and $d = |i-j|$ denotes the temporal distance between positions $i$ and $j$.

Intuitively, the LDG kernel in \Eqref{eq:ldg} performs a smooth, locality-aware aggregation of the input sequence.
Each output element is computed as a weighted average of its neighboring elements with symmetrically decaying weights.
The parameter 
$s_d$ controls how broadly information is aggregated 
at distance $d$:
small values concentrate
the weights near the center, preserving local variations, while larger values spread the weights over a wider
temporal neighborhood, producing smoother representations.

Unlike classical discrete Gaussian filtering, we model the scale parameters $\vs$ as learnable and position-dependent. As a result, $k$ applies dynamic scaling to each element of $\vx$, allowing the operator to adaptively control the extent of temporal aggregation at each time step.
Since the LDG kernel can apply heterogeneous smoothing across the sequence, it avoids the structural limitations of discrete designs. These scale parameters are optimized from data via gradient-based learning with respect to the downstream forecasting loss.

\begin{figure}[t]
  \centering
  \includegraphics[
  width=0.8\linewidth,
  alt={Figure 4 illustrates the learnable discrete Gaussian kernel applied to a time series. The top panel shows an input sequence with a local smoothing window whose weights are largest near the target timestep and decay symmetrically with temporal distance. The bottom panel shows the resulting smoothed sequence, where local fluctuations are reduced while the broader periodic pattern is preserved. The shaded regions indicate that the kernel has a continuous, distance-aware receptive field that extends beyond a fixed local window, enabling long-range dependency modeling.}
  ]{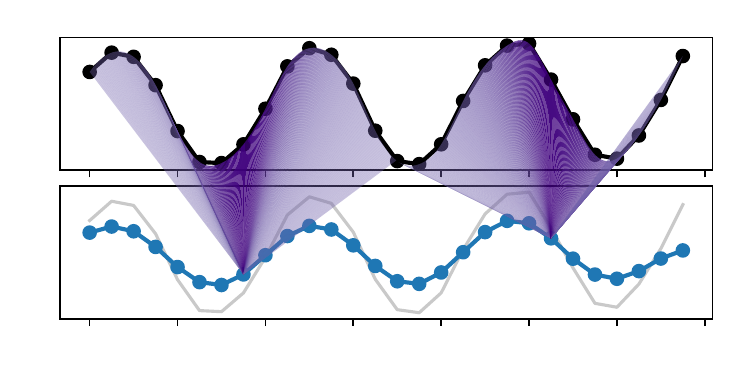}
  \vspace{-5pt}
  \caption{The LDG kernel performs smoothing with continuous 
  distance-aware scales.
  It induces an effectively unbounded receptive field, enabling long-range dependency modeling.}
  \vspace{-10pt}
  \label{fig:kernel}
\end{figure}

Figure~\ref{fig:kernel} illustrates how the LDG kernel actually works on a time series; it yields smooth, scale-controlled transformations at each time step with learnable scale parameters.
At the same time, it provides a symmetric and effectively unbounded receptive field for capturing long-range dependencies, which is its unique characteristic.

\begin{theorem}
\label{thm:gaussian-scale-op}
    The LDG kernel family, $\{ k(\cdot|\vs) \mid \vs \in \RplusM \}$, is a generalized scaling operator family.
\end{theorem}

\begin{theorem}
\label{thm:discrete-gaussian-optimality}
The LDG kernel is the unique generalized scaling operator family of symmetric kernels that satisfies the discrete scale-space axioms~\citep{lindeberg2002scale}.
\end{theorem}

Theorem~\ref{thm:gaussian-scale-op} establishes that this kernel constitutes a valid generalized scaling operator family, confirming its suitability for principled scaling.
Moreover, this Gaussian instantiation is the unique operator determined by the discrete scale-space axioms as shown in Theorem \ref{thm:discrete-gaussian-optimality}.
This finding aligns our method with the scale-space theory, a principled framework for multi-resolution signal analysis originating in computer vision \citep{witkin1987scale,lindeberg2013scale}. 
We empirically examine alternative kernel choices in \Cref{ssec:deep-analysis}, showing that replacing the LDG kernel with other smoothing or learnable operators consistently degrades performance.

\subsection{Lightweight Forecasting Module}
\label{ssec:impl}

Many existing multi-scale forecasting models integrate information across scales after applying the scaling operator, such as weighted summation or aggregation modules \citep{wu2023timesnet,cai2024msgnet}, or cross-scale mixing architectures that dynamically couple representations at different temporal resolutions \citep{hu2025adaptive,wang2024timemixer,wang2025timemixer}. While these designs improve expressivity under discrete scaling operator families, they typically incur additional parameters and computational overhead.

An essential strength of our learnable scaling operation is that additional multi-scale interaction is not required in the forecasting module, since the LDG kernel already preserves coherent cross-scale structure at the representation level.
Accordingly, prediction can be performed using an arbitrary nonlinear mapping. We therefore adopt a lightweight multi-layer perceptron (MLP), which is parameter-efficient, fully parallel, and empirically sufficient for real-world data, as confirmed by the ablation results in \Cref{ssec:deep-analysis}.

\begin{table*}[t]
\centering
\caption{
Long-term forecasting results across eight datasets with horizons
$T \in \{96,192,336,720\}$ and input length fixed at 96. 
Results are averaged across all horizons, while the full results are provided in Appendix~\ref{sec:full}.
\method achieves the smallest forecasting errors in 13 out of 16 evaluation settings and the second-best in 2 cases. 
}
\label{tab:long-term}
\small
\setlength{\tabcolsep}{3pt}
\renewcommand{\arraystretch}{1.0}
\resizebox{\textwidth}{!}{
\begin{tabular}{l|cc|cc|cc|cc|cc|cc|cc|cc|cc}
\toprule
\multirow{2}{*}{Models} &
\multicolumn{2}{c|}{\begin{tabular}[c]{@{}c@{}}\method\\ (Ours)\end{tabular}} 
 & \multicolumn{2}{c|}{\begin{tabular}[c]{@{}c@{}}AMD\\ \citeyearpar{hu2025adaptive}\end{tabular}} 
 & \multicolumn{2}{c|}{\begin{tabular}[c]{@{}c@{}}MultiPatch.\\ \citeyearpar{naghashi2025multiscale}\end{tabular}} 
 & \multicolumn{2}{c|}{\begin{tabular}[c]{@{}c@{}}WPMixer\\ \citeyearpar{murad2025wpmixer}\end{tabular}} 
 & \multicolumn{2}{c|}{\begin{tabular}[c]{@{}c@{}}TimeMixer\\ \citeyearpar{wang2024timemixer}\end{tabular}} 
 & \multicolumn{2}{c|}{\begin{tabular}[c]{@{}c@{}}MSGNet\\ \citeyearpar{cai2024msgnet}\end{tabular}} 
 & \multicolumn{2}{c|}{\begin{tabular}[c]{@{}c@{}}MICN\\ \citeyearpar{wang2023micn}\end{tabular}} 
 & \multicolumn{2}{c|}{\begin{tabular}[c]{@{}c@{}}TimesNet\\ \citeyearpar{wu2023timesnet}\end{tabular}} 
 & \multicolumn{2}{c}{\begin{tabular}[c]{@{}c@{}}Pyra.\\ \citeyearpar{liu2022pyraformer}\end{tabular}} \\
\cmidrule{2-19}
 
 & MSE & MAE & MSE & MAE & MSE & MAE & MSE & MAE
 & MSE & MAE & MSE & MAE & MSE & MAE & MSE & MAE & MSE & MAE \\
\toprule

{Weather} &
0.247 & \best{0.273} &
0.263 & 0.286 &
0.255 & 0.278 &
\best{0.245} & \best{0.273} &
\second{0.246} & \second{0.276} &
0.258 & 0.283 &
0.266 & 0.315 &
0.262 & 0.288 &
0.285 & 0.347 \\
\midrule

{Electricity} &
\best{0.175} & \best{0.269} &
0.208 & 0.289 &
0.198 & 0.282 &
0.194 & 0.282 &
\second{0.185} & \second{0.274} &
0.199 & 0.307 &
0.187 & 0.298 &
0.194 & 0.294 &
0.298 & 0.390 \\
\midrule

{Traffic} &
\best{0.458} & \best{0.302} &
0.546 & 0.344 &
\second{0.497} & 0.329 &
0.527 & 0.343 &
0.501 & \second{0.318} &
0.654 & 0.384 &
0.543 & 0.320 &
0.627 & 0.334 &
0.685 & 0.384 \\
\midrule

{Exchange} &
\second{0.353} & \best{0.400} &
0.358 & \second{0.401} &
0.387 & 0.419 &
0.376 & 0.408 &
0.384 & 0.414 &
0.431 & 0.446 &
\best{0.345} & 0.424 &
0.416 & 0.443 &
1.181 & 0.854 \\
\midrule

{ETTh1} &
\best{0.443} & \best{0.433} &
\second{0.447} & \second{0.434} &
\best{0.443} & 0.440 &
0.451 & 0.442 &
0.455 & 0.444 &
0.459 & 0.458 &
0.572 & 0.531 &
0.480 & 0.468 &
0.879 & 0.731 \\
\midrule

{ETTh2} &
\best{0.376} & \second{0.402} &
\best{0.376} & \best{0.400} &
\second{0.379} & 0.406 &
0.390 & 0.410 &
0.389 & 0.410 &
0.402 & 0.421 &
0.582 & 0.527 &
0.410 & 0.424 &
4.086 & 1.621 \\
\midrule

{ETTm1} &
\best{0.383} & \best{0.397} &
0.395 & 0.399 &
\second{0.385} & 0.400 &
\second{0.385} & \best{0.397} &
\second{0.385} & \second{0.398} &
0.401 & 0.412 &
0.399 & 0.427 &
0.410 & 0.415 &
0.737 & 0.614 \\
\midrule

{ETTm2} &
\best{0.276} & \best{0.322} &
0.285 & 0.328 &
0.283 & 0.328 &
\second{0.279} & \second{0.324} &
0.281 & 0.327 &
0.289 & 0.330 &
0.354 & 0.397 &
0.298 & 0.333 &
1.462 & 0.846 \\
\bottomrule
\end{tabular}
}
\vspace{-7pt}
\end{table*}

Formally, given an input sequence \(\vx \in \mathbb{R}^{L}\), we first produce $d$-dimensional embeddings \(\mX = \mathrm{Embed}(\vx) \in \mathbb{R}^{L \times d}\) as done in \citep{wu2021autoformer,wang2024timemixer,wang2025timemixer}. Based on these embeddings, our multi-scale forecasting model is succinctly represented by the following equation:
\begin{align}
\hat{\vy}
&= \mW_1 (\mathrm{MLP}(\mH) + \mH) \mW_2, \\
\mH &= \mK(\vs)\mX \mathbin\Vert (\mI-\mK(\vs))\mX \in \R^{2L \times d},
\end{align}
where $\mW_{1} \in \mathbb{R}^{T \times 2L}$ and $\mW_{2}\in\mathbb{R}^{d\times 1}$ are projection heads and $\vs\in\RplusL$ denotes learnable scale parameters.
These scale parameters are optimized as dataset-level parameters during training and kept fixed at inference.

One notable design decision is that we decompose the input into a smoothed component \(\mK(\vs)\mX\) and a residual component \((\mI-\mK(\vs))\mX\) analogous to the classical trend-seasonal decomposition \citep{cleveland1990stl}. To further stabilize optimization and preserve scale-specific information, we incorporate a skip connection that adds each scale’s input back to its transformed output~\citep{he2016deep}.
We also adopt channel independence, processing each variable separately, which makes the architecture naturally applicable to multivariate settings \citep{zeng2023transformers}.
We further apply reversible instance normalization to each time series variable to remove distribution shifts \citep{kim2022reversible}.

\begin{table*}[t]

\centering
\caption{Short-term forecasting results in the M4 benchmark dataset with various temporal granularities.
\method performs the best in 11 of the 15 cases.
This highlights its effectiveness in capturing multi-scale patterns across a diverse range of time series types.}
\setlength{\tabcolsep}{2pt}
\scriptsize
\renewcommand{\arraystretch}{0.95}
\newcolumntype{L}[1]{>{\raggedright\arraybackslash}p{#1}}
\newcolumntype{R}[1]{>{\raggedleft\arraybackslash}p{#1}}

\resizebox{\textwidth}{!}{%
\begin{tabular}{ll|r|r|r|r|r|r|r|r|r}

\midrule
 &
\multicolumn{1}{|c|}{\parbox{0.85cm}{\centering Metric}} &
{\parbox{1.25cm}{\centering \method\\(Ours)}} &
{\parbox{1.25cm}{\centering AMD\\\citeyearpar{hu2025adaptive}}} &
{\parbox{1.25cm}{\centering MultiPatch.\\\citeyearpar{naghashi2025multiscale}}} &
{\parbox{1.25cm}{\centering WPMixer\\\citeyearpar{murad2025wpmixer}}} &
{\parbox{1.25cm}{\centering TimeMixer\\\citeyearpar{wang2024timemixer}}} &
{\parbox{1.25cm}{\centering MSGNet\\\citeyearpar{cai2024msgnet}}} &
{\parbox{1.25cm}{\centering MICN\\\citeyearpar{wang2023micn}}} &
{\parbox{1.25cm}{\centering TimesNet\\\citeyearpar{wu2023timesnet}}} &
{\parbox{1.25cm}{\centering Pyra.\\\citeyearpar{liu2022pyraformer}}} \\

\specialrule{0.3pt}{1pt}{2pt}
\multirow{3}{*}{\rotatebox[origin=c]{90}{\scalebox{0.8}{Yearly}}}
& \multicolumn{1}{|l|}{SMAPE} & \second{13.314} & 13.447 & \best{13.296} & 13.632 & 13.326 & 13.354 & 14.580 & 13.482 & 14.987 \\
& \multicolumn{1}{|l|}{MASE}  & \best{2.989} & 3.022 & 3.009 & 3.075 & \second{3.002} & \best{2.989} & 3.382 & 3.056 & 3.361 \\
& \multicolumn{1}{|l|}{OWA}   & \best{0.783} & 0.792 & 0.785 & 0.804 & 0.785 & \second{0.784} & 0.871 & 0.797 & 0.881 \\
\specialrule{0.3pt}{1pt}{2pt}
\multirow{3}{*}{\rotatebox[origin=c]{90}{\scalebox{0.8}{Quarterly}}}
& \multicolumn{1}{|l|}{SMAPE} & \best{10.060} & 10.259 & 10.166 & 10.299 & 10.281 & 10.446 & 11.389 & \second{10.116} & 11.706 \\
& \multicolumn{1}{|l|}{MASE}  & \best{1.177} & 1.211  & \second{1.178} & 1.217  & 1.206  & 1.248  & 1.380  & 1.186  & 1.397 \\
& \multicolumn{1}{|l|}{OWA}   & \best{0.886}  & 0.907  & \second{0.892} & 0.911  & 0.907  & 0.929  & 1.020  & \second{0.892} & 1.041 \\
\specialrule{0.3pt}{1pt}{2pt}
\multirow{3}{*}{\rotatebox[origin=c]{90}{\scalebox{0.8}{Monthly}}}
& \multicolumn{1}{|l|}{SMAPE} & \best{12.750} & 12.898 & 12.810 & 12.945 & 12.984 & 12.970 & 13.797 & \second{12.775} & 14.444 \\
& \multicolumn{1}{|l|}{MASE}  & \best{0.936}  & 0.952  & \second{0.942} & 0.959  & 0.964  & 0.976  & 1.077  & 0.945  & 1.142  \\
& \multicolumn{1}{|l|}{OWA}   & \best{0.882}  & 0.895  & \second{0.887} & 0.900  & 0.903  & 0.908  & 0.985  & \second{0.887} & 1.038  \\
\specialrule{0.3pt}{1pt}{2pt}
\multirow{3}{*}{\rotatebox[origin=c]{90}{\scalebox{0.8}{Others}}}
 & \multicolumn{1}{|l|}{SMAPE} & 4.867  & \second{4.822}  & 4.849  & 4.925  & \best{4.739}  & 5.521  & 6.123  & 4.978  & 6.115  \\
& \multicolumn{1}{|l|}{MASE}  & 3.316  & \second{3.245}  & 3.271  & 3.259  & \best{3.234}  & 3.829  & 4.196  & 3.265  & 4.156  \\
& \multicolumn{1}{|l|}{OWA}   & 1.037  & \second{1.021}  & 1.028  & 1.028  & \best{1.014}  & 1.174  & 1.300  & 1.048  & 1.282  \\
\specialrule{0.3pt}{1pt}{2pt}
\multirow{3}{*}{\rotatebox[origin=c]{90}{\scalebox{0.8}{\shortstack{Weighted\\[-0.5ex]Average}}}}
& \multicolumn{1}{|l|}{SMAPE} & \best{11.840} & 11.987 & \second{11.889} & 12.067 & 12.002 & 12.080 & 13.015 & 11.910 & 13.495 \\
& \multicolumn{1}{|l|}{MASE}  & \best{1.585}  & 1.605  & \second{1.591}  & 1.623  & 1.604  & 1.647  & 1.836  & 1.604  & 1.864  \\
& \multicolumn{1}{|l|}{OWA}   & \best{0.868}  & 0.881  & \second{0.872}  & 0.887  & 0.882  & 0.898  & 0.983  & 0.876  & 1.015 \\
\midrule
\label{tab:short-term}
\end{tabular}
}
\renewcommand{\arraystretch}{1.0}
\vspace{-10pt}
\end{table*}

\section{Experiments}
\label{sec:exp}

We conduct a comprehensive empirical evaluation of \method on standard long-term and short-term forecasting datasets against eight state-of-the-art baseline models. We also provide deeper analyses to investigate its efficiency, robustness, and underlying design principles.

\mypar{Datasets}
We evaluate \method on both long-term and short-term forecasting tasks, following the TSLib benchmark~\citep{wang2026deep}. The evaluated datasets include (long-term) Weather, ETT (ETTh1, ETTh2, ETTm1, ETTm2), Electricity, Exchange, Traffic, and (short-term) M4.
Following the standard protocol \citep{zhou2021informer}, we partition all datasets into training, validation, and test sets in chronological order, using a 6:2:2 ratio for the ETT datasets and a 7:1:2 ratio for the others.
Consistent with previous work, the observation window is fixed to $L=96$, and forecasting horizons are set to $T\in\{96,192, 336, 720\}$ for long-term forecasting tasks \citep{wu2023timesnet}.
As the M4 benchmark consists of 100,000 univariate time series collected at multiple temporal frequencies ranging from yearly to hourly, we follow the official benchmark protocol, where the prediction lengths are fixed according to the frequency of each series~\citep{makridakis2018m4}.
Comprehensive details for each dataset are provided in Appendix~\ref{sec:exp-detail}.

\mypar{Baselines}
We compare \method with eight state-of-the-art baselines for time series forecasting based on multi-scale modeling strategies: AMD \citep{hu2025adaptive}, MultiPatchFormer \citep{naghashi2025multiscale}, WPMixer \citep{murad2025wpmixer}, TimeMixer \citep{wang2024timemixer}, MSGNet \citep{cai2024msgnet}, MICN \citep{wang2023micn}, TimesNet \citep{wu2023timesnet}, and Pyraformer \citep{liu2022pyraformer}.

\mypar{Experimental Settings}
We adopt evaluation metrics used in previous work for each benchmark.
For long-term forecasting, we use mean squared error (MSE) and mean absolute error (MAE)~\citep{wu2023timesnet}.
For short-term forecasting, we report symmetric mean absolute percentage error (SMAPE), mean absolute scaled error (MASE), and the overall weighted average (OWA)~\citep{Oreshkin2020N-BEATS:}. 
We apply the same batch size, number of training epochs, and early-stopping strategy across all methods~\citep{wu2023timesnet}.
Each experiment is repeated three times, and the averaged results are reported. 

\subsection{Comparison with Baseline Models}

\mypar{Long-Term Forecasting}
Table~\ref{tab:long-term} demonstrates that \method achieves the best performance in 13 out of 16 evaluation settings for long-term forecasting across all datasets and horizons. 
The improvements are particularly evident in the high-dimensional Electricity and Traffic datasets.
\method achieves a 5.4\% and 7.8\% reduction in MSE relative to the second-best model, respectively.
On the other hand, \method is less significant on the Weather dataset, which has relatively low dimensionality and high intrinsic forecastability. In such a setting, discrete scaling operators can also capture the dominant temporal structure, allowing existing methods to perform competitively.
Overall, these results highlight the importance of learning 
distance-aware scale parameters from data to capture fine-grained temporal dynamics through principled multi-scale modeling when applied to real-world multivariate time series, with particularly strong benefits in complex, high-dimensional settings.

\mypar{Short-Term Forecasting}
\Cref{tab:short-term} shows that \method also establishes clear superiority over the competitors on the M4 benchmark for short-term forecasting, achieving the best results in 11 out of 15 cases.
It achieves the best average results, particularly in the Quarterly and Monthly datasets, which contain the largest number of series.
On the other hand, \method exhibits limited performance in the ``Others'' category, which lacks sufficient data, as it accounts for less than 5\% of the benchmark.
This suggests that learning 
distance-aware scale parameters is most effective when trained on sufficiently large datasets, while less-structured series favor stronger inductive biases to extract meaningful patterns.
Nevertheless, \method outperforms all other baselines by a large margin overall.

\begin{figure}[t]
  \centering
  \includegraphics[
  width=0.8\columnwidth,
  alt={Figure 5 presents an efficiency analysis on the ETTh1 dataset under the predict-720 setting. The scatter plot compares multi-scale forecasting methods by training time on the x-axis, MSE on the y-axis, and memory footprint as marker size. SiGMA appears in the lower-left region, with 8 milliseconds per iteration, 76 megabytes of memory, and the lowest MSE, indicating the best accuracy-efficiency trade-off. Compared with AMD, MultiPatchFormer, WPMixer, and TimeMixer, SiGMA achieves lower forecasting error while requiring less training time and memory.}
  ]{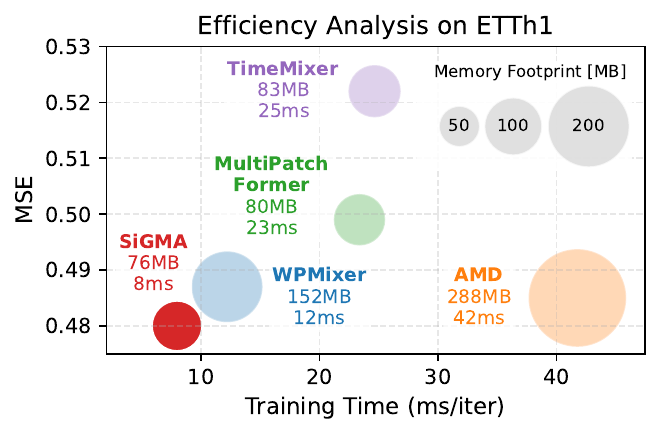}
  \vspace{-5pt}
  \caption{Efficiency analysis on ETTh1 with predict-720 setting. \method achieves the best trade-off between accuracy and efficiency, attaining the lowest MSE while requiring substantially less training time and memory than competing multi-scale methods.}
  \vspace{-10pt}
  \label{fig:efficiency}
\end{figure}

\mypar{Efficiency Analysis}
The performance gains achieved by \method are realized through a lightweight and parameter-efficient architecture.
We assess efficiency in terms of per-iteration training time and memory footprint on the ETTh1 dataset under the predict-720 setting, with all methods executed on a single GPU using the same batch size.  
As shown in~\Cref{fig:efficiency}, \method attains the fastest training time and the most compact memory footprint while achieving the lowest MSE.
Specifically, \method reduces memory consumption by 3.8 times and improves training speed by 5.3 times compared to AMD, the previous state-of-the-art baseline.

This significant efficiency gap comes from fundamental architectural differences.
AMD stacks three downsampling stages and applies per-variable multi-scale mixing in a sequential manner.
While this approach may improve accuracy over other baselines, it imposes substantial memory and runtime overhead.
In contrast, \method employs a single-kernel formulation that preserves full parallelism in its representation updates. 
Our principled simplification of multi-scale modeling yields a more favorable speed-memory trade-off without compromising modeling capacity, leading to superior overall performance.
We extend this efficiency analysis to additional datasets in Appendix~\ref{sec:efficiency}, showing that \method preserves faster training time with competitive memory usage on complex datasets.

\begin{figure*}[t]
\centering
\includegraphics[
width=\textwidth,
alt={Figure 6 provides a case study on the Traffic dataset under the predict-96 setting. The left side shows that SiGMA learns distance-aware scale parameters for each timestep and uses the LDG kernel to produce both a smoothed trend component and a residual variation component. The middle part shows that these components are separately processed and then integrated by an MLP. The right side compares the predicted sequence with the ground truth, showing that the integrated prediction closely follows both the long-term periodic trend and short-term fluctuations.}
]{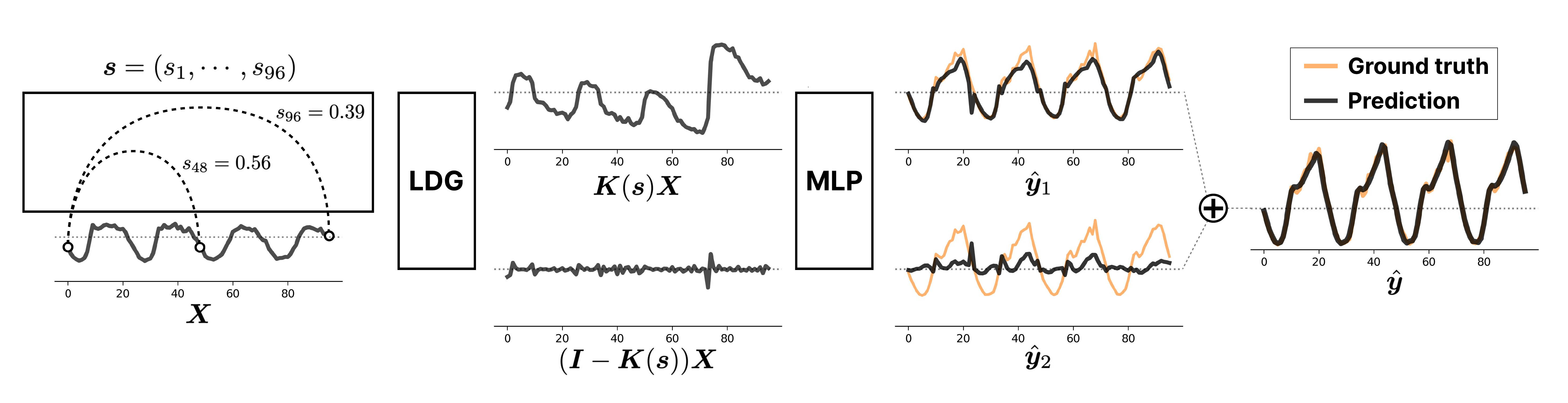}
\caption{
Case study on Traffic for predict-96 setting. 
\method learns 
distance-aware scale parameters to adaptively control temporal smoothing via the LDG kernel, while an MLP integrates the resulting multi-scale representations to capture both long-term trends and short-term variations.
By integrating these complementary signals, \method achieves more accurate and effective predictions.}

\vspace{-4pt}
\label{fig:case-study}
\end{figure*}


\subsection{Deeper Analysis on \method}
\label{ssec:deep-analysis}

\begin{wraptable}[10]{r}{0.45\columnwidth}
    \vspace{-13pt}
    \centering
    \caption{Ablation study with predict-720 setting on ETTh1.}
    \vspace{-5pt}
    \small
    \label{tab:ablation}
    \begin{tabular}{ccc}
        \toprule
        Method   & MSE   & MAE   \\
        \midrule
        \method & \best{0.480} & \second{0.468} \\
        \circone & \second{0.486} & \best{0.467} \\
        
        \circtwo  &  0.489     &  0.473     \\
        \circthree & 0.490 & 0.474 \\
        \circfour  &  0.492     &  0.475     \\
        \circfive  &  0.493     &  0.475     \\
        \circsix  &  0.524   & 0.492 \\
        
        \bottomrule
    \end{tabular}
    
\end{wraptable}

\mypar{Ablation Study}
To validate the design principles of \method, we conduct an ablation study on ETTh1 under the predict-720 setting, comparing \method against five of its variants:
\circone replaces the MLP predictor with the trend-seasonal mixing mechanism of TimeMixer, which is a more complicated variant of an MLP;
\circtwo changes the LDG kernel to have a single scale parameter applied to all elements;
\circthree changes the LDG kernel to use a sample-wise scale parameter using a two-layer MLP;
\circfour removes the scaling mechanism entirely and relies only on the raw input;
\circfive employs moving average, a non-learnable scale operator family;
\circsix uses unnormalized convolution, a learnable kernel that does not belong to the scaling operator family.

As shown in \Cref{tab:ablation}, 
most modifications result in performance degradation. This indicates that the expressivity of the LDG kernel in \circone and \circtwo, the effect of scaling in \circfour, and the 
distance-aware learnable scaling in \circfive all play a critical role.
While \circone achieves performance comparable to \method through a different multi-scale 
modeling strategy, it further validates that a simpler and more parameter-efficient MLP-based formulation is sufficient for effective forecasting.
\circthree shows that direct sample-wise scaling does not improve performance, suggesting that its increased flexibility may introduce higher variance and greater sensitivity to noise.
The significant performance drop in \circsix demonstrates that invalid scaling operations which do not belong to the scaling operator family can be sub-optimal, as they can distort the inherent characteristics of time series as well.

\begin{figure}[t]
    \centering

    \begin{subfigure}[t]{0.47\columnwidth}
        \centering
        \includegraphics[
        width=\linewidth,
        alt={Figure 7a shows forecasting performance on ETTh1 across input lengths 96, 192, 336, and 720. The y-axis reports MSE, and four lines correspond to prediction lengths 96, 192, 336, and 720. The longest prediction horizon, predict-720, improves as the input length increases, while shorter horizons achieve their best or near-best performance at intermediate input lengths.}
        ]{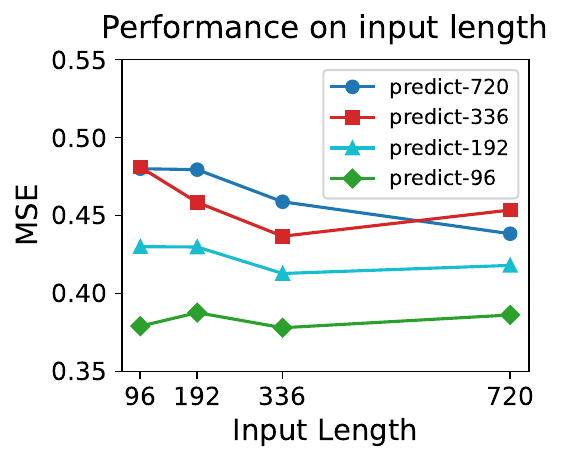}
        \caption{Performance across input lengths on ETTh1.}
        \label{fig:etth1_mse}
    \end{subfigure}
    \hfill
    \begin{subfigure}[t]{0.47\columnwidth}
        \centering
        \includegraphics[
        width=\linewidth,
        alt={Figure 7b shows computational cost on ETTh1 across input lengths 96, 192, 336, and 720. The left y-axis reports training time in milliseconds, and the right y-axis reports memory usage in megabytes. Both training time and memory usage increase as the input length becomes larger, indicating that longer input windows improve modeling capacity at the cost of higher computation and memory.}
        ]{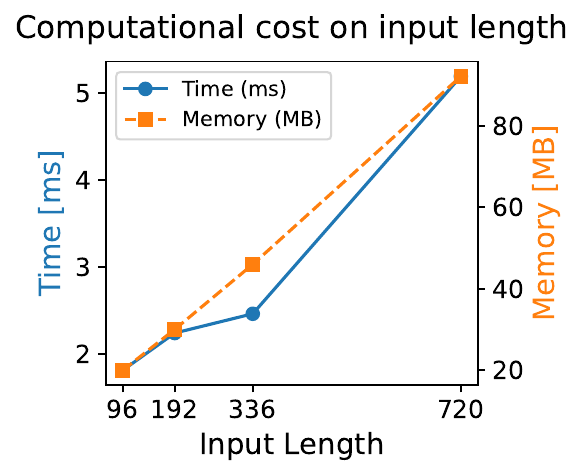}
        \caption{Computational cost across input lengths on ETTh1.}
        \label{fig:etth1_cost}
    \end{subfigure}

    \caption{Hyperparameter sensitivity of the input length $L$ on ETTh1.
    (a) Larger $L$ generally benefits long-horizon forecasting, while intermediate sizes suffice for shorter horizons. (b) Training time and memory usage scale linearly with $L$.}
    \label{fig:hyperparam_sensitivity}

\vspace{-10pt}
\end{figure}

\mypar{Hyperparameter Sensitivity}
We further examine the sensitivity of \method to the input length $L$, a key hyperparameter in multi-scale forecasting.
We evaluate forecasting performance (MSE) and computational cost on ETTh1 under varying lookback windows $L \in \{96, 192, 336, 720\}$.

As shown in Figure~\ref{fig:etth1_mse}, increasing the input length leads to consistent performance improvements for the long-horizon setting (predict-720), whereas for shorter horizons the best results are typically achieved with intermediate sizes.
This aligns with the properties of the LDG kernel, which has an effectively unbounded receptive field, allowing longer input sequences to better capture long-range temporal structure.
Figure~\ref{fig:etth1_cost} shows that memory usage grows linearly, while training time increases superlinearly with $L$. Under the current implementation, the LDG operator is applied via dense matrix multiplication, resulting in $O(L^2)$ time complexity. However, because the LDG kernel induces a Toeplitz operator, this cost could be reduced through more efficient implementations, such as truncated convolution or FFT-based multiplication. We provide further analysis in Appendix~\ref{sec:ldg-complexity}.
Overall, \method achieves an effective accuracy-cost trade-off as the lookback window increases.

\mypar{Case Study}
We conduct a case study on the Traffic dataset under the predict-96 setting to illustrate the behavior of \method. As shown in \Cref{fig:case-study}, \method learns a scale parameter at each timestep. The LDG kernel acts as a principled smoothing operator, allowing the model to selectively emphasize broad trends or localized variations. 
The resulting representations are then transformed by a lightweight MLP,
which captures richer temporal dependencies and adjusts the balance between coarse- and fine-grained components. 
By integrating these components, \method yields predictions that closely align with the ground truth, validating the effectiveness of its multi-scale design.

\section{Conclusion}
\label{sec:conclusion}

In this work, we introduce the concept of \emph{scaling operator families}, providing a principled foundation for understanding and unifying scaling mechanisms used in multi-scale time-series modeling.
Building on this foundation, we propose \method, a principled approach grounded in a \emph{generalized scaling operator family}.
\method employs the learnable discrete Gaussian (LDG) kernel, which generalizes classical scaling operators by enabling continuous, 
distance-aware and learnable scale parameters while preserving coherent cross-scale structure at the representation level.
Extensive experiments on both long- and short-term forecasting benchmarks demonstrate that \method consistently outperforms state-of-the-art baselines, while substantially reducing memory consumption and computational time.

\mypar{Limitations and Future Work}
A limitation of this work is that dynamic scaling parameters are learned at the dataset level, which is suboptimal when training samples are insufficient.
Future work can include extending our framework to multivariate scaling and sample-specific dynamic modeling to further enhance its expressiveness and applicability.
Nevertheless, our ablation results indicate that direct sample-wise parameterization requires careful design. A promising future direction is to combine shared dataset-level scales with residual sample-wise adaptations, thereby improving flexibility while preserving stability.

\section*{Acknowledgements}

This work was partly supported by the National Research Foundation of Korea (NRF) grant funded by the Korea government (MSIT) (RS-2024-00341425 and RS-2024-00406985), ``Advanced GPU Utilization Support Program'' funded by the Government of the Republic of Korea (Ministry of Science and ICT), and the New Faculty Startup Fund from Seoul National University.

\section*{Impact Statement}

This work contributes to the foundations of time-series forecasting through a structured multi-scale modeling approach. Its societal impact is indirect and mediated by downstream applications, and we do not foresee ethical concerns beyond those commonly associated with forecasting methods.





\bibliography{reference}
\bibliographystyle{icml2026}

\clearpage
\newpage
\onecolumn

\appendix

\section{Proofs}
\label{sec:proofs}

\subsection{Proof of \Cref{thm:scale-op}}

\begin{lemma}
\label{lem:nonexp-energy}
If $T:\mathbb{R}^m\to\mathbb{R}^n$ satisfies
$\|T(\vx)-T(\vy)\|_2\le \|\vx-\vy\|_2$ for all $\vx,\vy$, then
$\|T(\vx)\|_2\le\|\vx\|_2$ for all $\vx$.
\end{lemma}
\begin{proof}
$\|T(\vx)\|_2
= \|T(\vx)-T(\mathbf{0})\|_2
\le \|\vx-\mathbf{0}\|_2
= \|\vx\|_2.
$\end{proof}

\begin{lemma}
\label{lem:comp}
Suppose for scales $s_i,s_j$ we have
$f(\cdot|s_i) = G_{i,j}\circ f(\cdot|s_j)$ for some non-expansive $G_{i,j}$. Then, $\|f(\vx|s_i)\|_2 \le \|f(\vx|s_j)\|_2.$ for all $\vx$.
Moreover, if there exists some $\vy$ in the range of $f(\cdot|s_j)$ with
$\|G_{i,j}(\vy)\|_2 < \|\vy\|_2$, then there exists $\vx'\in\gX$ such that
$\|f(\vx'|s_i)\|_2 < \|f(\vx'|s_j)\|_2$.
\end{lemma}
\begin{proof}
By Lemma~\ref{lem:nonexp-energy}, we have $\|G_{i,j}(\vu)\|_2\le \|\vu\|_2$ for all $\vu$.
Thus for any $\vx$,
\begin{align*}
\|f(\vx|s_i)\|_2
&= \|G_{i,j}(f(\vx|s_j))\|_2
\le \|f(\vx|s_j)\|_2.
\end{align*}
If some $\vy$ in the range of $f(\cdot|s_j)$ satisfies
$\|G_{i,j}(\vy)\|_2<\|\vy\|_2$, pick $\vx'$ with $f(\vx'|s_j)=\vy$.
Then
\begin{align*}
\|f(\vx'|s_i)\|_2
= \|G_{i,j}(\vy)\|_2
< \|\vy\|_2
= \|f(\vx'|s_j)\|_2.
\end{align*}
\end{proof}

\paragraph{Average Pooling.}

Define 
\begin{align*}
\big(f_{\mathrm{avg}}(\vx|s)\big)_m
:= \frac{1}{s}\sum_{i\in B_m} x_i,
\end{align*}
where $B_m=\{(m-1)s+1,\dots,ms\}$ for $m=1,\dots,L_s:=P/s$.

\emph{Non-expansiveness.}
For any $\vx,\vy$ and any block $B_m$,
\begin{align*}
\big|f_{\mathrm{avg}}(\vx|s)_m - f_{\mathrm{avg}}(\vy|s)_m\big|
&= \frac{1}{s}\Big|\sum_{i\in B_m} (x_i-y_i)\Big| \\
&\le \frac{1}{s}\sqrt{s}\,\|\vx|_{B_m}-\vy|_{B_m}\|_2
\le \|\vx|_{B_m}-\vy|_{B_m}\|_2.
\end{align*}
Summing squares over $m$ yields
\[
\|f_{\mathrm{avg}}(\vx|s)-f_{\mathrm{avg}}(\vy|s)\|_2\le\|\vx-\vy\|_2,
\]
so non-expansiveness holds.

\emph{Energy Reduction.}
Let $s_i = m\cdot s_j$ with $m>1$.
Each $s_i$-block is a union of $m$ consecutive $s_j$-blocks.
Let $y = f_{\mathrm{avg}}(\vx|s_j)\in\mathbb{R}^{L_{s_j}}$, and define
\begin{align*}
\big(G_{i,j}(y)\big)_k := \frac{1}{m}\sum_{r=0}^{m-1} y_{km+r},
\end{align*}
i.e., average pooling on $y$ with block size $m$.
Then one checks directly that
\begin{align*}
f_{\mathrm{avg}}(\vx|s_i) = G_{i,j}\big(f_{\mathrm{avg}}(\vx|s_j)\big),
\end{align*}
and $G_{i,j}$ is the same blockwise averaging operator as above, hence non-expansive.
By Lemma~\ref{lem:comp}, for all $\vx$,
\begin{align*}
\|f_{\mathrm{avg}}(\vx|s_i)\|_2 \le \|f_{\mathrm{avg}}(\vx|s_j)\|_2.
\end{align*}

To see strictness for some input, choose $\vy$ with two distinct values in each $m$-block, so that the block average has strictly smaller energy than $\vy$.
Since $f_{\mathrm{avg}}(\cdot|s_j)$ is surjective onto $\mathbb{R}^{L_{s_j}}$, there exists $\vx'$ with $f_{\mathrm{avg}}(\vx'|s_j)=\vy$, and Lemma~\ref{lem:comp} yields
\begin{align*}
\|f_{\mathrm{avg}}(\vx'|s_i)\|_2 < \|f_{\mathrm{avg}}(\vx'|s_j)\|_2.
\end{align*}

\paragraph{Max- and Min-pooling.}

Define
\begin{align*}
\big(f_{\max}(\vx|s)\big)_m := \max_{i\in B_m} x_i,\,
\big(f_{\min}(\vx|s)\big)_m := \min_{i\in B_m} x_i.
\end{align*}

\emph{Non-expansiveness.}
For any $\vu,\vv\in\mathbb{R}^s$,
\begin{align*}
\big|\max_i u_i - \max_i v_i\big|
\le \max_i |u_i-v_i|
\le \|\vu-\vv\|_2,
\end{align*}
and similarly for the minimum.
Thus each block map is $1$-Lipschitz; by summing over blocks we get
\[
\|f_{\max}(\vx|s)-f_{\max}(\vy|s)\|_2\le\|\vx-\vy\|_2
\]
and likewise for $f_{\min}$.

\emph{Energy Reduction.}
Let $s_i = m\cdot s_j$ with $m>1$. Define $y=f_{\max}(\vx|s_j)$ and
\begin{align*}
\big(G_{i,j}(y)\big)_k := \max_{r\in\{0,\cdots,m-1\}} y_{km+r},
\end{align*}
i.e., max-pooling over $m$ consecutive entries of $y$.
Then
\begin{align*}
f_{\max}(\vx|s_i) = G_{i,j}\big(f_{\max}(\vx|s_j)\big),
\end{align*}
and the same max-Lipschitz argument as above shows $G_{i,j}$ is non-expansive.
Thus by Lemma~\ref{lem:comp}, for all $\vx$,
\begin{align*}
\|f_{\max}(\vx|s_i)\|_2 \le \|f_{\max}(\vx|s_j)\|_2.
\end{align*}

To get strictness, choose $\vy$ such that in each group of $m$ entries there is at least one strictly smaller than the maximum; then $\|G_{i,j}(\vy)\|_2<\|\vy\|_2$.
As before, we can realize such a $\vy$ as $f_{\max}(\vx'|s_j)$ by setting one large value and smaller ones within each block. The same reasoning applies to min-pooling by symmetry.

\paragraph{Moving Average.}

Define
\begin{align*}
\big(f_{\mathrm{ma}}(\vx|s)\big)_t
:= \frac{1}{s}\sum_{i=0}^{s-1} x_{t-i},
\qquad t=1,\dots,P.
\end{align*}
This is convolution with $h_s = \frac{1}{s}\mathbf{1}_s$.

\emph{Non-expansiveness.}
By Young’s inequality,
\begin{align*}
\|f_{\mathrm{ma}}(\vx|s)-f_{\mathrm{ma}}(\vy|s)\|_2
= \|h_s*(\vx-\vy)\|_2
\le \|h_s\|_1\|\vx-\vy\|_2
= \|\vx-\vy\|_2.
\end{align*}
Thus $f_{\mathrm{ma}}(\cdot|s)$ is non-expansive.

\emph{Energy Reduction.}
Let $X(\omega)$ and $Y_s(\omega)$ denote the Fourier transforms of $\vx$ and $f_{\mathrm{ma}}(\vx|s)$, respectively, and $H_s(\omega)$ the frequency response of $h_s$.
We have
\begin{align*}
Y_s(\omega) = H_s(\omega)X(\omega),\,
H_s(\omega)
= \frac{1}{s}\sum_{k=0}^{s-1} e^{-i\omega k}
= \frac{1}{s}\,\frac{\sin(s\omega/2)}{\sin(\omega/2)}e^{-i\omega(s-1)/2}.
\end{align*}
Thus
\begin{align*}
\|f_{\mathrm{ma}}(\vx|s)\|_2^2
= \frac{1}{2\pi}\int_{-\pi}^{\pi} |H_s(\omega)|^2 |X(\omega)|^2\,d\omega.
\end{align*}

For $s_i = m\cdot s_j$, write $\alpha = s_j\omega/2$.
Then
\begin{align*}
\frac{|H_{s_i}(\omega)|}{|H_{s_j}(\omega)|}
= \frac{1}{m}\,\frac{|\sin(m\alpha)|}{|\sin(\alpha)|}
\le 1,
\end{align*}
since $|\sin(m\alpha)|\le m|\sin(\alpha)|$ for all $m\in\mathbb{N}$ and all $\alpha$.
Hence $|H_{s_i}(\omega)|^2\le |H_{s_j}(\omega)|^2$ for all $\omega$, and therefore, for any $\vx$,
\begin{align*}
\|f_{\mathrm{ma}}(\vx|s_i)\|_2^2
\le \|f_{\mathrm{ma}}(\vx|s_j)\|_2^2.
\end{align*}

To see when the inequality is strict, define
\[
D(\omega)\;:=\;|H_{s_j}(\omega)|^2 - |H_{s_i}(\omega)|^2 \;\ge 0.
\]
The above argument shows $D(\omega)\ge 0$ for all $\omega$, and $H_{s_i}\not\equiv H_{s_j}$ implies that $D(\omega)>0$ on a set of nonzero measure.
For any $\vx$,
\begin{align*}
\|f_{\mathrm{ma}}(\vx|s_j)\|_2^2 - \|f_{\mathrm{ma}}(\vx|s_i)\|_2^2
&= \frac{1}{2\pi}\int_{-\pi}^{\pi} D(\omega)\,|X(\omega)|^2\,d\omega.
\end{align*}
Therefore,
\[
\|f_{\mathrm{ma}}(\vx|s_j)\|_2^2 > \|f_{\mathrm{ma}}(\vx|s_i)\|_2^2
\]
for every input $\vx$ whose spectrum $|X(\omega)|^2$ places nonzero mass on the region where $D(\omega)>0$.
Equivalently, the inequality is strict for any non-degenerate $\vx$ whose energy is not concentrated entirely on the set where $|H_{s_i}(\omega)|=|H_{s_j}(\omega)|$.

\paragraph{Subsampling.}

Define
\begin{align*}
\big(f_{\mathrm{sub}}(\vx|s)\big)_m := x_{(m-1)s+1},
\qquad L_s = P/s.
\end{align*}

\emph{Non-expansiveness.}
For each block $B_m$, define $g_s(\mathbf{u}) = u_1$.
Then
\begin{align*}
|g_s(\mathbf{u})-g_s(\mathbf{v})| = |u_1-v_1|\le \|\mathbf{u}-\mathbf{v}\|_2.
\end{align*}
Applying the blockwise arguments shows that $f_{\mathrm{sub}}(\cdot|s)$ is non-expansive.

\emph{Energy Reduction.}
If $s_i = m\cdot s_j$, pure decimation satisfies
\begin{align*}
f_{\mathrm{sub}}(\vx|s_i)
= f_{\mathrm{sub}}\big(f_{\mathrm{sub}}(\vx|s_j)\big|m\big).
\end{align*}
So $f_{\mathrm{sub}}(\cdot|s_i) = G_{i,j}\circ f_{\mathrm{sub}}(\cdot|s_j)$ with $G_{i,j}$ another subsampling operator, which is non-expansive.
Hence Lemma~\ref{lem:comp} yields
\begin{align*}
\|f_{\mathrm{sub}}(\vx|s_i)\|_2\le\|f_{\mathrm{sub}}(\vx|s_j)\|_2
\end{align*}
for all $\vx$. The strict inequality holds by setting $\vx$ as nonzero values concentrated on indices that are dropped at the coarser stride.

\paragraph{Segmentation.}

Define
\begin{align*}
f_{\mathrm{seg}}(\vx|s)
:= (x_1,\dots,x_{L_s}) \in \mathbb{R}^{L_s},\qquad L_s=P/s,
\end{align*}

\emph{Non-expansiveness.}
This is an orthogonal projection onto the first $L_s$ coordinates:
\begin{align*}
f_{\mathrm{seg}}(\vx|s) = R_s\vx,\quad R_sR_s^\top=I_{L_s}.
\end{align*}
Thus
\begin{align*}
\|f_{\mathrm{seg}}(\vx|s)-f_{\mathrm{seg}}(\vy|s)\|_2
= \|R_s(\vx-\vy)\|_2 \le \|\vx-\vy\|_2,
\end{align*}
so $f_{\mathrm{seg}}$ is non-expansive and
$\|f_{\mathrm{seg}}(\vx|s)\|_2\le\|\vx\|_2$.

\emph{Energy Reduction.}
If $s_i = m\cdot s_j$, then
\begin{align*}
f_{\mathrm{seg}}(\vx|s_i)
= R_{s_i}\vx
= R_{s_i}R_{s_j}^\top f_{\mathrm{seg}}(\vx|s_j).
\end{align*}
Here $G_{i,j}:=R_{s_i}R_{s_j}^\top$ is an orthogonal projection from $\mathbb{R}^{L_{s_j}}$ to $\mathbb{R}^{L_{s_i}}$, hence non-expansive.
Thus by Lemma~\ref{lem:comp}, for all $\vx$,
\begin{align*}
\|f_{\mathrm{seg}}(\vx|s_i)\|_2\le\|f_{\mathrm{seg}}(\vx|s_j)\|_2.
\end{align*}
The inequality is strict exactly when at least one of the discarded coordinates is nonzero, i.e., when
\[
x_k \neq 0 \quad\text{for some }k\in\{L_{s_i}+1,\dots,L_{s_j}\}.
\]
For such $\vx$ we have
\begin{align*}
\sum_{k=L_{s_i}+1}^{L_{s_j}} x_k^2 > 0\Rightarrow \|f_{\mathrm{seg}}(\vx|s_i)\|_2
< \|f_{\mathrm{seg}}(\vx|s_j)\|_2.    
\end{align*}

\paragraph{Wavelet Decomposition.}
Let $\{\phi_{s,k}\}_{k}$ be an orthonormal basis of $V_s$, and
define the wavelet approximation coefficients by
\begin{align*}
f_{\mathrm{wav}}(\vx|s)_k := \langle \vx,\phi_{s,k}\rangle.
\end{align*}
For standard wavelets, the approximation spaces $\{V_s\}$ are typically constructed as nested
subspaces indexed by dyadic scales $s=2^j$.
Here, we consider a more general multiplicative scale set $\mathcal{S}\subset\mathbb{Z}_+$ and
assume that to each $s\in\mathcal{S}$ we associate an approximation subspace
$V_s\subset\mathbb{R}^{L_s}$ with orthogonal projector $P_s:\mathbb{R}^L\to V_s$, such that
\begin{align*}
s_i = m s_j\ \text{ with }m>1
\quad\Longrightarrow\quad
V_{s_i}\subset V_{s_j}
\ \text{ and }\
P_{s_i} = P_{s_i}P_{s_j}.
\end{align*}
Equivalently, if $C_s:V_s\to\mathbb{R}^{L_s}$ denotes the isometry mapping
$P_s\vx$ to its coordinates in the basis $\{\phi_{s,k}\}_k$, then $f_{\mathrm{wav}}(\vx|s) = C_s P_s\vx \in\mathbb{R}^{L_s}.$

\emph{Non-expansiveness.}
For any $s\in\mathcal{S}$, $P_s$ is an orthogonal projector, hence
\begin{align*}
\|P_s\vx - P_s\vy\|_2 \le \|\vx - \vy\|_2.
\end{align*}
Since $C_s$ is an isometry on $V_s$,
\begin{align*}
\|f_{\mathrm{wav}}(\vx|s)-f_{\mathrm{wav}}(\vy|s)\|_2
= \| C_s(P_s\vx - P_s\vy)\|_2
= \|P_s\vx - P_s\vy\|_2
\le \|\vx - \vy\|_2,
\end{align*}
so $f_{\mathrm{wav}}(\cdot|s)$ is non-expansive.

\emph{Energy Reduction.}
Fix $s_i,s_j\in\mathcal{S}$ with $s_i = m s_j, m>1.$ 
For any $\vx$, we can write
\begin{align*}
f_{\mathrm{wav}}(\vx|s_j) = C_{s_j}P_{s_j}\vx,\quad
f_{\mathrm{wav}}(\vx|s_i) = C_{s_i}P_{s_i}\vx
= C_{s_i}P_{s_i}P_{s_j}\vx.
\end{align*}
Define a linear map $G_{i,j}:\mathbb{R}^{L_{s_j}}\to\mathbb{R}^{L_{s_i}}$ by
\begin{align*}
G_{i,j}
:= C_{s_i}P_{s_i}P_{s_j}^\top C_{s_j}^\top,
\quad\text{so that}\quad
f_{\mathrm{wav}}(\vx|s_i)
= G_{i,j}\big(f_{\mathrm{wav}}(\vx|s_j)\big).
\end{align*}
Here $C_{s_j},C_{s_i}$ are isometries and $P_{s_i}$ is an orthogonal projector, so
\[
\|G_{i,j}\|_{2}\le 1,
\]
i.e., $G_{i,j}$ is non-expansive.
By Lemma~\ref{lem:comp}, for all $\vx$,
\begin{align*}
\|f_{\mathrm{wav}}(\vx|s_i)\|_2
= \|G_{i,j}(f_{\mathrm{wav}}(\vx|s_j))\|_2
\le \|f_{\mathrm{wav}}(\vx|s_j)\|_2.
\end{align*}
For strictness, pick any $\vx\in V_{s_j}\setminus V_{s_i}$.
Then $P_{s_j}\vx = \vx$, while $P_{s_i}\vx$ is the orthogonal projection of $\vx$
onto the strictly smaller subspace $V_{s_i}$, so by Pythagoras theorem,
$\|P_{s_i}\vx\|_2 < \|\vx\|_2.$ Then,
\begin{align*}
\|f_{\mathrm{wav}}(\vx|s_i)\|_2
= \|C_{s_i}P_{s_i}\vx\|_2
< \|C_{s_j}\vx\|_2
= \|f_{\mathrm{wav}}(\vx|s_j)\|_2.
\end{align*}

\subsection{Proof of Theorem~\ref{thm:not-scale-op}}
\begin{proof}
We will show that each of  operations is scale-degenerate in the sense
that, for some pair $s_i>s_j$, we have
\[
\|f(\vx|s_i)\|_2 = \|f(\vx|s_j)\|_2
\quad\text{for all }\vx\in\gX,
\]
so there exists no input $\vx$ that yields strict inequality. In other words, all such families are degenerate in the scale parameter and therefore violate the energy reduction condition in Definition~\ref{def:scale-op}.

\paragraph{Constant Mappings.}  
If $f(\vx|s)=\vc$ for some fixed $\vc\in\R^{L_s}$, then 
\[
\|f(\vx|s_i)\|_2=\|f(\vx|s_j)\|_2=\|\vc\|_2.
\]

\paragraph{Permutations.}  
If $f(\vx|s)=\Pi_s \vx$ for a permutation matrix $\Pi_s$, then
\[
\|f(\vx|s_i)\|_2  = \|f(\vx|s_j)\|_2 = \|\vx\|_2.
\]

\paragraph{Additive Shifts.}  
If $f(\vx|s)=\vx+\vb$, then
\[
\|f(\vx|s_i)\|_2  = \|f(\vx|s_j)\|_2 = \|\vx+\vb\|_2.
\]

\paragraph{Scalar Multiplications.}  
If $f(\vx|s)=c\vx$ for a constant $c\in\R$, then
\[
\|f(\vx|s_i)\|_2=\|f(\vx|s_j)\|_2=|c|\|\vx\|_2.
\]

\paragraph{General Linear Maps.}  
If $f(\vx|s)=\mW\vx$, then 
\[
\|f(\vx|s_i)\|_2 = \|f(\vx|s_j)\|_2 = \|\mW\vx\|_2.
\]

\end{proof}

\subsection{Proof of \Cref{thm:quantization-gap}}
\begin{proof}
Fix any $\vu\in\Pi_\mathbf{1}$. Consider
\[
h_\vu(\tau):=g_\vx(t\mathbf{1}+\tau\vu),\qquad \tau\in[0,1].
\]
By the fundamental theorem of calculus,
\[
h_\vu(\tau)
= h_\vu(0)+\tau h_\vu'(0)+\int_0^\tau(\tau-r)\,h_\vu''(r)\,dr.
\]
Using $h_\vu''(r)\le \sup_{r'\in[0,\tau]}h_\vu''(r')$ and $\int_0^\tau(\tau-r)\,dr=\tau^2/2$,
we obtain for all $\tau\in[0,1]$,
\[
h_\vu(\tau)
\ge
h_\vu(0)+\tau h_\vu'(0)-\frac{\tau^2}{2}\sup_{r\in[0,\tau]}h_\vu''(r).
\]
By the chain rule,
\[
h_\vu'(0)=\langle\nabla_\vs g_\vx(t\mathbf{1}),\vu\rangle,
\qquad
h_\vu''(r)=\vu^\top\nabla_\vs^2 g_\vx(t\mathbf{1}+r\vu)\,\vu.
\]
Define the directional curvature envelope
\[
\beta(\vx)
:=
\sup_{\vu\in\Pi_\mathbf{1}}\sup_{\tau\in[0,1]}
\vu^\top\nabla_\vs^2 g_\vx(t\mathbf{1}+\tau\vu)\,\vu
\in(0,\infty).
\]
Then
\[
g_\vx(t\mathbf{1}+\tau\vu)
\ge
g_\vx(t\mathbf{1})
+\tau\langle\nabla_\vs g_\vx(t\mathbf{1}),\vu\rangle
-\frac{\beta(\vx)}{2}\tau^2,
\qquad \forall\tau\in[0,1].
\]
For fixed $\vu$, the right-hand side is a concave quadratic in $\tau$, whose maximum over
$\tau\in[0,1]$ is at least
\[
\frac{\langle\nabla_\vs g_\vx(t\mathbf{1}),\vu\rangle^2}{2\beta(\vx)}.
\]
Hence,
\[
\max_{\tau\in[0,1]} g_\vx(t\mathbf{1}+\tau\vu)-g_\vx(t\mathbf{1})
\ge
C(\vx)\,\langle\nabla_\vs g_\vx(t\mathbf{1}),\vu\rangle^2,
\]
where $C(\vx)=1/(2\beta(\vx))$. Taking the supremum over $\vu\in\Pi_\mathbf{1}$, we obtain
\[
\sup_{\vu\in\Pi_\mathbf{1}}\max_{\tau\in[0,1]} g_\vx(t\mathbf{1}+\tau\vu)-g_\vx(t\mathbf{1})
\ge
C(\vx)\sup_{\vu\in\Pi_\mathbf{1}}
\langle\nabla_\vs g_\vx(t\mathbf{1}),\vu\rangle^2.
\]
By definition, $\Phi_c(\vx)=\max_{\vs\in\gS_c}g_\vx(\vs)$.
Since $\{t\mathbf{1}+\tau\vu:\tau\in[0,1]\}\subset\gS_c$ for all $\vu\in\Pi_\mathbf{1}$,
\[
\Phi_c(\vx)\ge \max_{\tau\in[0,1]} g_\vx(t\mathbf{1}+\tau\vu),
\quad \forall\vu\in\Pi_\mathbf{1}.
\]
Moreover, by construction of $t$ as a maximizer along the uniform-scale line,
\[
g_\vx(t\mathbf{1})\ge \max_{s\in\gS_d}g_\vx(s\mathbf{1})=\Phi_d(\vx).
\]
Combining the above inequalities yields
\[
\Phi_c(\vx)-\Phi_d(\vx)
\ge
C(\vx)\sup_{\vu\in\Pi_\mathbf{1}}
\langle\nabla_\vs g_\vx(t\mathbf{1}),\vu\rangle^2,
\]
which proves the claim.
\end{proof}

\subsection{Proof of \Cref{thm:gaussian-scale-op}}
\begin{proof}
We verify (i) \emph{differentiability} in $\vs$ and (ii) \emph{consistency} to a (discrete) scaling
operator family.

\paragraph{Differentiability in $\vs$.}
For each fixed pair $(i,j)$, the mapping $s_i\mapsto e^{-s_i}I_{|i-j|}(s_i)$ is real-analytic on
$\R_{+}$ because $e^{-s_i}$ and the modified Bessel function $I_{\nu}(s_i)$ are analytic for
$s_i>0$ and integer $\nu$. Hence $\vs\mapsto \mK(\vs)$ is $C^\infty$ on $\R_+^M$, and so is
$\vs\mapsto k(\vx|\vs)=\mK(\vs)\vx$ for any $\vx$. This proves the differentiability
requirement in \Cref{def:ext-scale-op}.

\paragraph{Consistency to a Discrete Scaling Operator Family.}

Fix $s\in\sZ_+$ and set $\vs=s\mathbf{1}$. Then $\mK(s\mathbf{1})$ becomes space-invariant:
\[
[\mK(s\mathbf{1})]_{i,j}=e^{-s}I_{|i-j|}(s)=:g_s(i-j),
\]
so $k(\vx|s\mathbf{1})$ is the convolution of $\vx$ with the kernel
$g_s\in\ell^1(\sZ)$. Two classical properties hold:

\smallskip
\emph{(a) Semigroup and normalization.} For all $s,t\ge 0$,
$g_{s+t}=g_s\ast g_t$, and $\sum_{n\in\sZ} g_s(n)=1$
. Thus $\mK(s\mathbf{1})$ is a symmetric,
doubly-stochastic Markov operator.

\smallskip
\emph{(b) Frequency response and contraction.} Let $\widehat{g}_s(\omega)$ denote the discrete-time
Fourier transform (DTFT) of $g_s$. A standard computation gives
\[
\widehat{g}_s(\omega)
= \exp\!\big(s(\cos\omega-1)\big)
= \exp\!\big(-2s\sin^2(\omega/2)\big),\qquad \omega\in[-\pi,\pi],
\]
hence $|\widehat{g}_s(\omega)|\le 1$ with equality only at $\omega=0$ for $s>0$.

\smallskip
We now verify the two axioms in \Cref{def:scale-op} for the discrete family
$f(\vx|s):=k(\vx|s\mathbf{1})$:

\emph{Non-expansiveness.}
By Plancherel,
\[
\|f(\vx|s)-f(\vx'|s)\|_2^2
=\frac{1}{2\pi}\!\int_{-\pi}^{\pi}\!|\widehat{g}_s(\omega)|^2\,
|\widehat{\vx-\vx'}(\omega)|^2\,d\omega
\le \|\vx-\vx'\|_2^2,
\]
since $|\widehat{g}_s(\omega)|\le 1$. Thus
$\|f(\vx|s)-f(\vx'|s)\|_2 \le \|\vx-\vx'\|_2$.

\emph{Energy Reduction.}
Let $s_i,s_j\in\sZ_+$ with $s_i = m s_j$ for some integer $m>1$.
Let $X(\omega)$ denote the DTFT of $\vx$; then by Plancherel,
\[
\|f(\vx|s)\|_2^2
= \frac{1}{2\pi}\int_{-\pi}^{\pi} |\widehat{g}_s(\omega)|^2\,|X(\omega)|^2\,d\omega.
\]
Using the semigroup property, 
\[
g_{s_i} = g_{s_j}^{\ast m}
:= \underbrace{g_{s_j} \ast \cdots \ast g_{s_j}}_{m\ \text{times}},
\] 
the frequency responses satisfy
\[
\widehat{g}_{s_i}(\omega)
= \big(\widehat{g}_{s_j}(\omega)\big)^m,
\qquad
|\widehat{g}_{s_i}(\omega)|^2
= |\widehat{g}_{s_j}(\omega)|^{2m}.
\]
Since $|\widehat{g}_{s_j}(\omega)|\le 1$ for all $\omega$ and $m>1$, we have
\[
|\widehat{g}_{s_i}(\omega)|^2
= |\widehat{g}_{s_j}(\omega)|^{2m}
\le |\widehat{g}_{s_j}(\omega)|^{2}
\quad\text{for all }\omega.
\]
Therefore, for any $\vx$,
\begin{align*}
\|f(\vx|s_i)\|_2^2
&= \frac{1}{2\pi}\int_{-\pi}^{\pi} |\widehat{g}_{s_i}(\omega)|^2\,|X(\omega)|^2\,d\omega \\
&\le \frac{1}{2\pi}\int_{-\pi}^{\pi} |\widehat{g}_{s_j}(\omega)|^2\,|X(\omega)|^2\,d\omega \\
&= \|f(\vx|s_j)\|_2^2,
\end{align*}
which proves the \emph{energy reduction} inequality
\[
\|f(\vx|s_i)\|_2 \le \|f(\vx|s_j)\|_2
\quad\text{for all }\vx
\quad\text{whenever }s_i = m s_j,\;m>1.
\]

To see that the inequality is strict for some $\vx$, note that for $s_j>0$ we have
\[
|\widehat{g}_{s_j}(\omega)| < 1
\quad\text{for all }\omega\neq 0,
\]
and hence
\[
|\widehat{g}_{s_i}(\omega)|^2 = |\widehat{g}_{s_j}(\omega)|^{2m}
< |\widehat{g}_{s_j}(\omega)|^2
\quad\text{for all }\omega\neq 0.
\]
Choose any $\vx$ whose spectrum is not supported solely at $\omega=0$, i.e., such that
$|X(\omega)|^2>0$ on a set of nonzero measure in $\{\omega\in[-\pi,\pi]:\omega\neq 0\}$.
Then on a set of positive measure,
\[
|\widehat{g}_{s_i}(\omega)|^2\,|X(\omega)|^2
< |\widehat{g}_{s_j}(\omega)|^2\,|X(\omega)|^2,
\]
so the inequality above is strict:
\[
\|f(\vx|s_i)\|_2 < \|f(\vx|s_j)\|_2.
\]
Thus the LDG kernel family satisfies both the universal inequality and the existence of a strict
example required by the energy reduction condition in Definition~\ref{def:scale-op}.

\end{proof}

\subsection{Proof of \Cref{thm:discrete-gaussian-optimality}}
\subsubsection{Discrete scale-space axioms}
\label{ssec:dssa-ldg-uniqueness}

\paragraph{Axioms (Discrete, 1D, Symmetric Case).}
For each scale $s\ge 0$, let $\mathcal{T}_s:\ell^2(\mathbb{Z})\to\ell^2(\mathbb{Z})$ denote the smoothing operator at scale $s$.
A family $\{\mathcal{T}_s\}_{s\ge 0}$ satisfies the \emph{discrete scale-space axioms} if:

\begin{enumerate}[label=\textbf{(A\arabic*)}, leftmargin=20pt, nosep]
\item \textbf{Linearity and Shift Invariance.} $\mathcal{T}_s$ is linear and commutes with shifts: there exists $K_s\in\ell^1(\mathbb{Z})$ with $(\mathcal{T}_s x)[n]=(K_s*x)[n]$.
\item \textbf{Semigroup over Scale and Identity at $0$.} $\mathcal{T}_0=\mathrm{Id}$ and $\mathcal{T}_{s+t}=\mathcal{T}_s\circ\mathcal{T}_t$, i.e., $K_{s+t}=K_s*K_t$ for all $s,t\ge 0$.
\item \textbf{Regularity in Scale.} The map $s\mapsto K_s$ is $C^1$ in $s$ in the $\ell^1$ topology.
\item \textbf{DC Normalization.} $\sum_{n\in\mathbb{Z}} K_s[n]=1$ for all $s\ge 0$ (constants are preserved).
\item \textbf{Non-enhancement of Local Extrema (NELE).} If $y(\cdot;s)=\mathcal{T}_s x$ has a local maximum at index $n$, then $\partial_s y[n;s]\le 0$ (resp.\ $\partial_s y[n;s]\ge 0$ at local minima).
\item \textbf{Symmetry.} $K_s[n]=K_s[-n]$ for all $n\in\mathbb{Z}$ and $s\ge 0$.
\end{enumerate}

\subsubsection{Uniqueness of the discrete Gaussian}

\begin{lemma}
\label{lem:cm-exist}
Let $A$ be a translation-invariant, symmetric operator on $\ell^2(\mathbb{Z})$ with impulse response $a[\cdot]\in\ell^1(\mathbb{Z})$ such that
\[
\sum_{n\in\mathbb{Z}} a[n]=0,\qquad a[n]\ge 0\ \text{for }n\ne 0,\qquad a[-n]=a[n].
\]
Then its Fourier symbol $\phi(\omega)=\widehat{A}(\omega)=\sum_{n} a[n]e^{-i n\omega}$ admits the Lévy-Khintchine form
\[
\phi(\omega)=\sum_{m=1}^{\infty} c_m\big(\cos(m\omega)-1\big),\qquad c_m\ge 0,
\]
with real coefficients $c_m=2a[m]$. Conversely, given any sequence $\{c_m\}_{m\ge 1}$ with $c_m\ge 0$ and $\sum_{m\ge 1} c_m<\infty$, defining
\[
a[0]=-\sum_{m\ge 1} c_m,\qquad a[\pm m]=\tfrac{1}{2}c_m\ (m\ge 1),
\]
yields a symmetric Toeplitz operator $A$ with the above properties and symbol $\phi(\omega)=\sum_{m\ge 1} c_m(\cos m\omega-1)$.
\end{lemma}

\begin{proof}
Since $a\in\ell^1(\mathbb{Z})$ and $a[-n]=a[n]$,
\[
\phi(\omega)=\sum_{n\in\mathbb{Z}} a[n]e^{-i n\omega}
= a[0]+2\sum_{m=1}^{\infty} a[m]\cos(m\omega).
\]
The zero-sum condition gives $a[0]=-2\sum_{m\ge 1}a[m]$, hence
\[
\phi(\omega)=2\sum_{m=1}^{\infty} a[m]\big(\cos(m\omega)-1\big)
=\sum_{m=1}^{\infty} c_m\big(\cos(m\omega)-1\big),\quad c_m:=2a[m]\ge 0.
\]
Absolute summability of $a$ implies $\sum_m c_m<\infty$, ensuring uniform convergence and continuity of $\phi$. The converse follows by reversing the construction.
\end{proof}

\begin{theorem}[Uniqueness of the discrete Gaussian scale space]
\label{thm:ldg-unique}
Among symmetric kernel families, the \emph{discrete Gaussian}
\[
K_s[n]=e^{-\alpha s}\,I_{|n|}(\alpha s),\qquad \alpha>0,
\]
is the unique generalized scaling operator family that satisfies \textbf{(A1)}–\textbf{(A6)}. Equivalently,
\[
\widehat{K}_s(\omega)=\exp\big(\alpha s(\cos\omega-1)\big)
=\exp\big(-2\alpha s\,\sin^2(\omega/2)\big),\qquad \omega\in[-\pi,\pi],
\]
and any other family satisfying the axioms coincides with $K_s$ up to reparameterizing scale $s\mapsto \alpha s$.
\end{theorem}

\begin{proof}
By \textbf{(A1)}, $(\mathcal{T}_s x)[n]=(K_s*x)[n]$ for some $K_s\in\ell^1(\mathbb{Z})$. Let $\widehat{K}_s(\omega)$ be its DTFT. By \textbf{(A2)}–\textbf{(A3)}, for each fixed $\omega$ the map $s\mapsto \widehat{K}_s(\omega)$ is a continuous one-parameter semigroup with $\widehat{K}_0(\omega)=1$, hence there exists a real, even \emph{generator} $\phi(\omega)$ such that
\begin{equation*}
\label{eq:semigroup}
\widehat{K}_s(\omega)=\exp\big(s\,\phi(\omega)\big),\qquad s\ge 0.
\end{equation*}
From \textbf{(A4)}, $\phi(0)=0$. From \textbf{(A6)}, $\phi$ is real and even.

Let $y(\cdot;s)=\mathcal{T}_s x$. Differentiating in $s$,
\[
\partial_s y=A*x,\qquad A:=\partial_s K_s\big|_{s=0},
\]
so that $\widehat{A}(\omega)=\phi(\omega)$. By \textbf{(A4)} and \textbf{(A6)}, $A$ is symmetric with zero row sum. The NELE axiom \textbf{(A5)} at vanishing scale enforces the discrete maximum principle for $A$: off-diagonal coefficients are nonnegative while the diagonal is nonpositive, and rows sum to zero. Thus $A$ fits the hypotheses of \Cref{lem:cm-exist}, and
\begin{equation*}
\label{eq:lk}
\phi(\omega)=\sum_{m=1}^{\infty} c_m\big(\cos(m\omega)-1\big),\qquad c_m\ge 0.
\end{equation*}

We now show that \textbf{(A5)} forces $c_m=0$ for all $m\ge 2$. Consider signals supported on three consecutive sites and apply \textbf{(A5)} at $s=0$ to both local maxima and minima. A standard extremum test yields that any positive coefficient at distance $m\ge 2$ would produce, for sufficiently small $s>0$, an increase at a newly formed off-center extremum before nearest neighbors equilibrate. Hence $c_m=0$ for $m\ge 2$, and
\[
\phi(\omega)=c_1(\cos\omega-1)=:\alpha(\cos\omega-1),\qquad \alpha:=c_1>0.
\]
Substituting into \eqref{eq:semigroup} gives
\[
\widehat{K}_s(\omega)=\exp\big(\alpha s(\cos\omega-1)\big).
\]
Taking the inverse DTFT yields $K_s[n]=e^{-\alpha s}I_{|n|}(\alpha s)$, the discrete Gaussian kernel. This family satisfies \textbf{(A1)}–\textbf{(A6)}; conversely, any other family obeying the axioms has the same generator up to the multiplicative constant $\alpha$, i.e., a reparameterization of scale.
\end{proof}

\newpage

\section{{Properties of Scaling Operator Family}}
\label{sec:scale-op-prop}

Definition~\ref{def:scale-op} characterizes a valid scaling operator through two essential properties: 
(1) non-expansiveness, requiring that applying the operator does not amplify differences between nearby inputs, and 
(2) energy reduction, requiring that coarser scales retain strictly less signal energy over multiplicable scales. 
These properties ensure that scaling progressively smooths the input while preserving stability.

Figure~\ref{fig:appendix-scale-op} reports the empirical behavior of all operator families in \Cref{tab:scale-operators} across seven datasets. 
For every dataset, the induced output differences remain strictly smaller than the input differences, confirming non-expansiveness. 
Likewise, the average energy decreases monotonically over dyadic scales, indicating that each operator consistently removes high-frequency variation as the scale increases. 
Taken together, these results demonstrate that all operator families satisfy both conditions of Definition~\ref{def:scale-op} not only in theory but also in practice, across diverse real-world time series.

\begin{figure*}[h]
\centering
\resizebox{0.8\columnwidth}{!}{%
\begin{minipage}{\columnwidth}

    \begin{subfigure}{0.48\textwidth}
        \centering
        \includegraphics[
        width=\linewidth,
        alt={Figure 8a shows the ETTh1 results. Across scales 1, 2, 4, 8, and 16, most scaling operators reduce both average signal energy and transformed-sequence distance as the scale increases. Moving average preserves much of the signal energy, while the other operators show stronger energy reduction.}
        ]{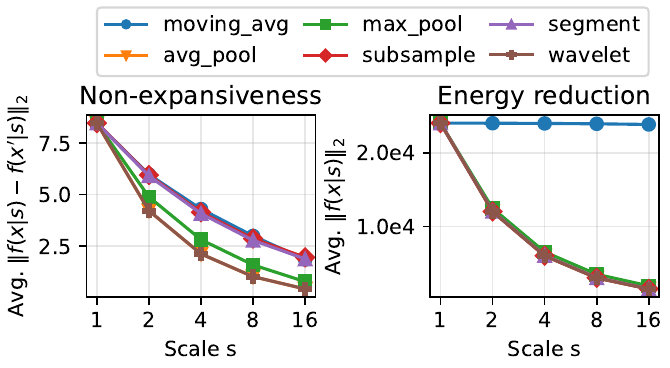}
        \caption{ETTh1}
    \end{subfigure}
    \hfill
    \begin{subfigure}{0.48\textwidth}
        \centering
        \includegraphics[
        width=\linewidth,
        alt={Figure 8b shows the ETTh2 results. The same overall pattern appears as in ETTh1: larger scales produce lower transformed-sequence distances, and most operators reduce energy substantially. Moving average remains comparatively stable in energy, while pooling, subsampling, segmentation, and wavelet decomposition decrease more sharply.}
        ]{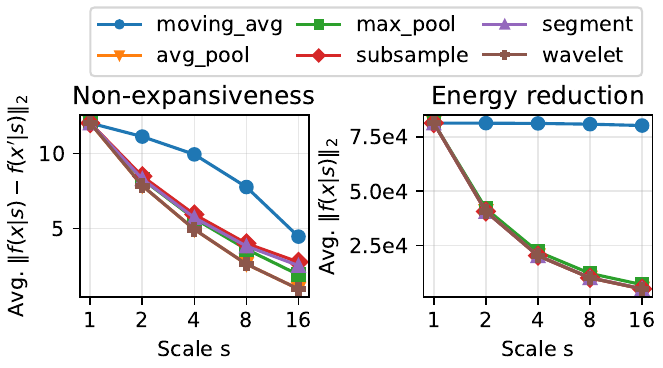}
        \caption{ETTh2}
    \end{subfigure}


    \begin{subfigure}{0.48\textwidth}
        \centering
        \includegraphics[
        width=\linewidth,
        alt={Figure 8c shows the ETTm1 results. Increasing the scale consistently reduces the distance between transformed neighboring sequences. Energy also decreases for most operators, with moving average again changing only slightly compared with the other scaling families.}
        ]{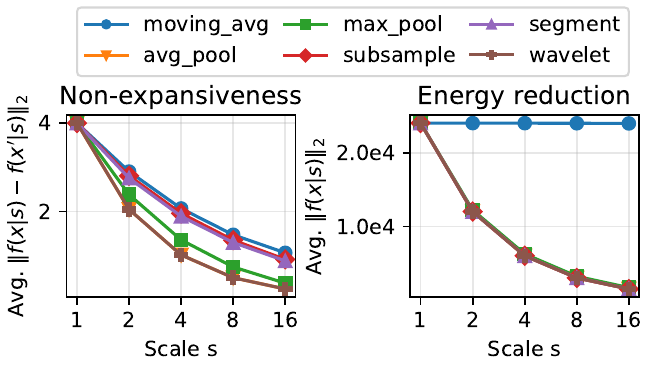}
        \caption{ETTm1}
    \end{subfigure}
    \hfill
    \begin{subfigure}{0.48\textwidth}
        \centering
        \includegraphics[
        width=\linewidth,
        alt={Figure 8d shows the ETTm2 results. All operators become less expansive at larger scales, as indicated by decreasing transformed-sequence distance. Energy reduction is strongest for pooling, subsampling, segmentation, and wavelet decomposition, while moving average stays nearly flat.}
        ]{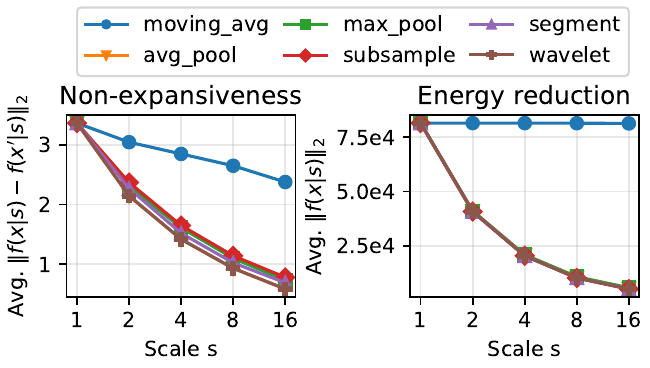}
        \caption{ETTm2}
    \end{subfigure}


    \begin{subfigure}{0.48\textwidth}
        \centering
        \includegraphics[
        width=\linewidth,
        alt={Figure 8e shows the Electricity results. This dataset has much larger absolute energy and distance values than the ETT datasets, but the qualitative behavior is similar: increasing the scale lowers transformed-sequence distance and reduces energy for most operators.}
        ]{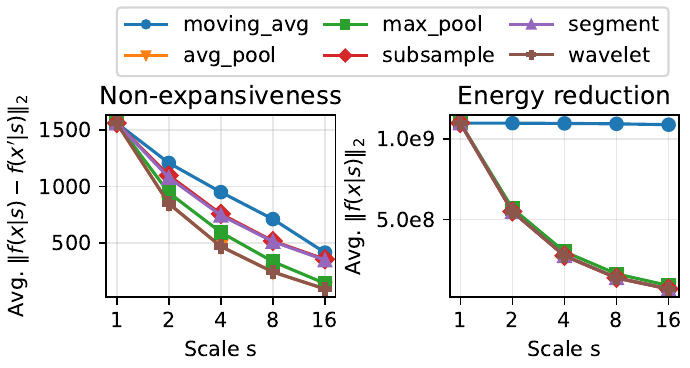}
        \caption{Electricity}
    \end{subfigure}
    \hfill
    \begin{subfigure}{0.48\textwidth}
        \centering
        \includegraphics[
        width=\linewidth,
        alt={Figure 8f shows the Exchange results. Although the absolute values are smaller than in the other datasets, the operators follow the same trend: larger scales yield lower transformed-sequence distances, and most operators reduce average energy with scale.}
        ]{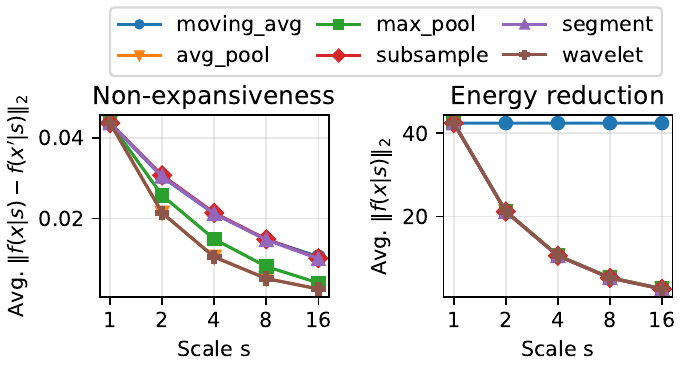}
        \caption{Exchange}
    \end{subfigure}


    \begin{subfigure}{0.48\textwidth}
        \centering
        \includegraphics[
        width=\linewidth,
        alt={Figure 8g shows the Weather results. The operators again show decreasing transformed-sequence distance as scale increases. Energy reduction is evident for most scaling families, while moving average retains comparatively more energy across scales.}
        ]{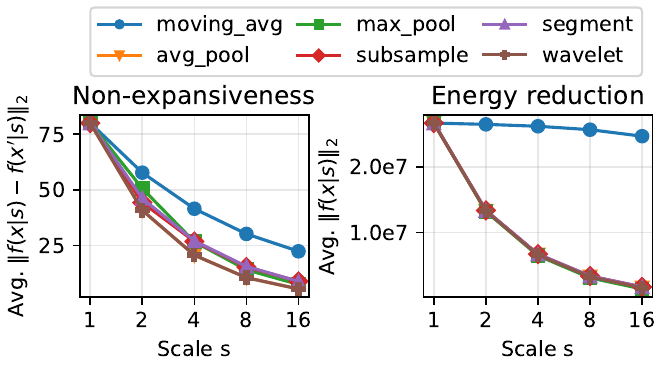}
        \caption{Weather}
    \end{subfigure}
    \hfill
    \begin{subfigure}{0.48\textwidth}
        \centering
    \end{subfigure}

\end{minipage}
}
\caption{Energy reduction and non-expansiveness of the scaling operator families across all datasets.}
    \label{fig:appendix-scale-op}
\end{figure*}

\newpage

\section{Empirical Validation of Theorem~\ref{thm:quantization-gap}}
\label{sec:thm-42}

We additionally provide empirical validation of Theorem~\ref{thm:quantization-gap} across all datasets. In this experiment, the scaling operator family $f(\mathbf{x}|\mathbf{s})$ is instantiated using an extended mean-pooling family, where integer scales recover standard pooling and continuous scales are obtained through interpolation. The results show that forecastability is consistently maximized at non-discrete scales and that the expressivity gap exceeds the theoretical lower bound across all samples.

\vspace{10pt}
\begin{figure*}[h]
\centering
\resizebox{0.8\columnwidth}{!}{%
\begin{minipage}{\columnwidth}

    \begin{subfigure}{0.49\textwidth}
        \centering
        \includegraphics[
        width=\linewidth,
        alt={Figure 9a shows the ETTh1 results. Forecastability reaches its maximum at the non-discrete scale s = 12.4, rather than at one of the discrete scales. The expressivity-gap plot shows all sample points above the lower-bound reference, confirming that the continuous-scale operator is more expressive than the discrete-scale version.}
        ]{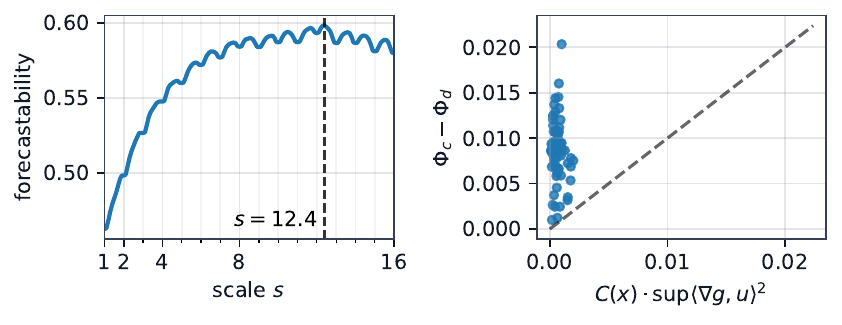}
        \caption[Empirical validation of Theorem 4.2.]{ETTh1}
    \end{subfigure}
    \hfill
    \begin{subfigure}{0.49\textwidth}
        \centering
        \includegraphics[
        width=\linewidth,
        alt={Figure 9b shows the ETTh2 results. The same pattern holds, with forecastability maximized at the non-discrete scale s = 15.6. The observed expressivity gaps remain above the theoretical lower bound for all samples.}
        ]{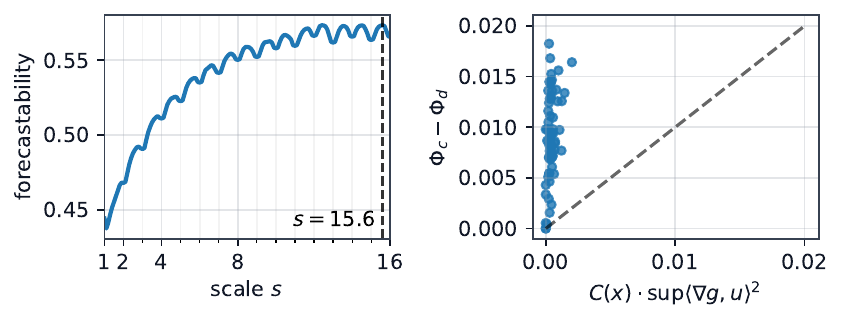}
        \caption{ETTh2}
    \end{subfigure}

    \begin{subfigure}{0.49\textwidth}
        \centering
        \includegraphics[
        width=\linewidth,
        alt={Figure 9c shows the ETTm1 results. Forecastability increases toward a maximum at s = 15.6, indicating that the best scale is again non-discrete. The gap-versus-bound plot supports the theorem by showing that the empirical expressivity gap exceeds the lower bound.}
        ]{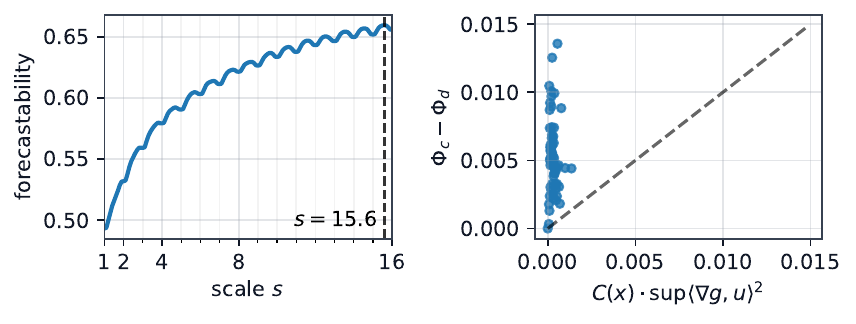}
        \caption{ETTm1}
    \end{subfigure}
    \hfill
    \begin{subfigure}{0.49\textwidth}
        \centering
        \includegraphics[
        width=\linewidth,
        alt={Figure 9d shows the ETTm2 results. Forecastability is highest at the non-discrete scale s = 15.6. The right plot shows that every sample satisfies the theoretical lower-bound condition for the expressivity gap.}
        ]{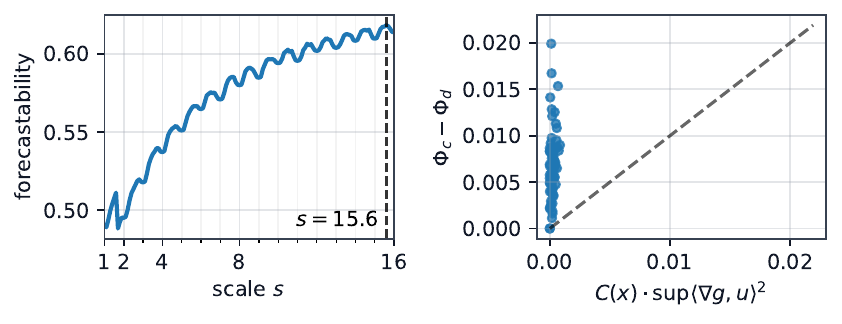}
        \caption{ETTm2}
    \end{subfigure}

    \begin{subfigure}{0.49\textwidth}
        \centering
        \includegraphics[
        width=\linewidth,
        alt={Figure 9e shows the Electricity results. The peak forecastability occurs at s = 11.7, a non-discrete scale. The expressivity-gap plot follows the same validation pattern, with all points lying above the lower-bound reference.}
        ]{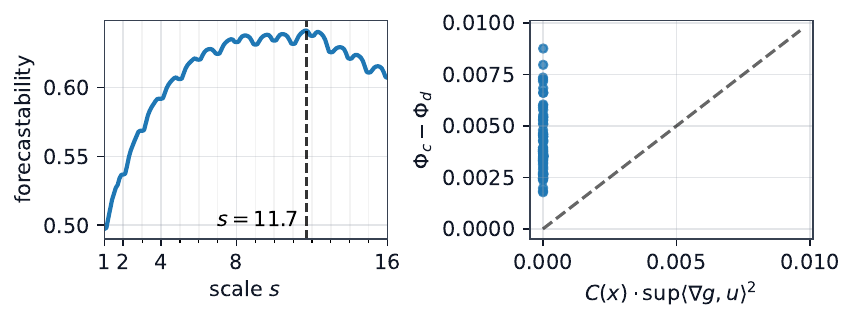}
        \caption{Electricity}
    \end{subfigure}
    \hfill
    \begin{subfigure}{0.49\textwidth}
        \centering
        \includegraphics[
        width=\linewidth,
        alt={Figure 9f shows the Exchange results. Forecastability is maximized at s = 15.6, showing that continuous scales provide a better representation than the fixed discrete scale set. The empirical expressivity gap is consistently larger than the theoretical lower bound.}
        ]{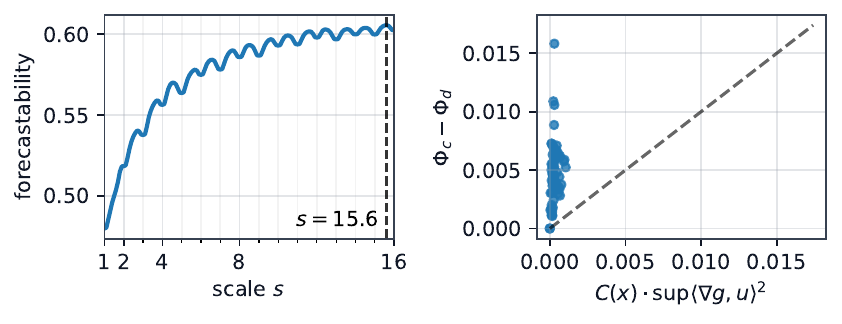}
        \caption{Exchange}
    \end{subfigure}

    \begin{subfigure}{0.49\textwidth}
        \centering
        \includegraphics[
        width=\linewidth,
        alt={Figure 9g shows the Weather results. Forecastability again reaches its maximum at the non-discrete scale s = 15.6. The right plot confirms that the continuous-scale advantage satisfies the lower-bound guarantee across all samples.}
        ]{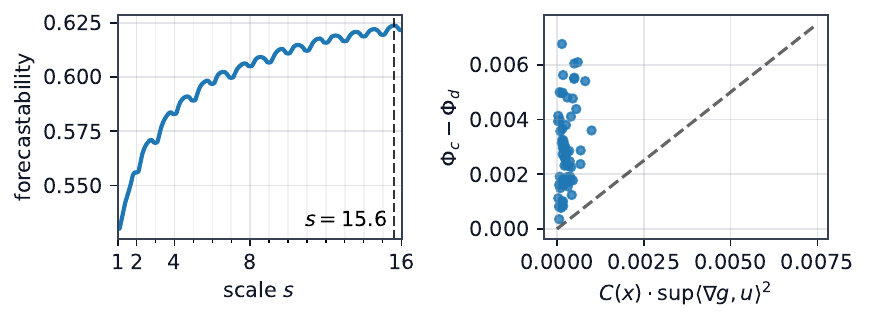}
        \caption{Weather}
    \end{subfigure}

\end{minipage}
}
\vspace{5pt}
\caption[Empirical validation of Theorem 4.2.]{Empirical validation of Theorem~4.2 across all datasets. The scaling operator $f$ is instantiated using an extended mean-pooling family. The left panels show that forecastability is consistently maximized at non-discrete scales, indicating $\Phi_c > \Phi_d$. The right panels show that the expressivity gap $\Phi_c - \Phi_d$ exceeds the theoretical lower bound for all samples.
These results confirm that continuous scales yield strictly higher expressivity than discrete scales, supporting the theorem.}
\label{fig:appendix-theorem42}
\end{figure*}

\newpage

\section{Design Guide for Generalized Scaling Operator Families}
\label{sec:design-guide-ext}

The generalized scaling operator family proposed in Section~\ref{sec:generalized-scale-genearator}
provides a flexible framework for constructing continuous, input-dependent multi-scale
transformations.  
Here we summarize practical guidelines for designing new instances of this family and
illustrate how these principles are instantiated in our implementation.

Classical downsampling operators are defined only at integer scales.  
To make these operators learnable and input-adaptive, we extend the discrete index
$j \in \{1,\dots,J\}$ to a continuous scale variable $s_t \in [1,J]$ at each position~$t$.
Instead of selecting a single integer operator, we blend nearby integer scales through a
convex combination
\[
    f(\vx | s_t)
    = \sum_{j=1}^J w_j(s_t)\, f_j(\vx)_t,
\]
where $f_j$ is the original operator at scale~$j$.  
This convex construction preserves non-expansiveness whenever each $f_j$ is
non-expansive, and it ensures that $f(\cdot | s)$ varies smoothly with respect to~$s$.

To satisfy Definition~\ref{def:scale-op}, the generalized operator must reproduce the
original discrete operator exactly when $s_t$ is an integer.  
This is achieved by choosing weights $\{w_j(s)\}$ such that
\[
   w_j(k) = \mathbf{1}\{j = k\}
   \qquad \text{for all integers } k \in \{1,\dots,J\}.
\]
A convenient construction uses $C^\infty$ bump functions:
\[
   w_j(s) = 
   \frac{\phi(s-j)}{\sum_{k=1}^J \phi(s-k)}, \qquad
   \phi(u) =
   \begin{cases}
      \exp\bigl(-1/(1-u^2)\bigr), & |u| < 1,\\
      0, & \text{otherwise}.
   \end{cases}
\]
At integer $s=k$, the only nonzero bump is $\phi(0)$, which yields $w_k(k)=1$ and
$w_{j\neq k}(k)=0$.  
Thus the generalized operator reduces exactly to the classical operator at integer scales,
ensuring consistency.

Because the bump function $\phi(\cdot)$ is $C^\infty$ and the normalization preserves
smoothness, the weights $w_j(s)$ are differentiable in~$s$ for all non-integer values.
Since the output is a convex combination of the $f_j(\vx)$, the entire operator $f(
\vx|s)$
is differentiable in~$s$, enabling backpropagation through the scale field and supporting
input-adaptive scale prediction via a learnable scale head.

\newpage

\section{Experimental Details}
\label{sec:exp-detail}

For the evaluation of forecasting models, we follow the protocol used in TimesNet~\cite{wu2023timesnet}. 
For \textbf{long-term forecasting} (see Table~\ref{tab:long-term}), 
we report the mean square error (MSE) and mean absolute error (MAE). 
For \textbf{short-term forecasting} (see Table~\ref{tab:short-term}), 
we adopt symmetric mean absolute percentage error (SMAPE), mean absolute scaled error (MASE), 
and overall weighted average (OWA). These metrics are defined as follows:
\begin{align}
\text{SMAPE} &= \frac{200}{H} \sum_{i=1}^{T} \frac{|\mX_i - \hat{\mX}_i|}{|\mX_i| + |\hat{\mX}_i|},\\
\text{MASE} &= \frac{1}{T} \sum_{i=1}^{T} \frac{|\mX_i - \hat{\mX}_i|}{\frac{1}{T-m} \sum_{j=m+1}^{T} |\mX_j - \mX_{j-m}|}, \\
\text{OWA}  &= \tfrac{1}{2} \left[ \frac{\text{SMAPE}}{\text{SMAPE}_{\text{Naive2}}} + \frac{\text{MASE}}{\text{MASE}_{\text{Naive2}}} \right],
\end{align}
where $m$ denotes the seasonal periodicity of the data, and $\mX \in \mathbb{R}^{T \times C}$ and 
$\hat{\mX} \in \mathbb{R}^{T \times C}$ represent the ground truth and predictions for $T$ future time steps with $C$ dimensions. 

For methods that do not originally provide results on the M4 dataset, we use official implementations and conduct a controlled hyperparameter search for fair comparison. All models are implemented in PyTorch~\citep{paszke2019pytorch}, with the modified Bessel function implemented via SciPy, and all experiments are executed on a single \textbf{NVIDIA RTX A6000} GPU. Dataset statistics and detailed experimental configurations are provided in Tables~\ref{tab:data-statistic} and~\ref{tab:hyperparameter-configuration}, respectively.

\begin{table}[h]
\centering
\caption{
Dataset descriptions for long-term and short-term forecasting benchmarks. 
The dataset size is organized as (Train, Validation, Test). For long-term tasks, 
we adopt multivariate datasets covering diverse domains. For short-term tasks, we follow the official M4 benchmark, 
which consists of univariate series at different temporal frequencies.
Following TimeMixer++~\citep{wang2025timemixer}, forecastability is measured as one minus spectral entropy~\citep{goerg2013forecastable}, with higher values indicating better predictability.
}
\label{tab:data-statistic}
\resizebox{\linewidth}{!}{
\begin{tabular}{l|l|c|c|c|l|l|c}
\toprule
\textbf{Tasks} & \textbf{Dataset} & \textbf{Dim} & \textbf{Series Length} & \textbf{Dataset Size} & \textbf{Domain} & \textbf{Frequency} & \textbf{Forecast.} \\
\midrule
\multirow{8}{*}{\shortstack{Forecasting \\ (Long-term)}} 
& ETTm1        & 7   & \{96, 192, 336, 720\} & (34465, 11521, 11521) & Temperature      & 15 min & 0.46 \\
& ETTm2        & 7   & \{96, 192, 336, 720\} & (34465, 11521, 11521) & Temperature & 15 min & 0.55 \\
& ETTh1        & 7   & \{96, 192, 336, 720\} & (8545, 2881, 2881)    & Temperature      & Hourly  & 0.38 \\
& ETTh2        & 7   & \{96, 192, 336, 720\} & (8545, 2881, 2881)    & Temperature      & Hourly  & 0.45 \\
& Electricity  & 321 & \{96, 192, 336, 720\} & (18317, 2633, 5261)   & Electricity      & Hourly  & 0.77 \\
& Traffic      & 862 & \{96, 192, 336, 720\} & (12185, 1757, 3509)   & Transportation   & Hourly  & 0.68 \\
& Exchange     & 8   & \{96, 192, 336, 720\} & (5120, 665, 1422)     & Exchange rate          & Daily   & 0.41 \\
& Weather      & 21  & \{96, 192, 336, 720\} & (36792, 5271, 10540)  & Weather          & 10 min  & 0.75 \\
\midrule
\multirow{6}{*}{\shortstack{Forecasting \\ (Short-term)}} 
& M4-Yearly     & 1 & 6  & (23000, 0, 23000) & Demographic & Yearly    & 0.43 \\
& M4-Quarterly  & 1 & 8  & (24000, 0, 24000) & Finance     & Quarterly & 0.47 \\
& M4-Monthly    & 1 & 18 & (48000, 0, 48000) & Industry    & Monthly   & 0.44 \\
& M4-Weekly     & 1 & 13 & (359, 0, 359)     & Macro       & Weekly    & 0.43 \\
& M4-Daily      & 1 & 14 & (4227, 0, 4227)   & Micro       & Daily     & 0.44 \\
& M4-Hourly     & 1 & 48 & (414, 0, 414)     & Other       & Hourly    & 0.46 \\
\bottomrule
\end{tabular}
}
\vspace{-5pt}
\end{table}
\begin{table}[h]
\centering
\caption[Experiment configurations.]{Experiment configurations of \method\ across datasets. 
All experiments adopt the ADAM~\citep{kinga2015method} optimizer. 
We report the model dimension ($d_{\text{model}}$), 
initial learning rate (LR), loss function, batch size, and training epochs.}
\label{tab:hyperparameter-configuration}
\resizebox{0.47\linewidth}{!}{
\begin{tabular}{c|cr|ccc}
\toprule
\textbf{Dataset} & $d_{\text{model}}$ & LR & Loss & Batch Size & Epochs \\
\midrule
ETTh1       & 32 & 0.0005 & MSE   & 32 & 10 \\
ETTh2       & 32 & 0.0005 & MSE   & 32 & 10 \\
ETTm1       & 8  & 0.02   & MSE   & 32 & 10 \\
ETTm2       & 16 & 0.0005 & MSE   & 32 & 10 \\
Exchange    & 16 & 0.0001 & MSE   & 32 & 10 \\
Electricity & 16 & 0.01   & MSE   & 32 & 10 \\
Traffic     & 16 & 0.01   & MSE   & 32 & 10 \\
Weather     & 8  & 0.005  & MSE   & 32 & 10 \\
M4          & 16 & 0.01   & SMAPE & 32 & 10 \\
\bottomrule
\end{tabular}
}
\end{table}

\newpage

\section{Full Long-Term Forecasting Results}
\label{sec:full}

We report the full long-term forecasting results across all datasets and prediction horizons in Table~\ref{tab:long-term-full}. Overall, \method achieves the best performance in 55 out of 80 settings and ranks second-best in 19 settings.
The improvements are consistently observed across diverse datasets and forecasting horizons, demonstrating the effectiveness of principled multi-scale modeling for diverse long-horizon forecasting tasks.

\begin{table*}[h]
\centering
\caption{Full long-term forecasting results across eight datasets with horizons $T \in \{96,192,336,720\}$ and input length fixed at 96. \method achieves the smallest forecasting errors in 55 out of 80 evaluation settings and the second-best in 19 cases.}
\resizebox{\textwidth}{!}{
\LARGE
\begin{tabular}{cc|cc|cc|cc|cc|cc|cc|cc|cc|cccc}
\toprule
\multicolumn{2}{c|}{Method}
 & \multicolumn{2}{c|}{\begin{tabular}[c]{@{}c@{}}\method\\ (Ours)\end{tabular}} 
 & \multicolumn{2}{c|}{\begin{tabular}[c]{@{}c@{}}AMD\\ \citeyearpar{hu2025adaptive}\end{tabular}} 
 & \multicolumn{2}{c|}{\begin{tabular}[c]{@{}c@{}}MultiPatch.\\ \citeyearpar{naghashi2025multiscale}\end{tabular}} 
 & \multicolumn{2}{c|}{\begin{tabular}[c]{@{}c@{}}WPMixer\\ \citeyearpar{murad2025wpmixer}\end{tabular}} 
 & \multicolumn{2}{c|}{\begin{tabular}[c]{@{}c@{}}TimeMixer\\ \citeyearpar{wang2024timemixer}\end{tabular}} 
 & \multicolumn{2}{c|}{\begin{tabular}[c]{@{}c@{}}MSGNet\\ \citeyearpar{cai2024msgnet}\end{tabular}} 
 & \multicolumn{2}{c|}{\begin{tabular}[c]{@{}c@{}}MICN\\ \citeyearpar{wang2023micn}\end{tabular}} 
 & \multicolumn{2}{c|}{\begin{tabular}[c]{@{}c@{}}TimesNet\\ \citeyearpar{wu2023timesnet}\end{tabular}} 
 & \multicolumn{2}{c}{\begin{tabular}[c]{@{}c@{}}Pyra.\\ \citeyearpar{liu2022pyraformer}\end{tabular}} \\
\cmidrule(lr){1-20}
 \multicolumn{2}{c|}{Metric} & MSE & MAE & MSE & MAE & MSE & MAE & MSE & MAE & MSE & MAE & MSE & MAE & MSE & MAE & MSE & MAE & MSE & MAE \\
\midrule
\multirow{5}{*}{\rotatebox{90}{Weather}} 
 & 96   & \best{0.160} & \best{0.204} & 0.182 & 0.227 & 0.172 & 0.211 & 0.164 & 0.210 & 0.166 & 0.214 & \second{0.161} & \second{0.209} & 0.192 & 0.250 & 0.172 & 0.221 & 0.195 & 0.281 \\
 & 192  & \second{0.209} & \best{0.248} & 0.231 & 0.266 & 0.218 & 0.254 & \second{0.209} & \second{0.250} & \best{0.208} & 0.251 & 0.217 & 0.257 & 0.233 & 0.289 & 0.225 & 0.265 & 0.245 & 0.322 \\
 & 336  & 0.270 & \second{0.293} & 0.283 & 0.302 & 0.275 & 0.296 & \best{0.264} & \best{0.290} & \second{0.265} & \second{0.293} & 0.280 & 0.303 & 0.283 & 0.332 & 0.289 & 0.309 & 0.307 & 0.365 \\
 & 720  & 0.348 & \second{0.345} & 0.357 & 0.350 & 0.355 & 0.348 & \best{0.344} & \best{0.343} & \second{0.345} & \second{0.345} & 0.373 & 0.362 & 0.354 & 0.388 & 0.361 & 0.355 & 0.394 & 0.420 \\
 & Avg  & 0.247 & \best{0.273} & 0.263 & 0.286 & 0.255 & 0.278 & \best{0.245} & \best{0.273} & \second{0.246} & \second{0.276} & 0.258 & 0.283 & 0.266 & 0.315 & 0.262 & 0.288 & 0.285 & 0.347 \\
\midrule
\multirow{5}{*}{\rotatebox{90}{Electricity}}
 & 96   & \best{0.146} & \best{0.241} & 0.187 & 0.269 & 0.173 & 0.259 & 0.167 & 0.259 & \second{0.157} & \second{0.248} & 0.169 & 0.281 & 0.171 & 0.284 & 0.165 & 0.268 & 0.283 & 0.377 \\
 & 192  & \best{0.163} & \best{0.257} & 0.191 & 0.274 & 0.181 & 0.267 & 0.179 & 0.268 & \second{0.169} & \second{0.260} & 0.189 & 0.298 & 0.178 & 0.290 & 0.182 & 0.284 & 0.296 & 0.391 \\
 & 336  & \best{0.179} & \best{0.273} & 0.206 & 0.290 & 0.199 & 0.285 & 0.197 & 0.288 & \second{0.187} & \second{0.277} & 0.201 & 0.310 & 0.189 & 0.301 & 0.198 & 0.299 & 0.306 & 0.401 \\
 & 720  & \second{0.213} & \best{0.304} & 0.248 & 0.323 & 0.239 & 0.318 & 0.232 & 0.315 & 0.227 & \second{0.312} & 0.238 & 0.340 & \best{0.208} & 0.318 & 0.231 & 0.325 & 0.305 & 0.393 \\
 & Avg  & \best{0.175} & \best{0.269} & 0.208 & 0.289 & 0.198 & 0.282 & 0.194 & 0.282 & \second{0.185} & \second{0.274} & 0.199 & 0.307 & 0.187 & 0.298 & 0.194 & 0.294 & 0.298 & 0.390 \\
\midrule
\multirow{5}{*}{\rotatebox{90}{Traffic}}
 & 96   & \best{0.431} & \best{0.288} & 0.544 & 0.345 & \second{0.471} & 0.318 & 0.528 & 0.347 & 0.477 & \second{0.309} & 0.599 & 0.353 & 0.516 & \second{0.309} & 0.590 & 0.318 & 0.678 & 0.384 \\
 & 192  & \best{0.444} & \best{0.296} & 0.527 & 0.335 & \second{0.480} & 0.319 & 0.511 & 0.337 & 0.488 & \second{0.312} & 0.634 & 0.372 & 0.535 & 0.317 & 0.614 & 0.327 & 0.672 & 0.377 \\
 & 336  & \best{0.461} & \best{0.303} & 0.538 & 0.339 & \second{0.499} & 0.329 & 0.519 & 0.337 & 0.506 & \second{0.319} & 0.663 & 0.391 & 0.548 & 0.322 & 0.640 & 0.342 & 0.681 & 0.381 \\
 & 720  & \best{0.494} & \best{0.320} & 0.573 & 0.358 & 0.537 & 0.351 & 0.548 & 0.350 & \second{0.535} & \second{0.332} & 0.721 & 0.420 & 0.574 & \second{0.332} & 0.662 & 0.350 & 0.709 & 0.395 \\
 & Avg  & \best{0.458} & \best{0.302} & 0.546 & 0.344 & \second{0.497} & 0.329 & 0.527 & 0.343 & 0.501 & \second{0.318} & 0.654 & 0.384 & 0.543 & 0.320 & 0.627 & 0.334 & 0.685 & 0.384 \\
\midrule
\multirow{5}{*}{\rotatebox{90}{Exchange}}
 & 96   & \second{0.084} & 0.204 & \best{0.083} & \best{0.201} & 0.089 & 0.208 & 0.086 & \second{0.202} & 0.090 & 0.210 & 0.104 & 0.230 & 0.093 & 0.226 & 0.112 & 0.242 & 0.630 & 0.645 \\
 & 192  & \best{0.174} & \second{0.297} & \second{0.175} & \second{0.297} & 0.187 & 0.308 & 0.176 & \best{0.296} & 0.185 & 0.305 & 0.200 & 0.322 & 0.184 & 0.331 & 0.214 & 0.334 & 0.935 & 0.782 \\
 & 336  & \best{0.322} & \best{0.411} & \second{0.326} & \second{0.412} & 0.357 & 0.436 & 0.343 & 0.421 & 0.351 & 0.428 & 0.394 & 0.460 & \best{0.322} & 0.443 & 0.379 & 0.452 & 1.204 & 0.873 \\
 & 720  & \second{0.833} & \best{0.687} & 0.847 & \second{0.693} & 0.913 & 0.724 & 0.900 & 0.712 & 0.911 & 0.713 & 1.027 & 0.772 & \best{0.780} & 0.694 & 0.961 & 0.746 & 1.956 & 1.117 \\
 & Avg  & \second{0.353} & \best{0.400} & 0.358 & \second{0.401} & 0.387 & 0.419 & 0.376 & 0.408 & 0.384 & 0.414 & 0.431 & 0.446 & \best{0.345} & 0.424 & 0.416 & 0.443 & 1.181 & 0.854 \\
\midrule
\multirow{5}{*}{\rotatebox{90}{ETTh1}}
 & 96   & \second{0.379} & \best{0.393} & 0.385 & \second{0.396} & \best{0.377} & 0.397 & 0.382 & 0.404 & 0.381 & 0.398 & 0.398 & 0.418 & 0.425 & 0.435 & 0.419 & 0.432 & 0.701 & 0.630 \\
 & 192  & \second{0.430} & \best{0.425} & 0.437 & \best{0.425} & \best{0.427} & 0.428 & 0.438 & \second{0.427} & 0.441 & 0.434 & 0.444 & 0.445 & 0.505 & 0.484 & 0.474 & 0.464 & 0.850 & 0.713 \\
 & 336  & 0.481 & \second{0.446} & 0.480 & \best{0.445} & \best{0.469} & 0.449 & 0.499 & 0.464 & \second{0.475} & 0.449 & 0.484 & 0.471 & 0.606 & 0.556 & 0.499 & 0.475 & 0.960 & 0.777 \\
 & 720  & \best{0.480} & \best{0.468} & \second{0.485} & \second{0.469} & 0.499 & 0.485 & 0.487 & 0.471 & 0.522 & 0.494 & 0.509 & 0.498 & 0.752 & 0.648 & 0.529 & 0.500 & 1.005 & 0.803 \\
 & Avg  & \best{0.443} & \best{0.433} & \second{0.447} & \second{0.434} & \best{0.443} & 0.440 & 0.451 & 0.442 & 0.455 & 0.444 & 0.459 & 0.458 & 0.572 & 0.531 & 0.480 & 0.468 & 0.879 & 0.731 \\

\midrule

\multirow{5}{*}{\rotatebox{90}{ETTh2}}
 & 96   & \best{0.289} & \best{0.339} & \second{0.291} & \second{0.340} & 0.293 & 0.347 & \second{0.291} & 0.342 & 0.296 & 0.348 & 0.327 & 0.369 & 0.358 & 0.405 & 0.327 & 0.369 & 1.439 & 0.922 \\
 & 192  & \best{0.369} & \second{0.394} & 0.373 & \best{0.391} & \second{0.372} & 0.396 & 0.376 & 0.397 & 0.375 & \second{0.395} & 0.408 & 0.417 & 0.497 & 0.484 & 0.404 & 0.412 & 5.640 & 1.894 \\
 & 336  & \best{0.415} & \best{0.427} & \second{0.416} & \best{0.427} & 0.420 & \second{0.431} & 0.435 & 0.439 & 0.430 & 0.438 & 0.429 & 0.439 & 0.618 & 0.552 & 0.459 & 0.456 & 4.800 & 1.844 \\
 & 720  & \second{0.431} & \second{0.446} & \best{0.424} & \best{0.441} & 0.431 & 0.450 & 0.459 & 0.462 & 0.456 & 0.460 & 0.446 & 0.459 & 0.856 & 0.666 & 0.451 & 0.459 & 4.466 & 1.824 \\
 & Avg  & \best{0.376} & \second{0.402} & \second{0.376} & \best{0.400} & 0.379 & 0.406 & 0.390 & 0.410 & 0.389 & 0.410 & 0.402 & 0.421 & 0.582 & 0.527 & 0.410 & 0.424 & 4.086 & 1.621 \\
\midrule

\multirow{5}{*}{\rotatebox{90}{ETTm1}}
 & 96   & 0.323 & \second{0.359} & 0.330 & 0.365 & \best{0.319} & \best{0.358} & \second{0.322} & \best{0.358} & \second{0.322} & \second{0.359} & 0.328 & 0.370 & \second{0.322} & 0.373 & 0.334 & 0.374 & 0.592 & 0.514 \\
 & 192  & \best{0.360} & \best{0.381} & 0.372 & 0.384 & 0.363 & 0.385 & \second{0.361} & \second{0.382} & 0.364 & 0.384 & 0.371 & 0.395 & \second{0.361} & 0.402 & 0.404 & 0.409 & 0.645 & 0.568 \\
 & 336  & \second{0.392} & \second{0.404} & 0.406 & 0.405 & 0.398 & 0.410 & \best{0.389} & \best{0.402} & 0.397 & 0.407 & 0.410 & 0.419 & 0.409 & 0.437 & 0.418 & 0.421 & 0.776 & 0.643 \\
 & 720  & \best{0.455} & \second{0.442} & 0.471 & \best{0.440} & 0.460 & 0.448 & 0.469 & 0.447 & \second{0.456} & 0.444 & 0.494 & 0.465 & 0.506 & 0.498 & 0.484 & 0.455 & 0.936 & 0.730 \\
 & Avg  & \best{0.383} & \best{0.397} & 0.395 & 0.399 & \second{0.385} & 0.400 & \second{0.385} & \best{0.397} & \second{0.385} & \second{0.398} & 0.401 & 0.412 & 0.399 & 0.427 & 0.410 & 0.415 & 0.737 & 0.614 \\
\midrule

\multirow{5}{*}{\rotatebox{90}{ETTm2}}
 & 96   & \best{0.174} & \best{0.257} & 0.183 & 0.267 & 0.177 & 0.259 & \second{0.175} & \best{0.257} & 0.176 & \second{0.258} & 0.179 & 0.263 & 0.186 & 0.283 & 0.187 & 0.266 & 0.387 & 0.464 \\
 & 192  & \best{0.239} & \best{0.299} & 0.246 & 0.306 & 0.243 & 0.304 & \second{0.240} & \best{0.299} & 0.241 & \second{0.302} & 0.250 & 0.308 & 0.278 & 0.352 & 0.257 & 0.309 & 0.676 & 0.623 \\
 & 336  & \best{0.296} & \best{0.337} & 0.305 & \second{0.342} & 0.305 & 0.346 & \second{0.304} & \second{0.342} & \second{0.304} & 0.345 & 0.311 & 0.345 & 0.405 & 0.437 & 0.322 & 0.349 & 1.196 & 0.836 \\
 & 720  & \best{0.394} & \best{0.394} & 0.404 & \second{0.397} & 0.407 & 0.404 & \second{0.398} & \second{0.397} & 0.405 & 0.402 & 0.416 & 0.406 & 0.546 & 0.515 & 0.427 & 0.409 & 3.588 & 1.460 \\
 & Avg  & \best{0.276} & \best{0.322} & 0.285 & 0.328 & 0.283 & 0.328 & \second{0.279} & \second{0.324} & 0.281 & 0.327 & 0.289 & 0.330 & 0.354 & 0.397 & 0.298 & 0.333 & 1.462 & 0.846 \\

\bottomrule
\label{tab:long-term-full}
\end{tabular}
}
\vspace{-15pt}
\end{table*}

\newpage

\section{Efficiency Evaluation on Additional Datasets}
\label{sec:efficiency}

We extend the efficiency analysis to more complex datasets, including Traffic and Electricity. As shown in Figure~\ref{fig:efficiency-v2}, the memory footprint of \method increases with the number of time-series variables because of its channel-independent design. In contrast, methods such as AMD jointly model cross-variable dependencies, which can lead to lower memory usage in these cases.
Nevertheless, \method achieves substantially faster training due to its fully parallel formulation. In particular, it is 25.0× faster on Traffic and 22.5× faster on Electricity compared to AMD.

\begin{figure*}[h]
  \centering
\begin{subfigure}[t]{0.45\textwidth}
\centering
\includegraphics[
width=\linewidth,
alt={Figure 10a shows the efficiency analysis on the Traffic dataset under the predict-720 setting. SiGMA achieves the lowest MSE with 3.8 GB memory usage and 116 milliseconds per training iteration. Compared with AMD, which uses 2.8 GB memory and takes 2897 milliseconds per iteration, SiGMA trains 25.0 times faster while achieving better accuracy. WPMixer is faster at 70 milliseconds per iteration and uses 1.5 GB memory, but has higher MSE.}
]{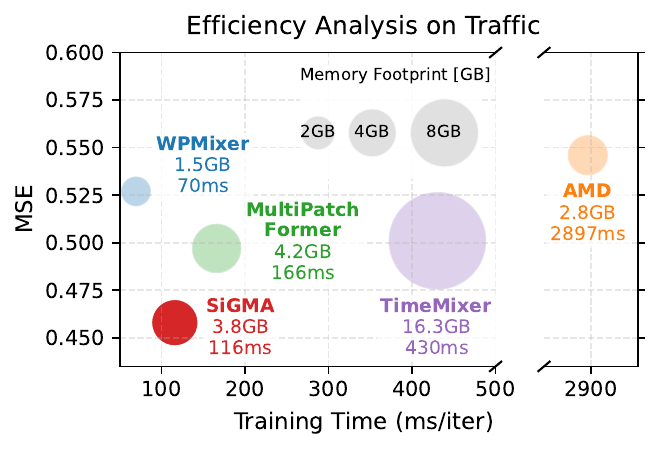}
\caption{Traffic}
\end{subfigure}
\hspace{10pt}
\begin{subfigure}[t]{0.45\textwidth}
\centering
\includegraphics[
width=\linewidth,
alt={Figure 10b shows the efficiency analysis on the Electricity dataset under the predict-720 setting. SiGMA achieves the lowest MSE with 1.4 GB memory usage and 47 milliseconds per training iteration. Compared with AMD, which uses 1.1 GB memory and takes 1060 milliseconds per iteration, SiGMA trains 22.5 times faster while achieving better accuracy. WPMixer is faster at 27 milliseconds per iteration and uses 0.6 GB memory, but has higher MSE.}
]{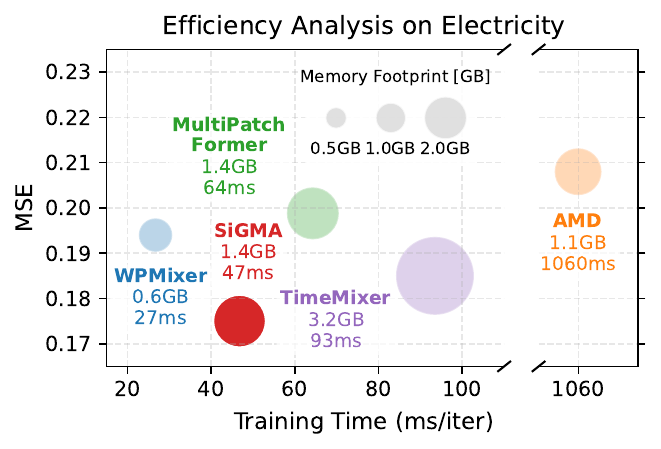}
\caption{Electricity}
\end{subfigure}
  \caption{Efficiency analysis on Traffic and Electricity with the predict-720 setting. \method
delivers the best accuracy while maintaining competitive training time and memory usage, showing
robustness on more complex datasets.}
  \vspace{-5pt}
  \label{fig:efficiency-v2}
\end{figure*}


\section{Computational Complexity and Efficient Application of LDG}
\label{sec:ldg-complexity}

The LDG operator adopts a distance-indexed parameterization, e.g., we learn a parameter vector $\vs \in \mathbb{R}^{L}$ indexed by pairwise distance $d = |i-j|$, and construct the kernel as
\begin{equation}
K_{ij} = e^{-s_d} I_d(s_d),
\end{equation}
where $I_d(\cdot)$ denotes the modified Bessel function of the first kind. 
As a result, each entry depends only on the relative distance $d$, implying that the induced kernel matrix is symmetric Toeplitz. In practice, we compute the corresponding distance-aware kernel and apply the operator through matrix multiplication. 
More generally, the LDG operator can be interpreted as a one-dimensional convolution with kernel coefficients $K_d$, leading to several implementation regimes:
\begin{itemize}[nosep, topsep=2pt]
    \item \textbf{Dense matrix multiplication:} $O(L^2)$
    \item \textbf{Truncated convolution:} $O(LW)$
    \item \textbf{FFT-based Toeplitz multiplication:} $O(L\log L)$~\citep{kressner2018fast}
\end{itemize}
Here, $W$ denotes the effective kernel support. A standard definition based on tail mass is
\begin{equation}
W(\epsilon)=
\min \left\{
w:
\sum_{|d|>w} K_d
\le
\epsilon \sum_d K_d
\right\},
\end{equation}
where $\epsilon$ is a user-specified tolerance~\citep{greengard1991fast}.

We report the computational time and memory usage of different LDG implementations in Figure~\ref{fig:ldg_cost}. The results show that truncated convolution exhibits the most favorable empirical scaling behavior, achieving approximately linear growth in both runtime and memory usage. By contrast, dense multiplication becomes increasingly expensive as $L$ grows, while FFT-based implementations incur substantial memory overhead and do not outperform simpler methods in the moderate-length regime.
In practice, FFT-based implementations involve additional overhead from zero-padding to length $2L-1$, complex-valued arithmetic, and multiple intermediate buffers (e.g., padded tensors and spectral-domain representations). Consequently, despite its favorable asymptotic complexity, FFT may be slower and more memory-intensive than simpler alternatives when $L$ is relatively small.

Overall, among the available strategies, truncated convolution provides the most favorable trade-off in practical settings, achieving near-linear scaling in both runtime and memory while effectively exploiting the strong locality induced by the LDG kernel.

\begin{figure}[h]
  \centering
  \begin{subfigure}{0.33\columnwidth}
    \centering
    \includegraphics[
    width=\linewidth,
    alt={Figure 11a shows computational time for three LDG implementation strategies as the input length L increases from 96 to 720. Dense multiplication grows from about 0.25 ms to 1.7 ms, FFT grows from about 0.3 ms to 1.35 ms, and truncated convolution grows more slowly from about 0.2 ms to 0.45 ms. Truncated convolution has the most favorable runtime scaling across all input lengths.}
    ]{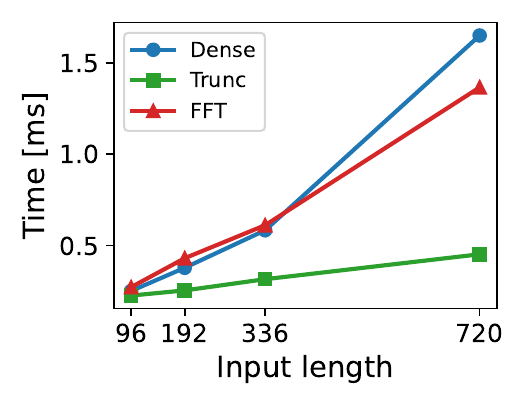}
    \caption{Computational Time}
  \end{subfigure}
  \hspace{1pt}
  \begin{subfigure}{0.33\columnwidth}
    \centering
    \includegraphics[
    width=\linewidth,
    alt={Figure 11b shows memory usage for the same three LDG implementation strategies as the input length L increases from 96 to 720. FFT has the largest memory growth, increasing from about 25 MB to 220 MB. Dense multiplication increases more moderately, reaching about 40 MB at L = 720, while truncated convolution remains the lowest, growing only from about 10 MB to 40 MB. Truncated convolution therefore provides the best memory-efficiency trend among the three strategies.}]{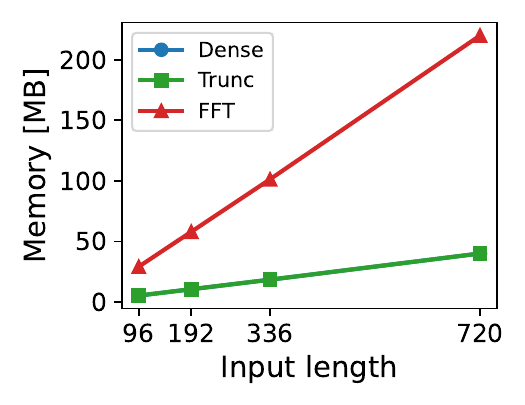}
    \caption{Memory}
  \end{subfigure}

  \caption{Empirical scaling analysis of the LDG operator under different implementation strategies. We report computational time (left) and memory usage (right) as functions of the input length $L$.}
  \vspace{-5pt}
  \label{fig:ldg_cost}
\end{figure}


\section{Error Bars}

To assess the robustness of our experiments, we report the mean performance and standard deviation for \method and the second-best baselines in \Cref{tab:errorbar-long,tab:errorbar-short}. 
On the long-term forecasting benchmarks, \method achieves lower average MSE and MAE than AMD on nearly all datasets, while maintaining sufficiently small standard deviations.

On the short-term M4 benchmark, the averaged SMAPE, MASE, and OWA scores of \method remain consistently better, and the corresponding standard deviations are of similar or smaller magnitude. 
Taken together, these error-bar analyses confirm that the gains reported by \method are stable across runs and not attributable to random fluctuations.

\begin{table}[h]
\centering
\caption{Standard deviation for \method and the second-best method (AMD) on long-term forecasting datasets.}
\label{tab:errorbar-long}
\small
\resizebox{0.6\linewidth}{!}{
\begin{tabular}{l|cc|cc}
\toprule
Method & \multicolumn{2}{c|}{\best{\method}} & \multicolumn{2}{c}{\second{AMD}~\citeyearpar{hu2025adaptive}} \\
\midrule
Dataset & MSE & MAE & MSE & MAE \\
\midrule
Weather        
& 0.247 {\tiny$\pm$ 0.003} & 0.273 {\tiny$\pm$ 0.002}
& 0.283 {\tiny$\pm$ 0.001} & 0.282 {\tiny$\pm$ 0.001} \\

Electricity    
& 0.175 {\tiny$\pm$ 0.001} & 0.269 {\tiny$\pm$ 0.001}
& 0.208 {\tiny$\pm$ 0.000} & 0.303 {\tiny$\pm$ 0.000} \\

Traffic        
& 0.458 {\tiny$\pm$ 0.001} & 0.302 {\tiny$\pm$ 0.001}
& 0.546 {\tiny$\pm$ 0.000} & 0.344 {\tiny$\pm$ 0.000} \\

Exchange       
& 0.353 {\tiny$\pm$ 0.001} & 0.400 {\tiny$\pm$ 0.001}
& 0.358 {\tiny$\pm$ 0.009} & 0.401 {\tiny$\pm$ 0.002} \\

ETTh1          
& 0.443 {\tiny$\pm$ 0.004} & 0.433 {\tiny$\pm$ 0.002}
& 0.447 {\tiny$\pm$ 0.004} & 0.434 {\tiny$\pm$ 0.003} \\

ETTh2          
& 0.376 {\tiny$\pm$ 0.003} & 0.402 {\tiny$\pm$ 0.002}
& 0.376 {\tiny$\pm$ 0.022} & 0.400 {\tiny$\pm$ 0.005} \\

ETTm1          
& 0.383 {\tiny$\pm$ 0.003} & 0.397 {\tiny$\pm$ 0.002}
& 0.395 {\tiny$\pm$ 0.003} & 0.399 {\tiny$\pm$ 0.001} \\

ETTm2          
& 0.276 {\tiny$\pm$ 0.001} & 0.322 {\tiny$\pm$ 0.001}
& 0.285 {\tiny$\pm$ 0.043} & 0.328 {\tiny$\pm$ 0.012} \\
\bottomrule
\end{tabular}}
\end{table}
\begin{table}[h]
\centering
\caption{Standard deviation for \method and the second-best method (MultiPatchFormer) on the short-term forecasting dataset (M4).}
\label{tab:errorbar-short}
\resizebox{0.8\linewidth}{!}{
\begin{tabular}{l|ccc|ccc}
\toprule
Method
& \multicolumn{3}{c|}{\best{\method}} 
& \multicolumn{3}{c}{\second{MultiPatchFormer}~\citeyearpar{naghashi2025multiscale}}\\
\midrule
Dataset & SMAPE & MASE & OWA & SMAPE & MASE & OWA \\
\midrule
Yearly     & 13.314 {\tiny$\pm$ 0.022} & 2.989 {\tiny$\pm$ 0.015} & 0.783 {\tiny$\pm$ 0.002} &
              13.296 {\tiny$\pm$ 0.012} & 3.009 {\tiny$\pm$ 0.008} & 0.785 {\tiny$\pm$ 0.001} \\
Quarterly  & 10.060 {\tiny$\pm$ 0.052} & 1.177 {\tiny$\pm$ 0.009} & 0.886 {\tiny$\pm$ 0.005} &
              10.166 {\tiny$\pm$ 0.008} & 1.178 {\tiny$\pm$ 0.003} & 0.892 {\tiny$\pm$ 0.002} \\
Monthly    & 12.750 {\tiny$\pm$ 0.019} & 0.936 {\tiny$\pm$ 0.006} & 0.882 {\tiny$\pm$ 0.003} &
              12.810 {\tiny$\pm$ 0.016} & 0.942 {\tiny$\pm$ 0.001} & 0.887 {\tiny$\pm$ 0.001} \\
Others     & 4.867 {\tiny$\pm$ 0.046} & 3.316 {\tiny$\pm$ 0.050} & 1.037 {\tiny$\pm$ 0.011} &
              4.849 {\tiny$\pm$ 0.037} & 3.271 {\tiny$\pm$ 0.023} & 1.028 {\tiny$\pm$ 0.007} \\
Averaged   & 11.840 {\tiny$\pm$ 0.016} & 1.585 {\tiny$\pm$ 0.009} & 0.868 {\tiny$\pm$ 0.002} &
              11.889 {\tiny$\pm$ 0.010} & 1.591 {\tiny$\pm$ 0.001} & 0.872 {\tiny$\pm$ 0.000} \\
\bottomrule
\end{tabular}}
\end{table}

\newpage

\section{Comparison with Recent Forecasting Baselines}

We evaluate \method against strong baselines, including 
TimePro~\citep{ma2025timepro},
TimeFilter~\citep{hu2025timefilter},
TimeMixer++~\citep{wang2025timemixer},
iTransformer~\citep{liu2024itransformer},
PatchTST~\citep{nie2023a}, and 
DLinear~\citep{zeng2023transformers} for long-term forecasting and short-term forecasting.
Across the long-term forecasting benchmarks, \method generally achieves the best accuracy on datasets with fewer variables (e.g., ETT).  
As the number of variables increases (e.g., Traffic), TimeFilter obtains the strongest overall performance, owing to its patch-wise filtration mechanism that explicitly 
captures spatiotemporal relationships.  
This suggests that explicitly modeling multi-scale patterns in multivariate 
settings can provide additional benefits. For short-term forecasting, \method also performs strongly across all prediction horizons, 
demonstrating that \method remains effective under diverse 
prediction-length settings.

\begin{table*}[h]
\centering
\caption{Long-term forecasting results comparing seven methods: \method, TimePro, TimeFilter, TimeMixer++, iTransformer, PatchTST, and DLinear. Red and blue denote the best and second-best results, respectively.}
\label{tab:non-multi-long}
\resizebox{0.93\textwidth}{!}{
\begin{tabular}{cc|cc|cc|cc|cc|cc|cc|cc}
\toprule
\multicolumn{2}{c|}{Method}
 & \multicolumn{2}{c|}{\begin{tabular}[c]{@{}c@{}}\method\\ (Ours)\end{tabular}} 
 & \multicolumn{2}{c|}{\begin{tabular}[c]{@{}c@{}}TimePro\\ \citeyearpar{ma2025timepro}\end{tabular}} 
 & \multicolumn{2}{c|}{\begin{tabular}[c]{@{}c@{}}TimeFilter\\ \citeyearpar{hu2025timefilter}\end{tabular}} 
 & \multicolumn{2}{c|}{\begin{tabular}[c]{@{}c@{}}TimeMixer++\\ \citeyearpar{wang2025timemixer}\end{tabular}} 
 & \multicolumn{2}{c|}{\begin{tabular}[c]{@{}c@{}}iTransformer\\ \citeyearpar{liu2024itransformer}\end{tabular}} 
 & \multicolumn{2}{c|}{\begin{tabular}[c]{@{}c@{}}PatchTST\\ \citeyearpar{nie2023a}\end{tabular}} 
 & \multicolumn{2}{c}{\begin{tabular}[c]{@{}c@{}}DLinear\\ \citeyearpar{zeng2023transformers}\end{tabular}} \\
\midrule
\multicolumn{2}{c|}{Metric}
& MSE & MAE
& MSE & MAE
& MSE & MAE
& MSE & MAE
& MSE & MAE
& MSE & MAE
& MSE & MAE \\
\midrule
\multirow{5}{*}{\rotatebox{90}{Weather}}
& 96  & \second{0.160} & \best{0.204} & 0.174 & 0.213 & \best{0.159} & \second{0.205} & 0.163 & 0.211 & 0.176 & 0.216 & 0.174 & 0.216 & 0.196 & 0.258 \\
& 192 & \second{0.209} & \best{0.248} & 0.219 & 0.256 & \best{0.207} & \second{0.249} & 0.213 & 0.255 & 0.225 & 0.257 & 0.220 & 0.257 & 0.237 & 0.296 \\
& 336 & \second{0.270} & \second{0.293} & 0.276 & 0.299 & \best{0.264} & \best{0.291} & 0.273 & 0.298 & 0.281 & 0.299 & 0.278 & 0.298 & 0.283 & 0.333 \\
& 720 & 0.348 & \second{0.345} & 0.355 & 0.350 & \best{0.343} & \best{0.343} & 0.348 & 0.348 & 0.360 & 0.351 & 0.354 & 0.347 & \second{0.346} & 0.383 \\
& Avg & \second{0.247} & \second{0.273} & 0.256 & 0.279 & \best{0.243} & \best{0.272} & 0.249 & 0.278 & 0.260 & 0.281 & 0.256 & 0.279 & 0.266 & 0.318 \\
\midrule
\multirow{5}{*}{\rotatebox{90}{Electricity}}
& 96  & \second{0.146} & \second{0.241} & 0.142 & 0.237 & \best{0.138} & \best{0.235} & 0.186 & 0.275 & 0.151 & 0.243 & 0.180 & 0.273 & 0.210 & 0.301 \\
& 192 & \second{0.163} & \second{0.257} & 0.160 & 0.252 & \best{0.156} & \best{0.253} & 0.190 & 0.277 & 0.167 & 0.258 & 0.187 & 0.279 & 0.210 & 0.305 \\
& 336 & \second{0.179} & \second{0.273} & 0.175 & 0.269 & \best{0.170} & \best{0.268} & 0.205 & 0.291 & 0.181 & 0.274 & 0.204 & 0.295 & 0.223 & 0.319 \\
& 720 & \second{0.213} & 0.304 & \best{0.212} & \second{0.302} & 0.191 & \best{0.289} & 0.249 & 0.327 & \second{0.213} & 0.301 & 0.246 & 0.328 & 0.258 & 0.350 \\
& Avg & \second{0.175} & \second{0.269} & 0.172 & 0.265 & \best{0.164} & \best{0.261} & 0.207 & 0.292 & 0.178 & \second{0.269} & 0.204 & 0.294 & 0.225 & 0.319 \\
\midrule
\multirow{5}{*}{\rotatebox{90}{Traffic}}
& 96  & 0.431 & 0.288 & \second{0.400} & \second{0.267} & \best{0.391} & \best{0.260} & 0.570 & 0.365 & 0.397 & 0.271 & 0.460 & 0.298 & 0.696 & 0.429 \\
& 192 & 0.444 & 0.296 & \second{0.424} & \second{0.277} & \best{0.413} & \best{0.269} & 0.554 & 0.349 & 0.417 & 0.279 & 0.467 & 0.302 & 0.646 & 0.407 \\
& 336 & 0.461 & 0.303 & \second{0.443} & \second{0.286} & \best{0.429} & \best{0.277} & 0.568 & 0.352 & 0.432 & 0.287 & 0.483 & 0.308 & 0.653 & 0.410 \\
& 720 & 0.494 & 0.320 & \second{0.475} & \second{0.306} & \best{0.462} & \best{0.296} & 0.604 & 0.370 & 0.466 & 0.305 & 0.516 & 0.325 & 0.694 & 0.429 \\
& Avg & 0.458 & 0.302 & \second{0.435} & \second{0.284} & \best{0.424} & \best{0.276} & 0.574 & 0.359 & 0.428 & 0.285 & 0.482 & 0.308 & 0.672 & 0.419 \\
\midrule
\multirow{5}{*}{\rotatebox{90}{Exchange}}
& 96  & \second{0.084} & \second{0.204} & 0.085 & 0.205 & \best{0.083} & \best{0.202} & 0.101 & 0.224 & 0.087 & 0.207 & 0.087 & 0.204 & 0.095 & 0.227 \\
& 192 & \best{0.174} & \best{0.297} & \second{0.179} & \second{0.301} & 0.178 & 0.299 & 0.200 & 0.321 & 0.179 & 0.302 & 0.188 & 0.308 & 0.184 & 0.323 \\
& 336 & \best{0.322} & \best{0.411} & \second{0.331} & \second{0.417} & 0.332 & 0.416 & 0.372 & 0.449 & 0.335 & 0.420 & 0.341 & 0.423 & 0.328 & 0.434 \\
& 720 & 0.833 & 0.687 & 0.921 & 0.724 & \second{0.785} & \second{0.669} & 1.000 & 0.762 & 0.854 & 0.697 & 0.921 & 0.720 & \best{0.762} & \best{0.667} \\
& Avg & 0.353 & \second{0.400} & 0.379 & 0.411 & \second{0.345} & \best{0.397} & 0.418 & 0.439 & 0.364 & 0.407 & 0.384 & 0.414 & \best{0.342} & 0.413 \\
\midrule
\multirow{5}{*}{\rotatebox{90}{ETTh1}}
& 96  & \best{0.379} & \best{0.393} & 0.378 & 0.399 & 0.384 & \second{0.395} & 0.407 & 0.418 & 0.388 & 0.405 & \best{0.379} & 0.398 & 0.396 & 0.410 \\
& 192 & \best{0.430} & \second{0.425} & 0.426 & 0.429 & 0.438 & \best{0.424} & 0.449 & 0.443 & 0.444 & 0.437 & \best{0.430} & 0.434 & 0.445 & 0.441 \\
& 336 & 0.481 & \best{0.446} & \best{0.469} & \second{0.452} & 0.481 & \best{0.446} & 0.493 & 0.466 & 0.486 & 0.457 & \second{0.473} & 0.460 & 0.493 & 0.471 \\
& 720 & \second{0.480} & \second{0.468} & \best{0.473} & 0.473 & 0.479 & \best{0.466} & 0.543 & 0.504 & 0.505 & 0.491 & 0.523 & 0.506 & 0.515 & 0.512 \\
& Avg & \best{0.443} & \best{0.433} & 0.436 & 0.438 & 0.446 & \best{0.433} & 0.473 & 0.458 & 0.456 & 0.448 & 0.451 & 0.450 & 0.462 & 0.459 \\
\midrule
\multirow{5}{*}{\rotatebox{90}{ETTh2}}
& 96  & \best{0.289} & \best{0.339} & 0.298 & 0.348 & \second{0.290} & \second{0.340} & 0.316 & 0.364 & 0.300 & 0.350 & 0.300 & 0.351 & 0.346 & 0.399 \\
& 192 & \best{0.369} & \best{0.394} & \second{0.367} & \second{0.395} & 0.377 & \best{0.394} & 0.402 & 0.418 & 0.379 & 0.398 & 0.380 & 0.400 & 0.478 & 0.477 \\
& 336 & \best{0.415} & \best{0.427} & \second{0.419} & \second{0.430} & 0.424 & 0.435 & 0.444 & 0.449 & 0.424 & 0.434 & 0.431 & 0.441 & 0.597 & 0.543 \\
& 720 & \second{0.431} & \best{0.446} & \best{0.429} & \best{0.446} & 0.464 & 0.464 & 0.464 & 0.469 & 0.433 & 0.448 & 0.447 & 0.461 & 0.841 & 0.661 \\
& Avg & \best{0.376} & \best{0.402} & \second{0.378} & \second{0.405} & 0.389 & 0.408 & 0.407 & 0.425 & 0.384 & 0.408 & 0.390 & 0.413 & 0.566 & 0.520 \\
\midrule
\multirow{5}{*}{\rotatebox{90}{ETTm1}}
& 96  & \second{0.323} & \second{0.359} & 0.328 & 0.366 & \best{0.320} & \best{0.357} & 0.332 & 0.368 & 0.344 & 0.377 & 0.330 & 0.368 & 0.345 & 0.372 \\
& 192 & \best{0.360} & \best{0.381} & 0.372 & 0.388 & \second{0.362} & \best{0.381} & 0.374 & 0.395 & 0.383 & 0.395 & 0.370 & 0.391 & 0.382 & \second{0.390} \\
& 336 & \second{0.392} & \second{0.404} & 0.401 & 0.408 & \best{0.391} & \best{0.402} & 0.386 & 0.407 & 0.417 & 0.418 & 0.402 & 0.411 & 0.414 & 0.414 \\
& 720 & \best{0.455} & \second{0.442} & 0.472 & 0.447 & 0.461 & \best{0.438} & 0.469 & 0.456 & 0.488 & 0.457 & \second{0.459} & 0.446 & 0.474 & 0.451 \\
& Avg & \best{0.383} & \second{0.397} & 0.393 & 0.402 & \best{0.383} & \best{0.394} & 0.390 & 0.407 & 0.408 & 0.412 & \second{0.390} & 0.404 & 0.404 & 0.407 \\
\midrule
\multirow{5}{*}{\rotatebox{90}{ETTm2}}
& 96  & \second{0.174} & \best{0.257} & 0.184 & 0.265 & \best{0.172} & \second{0.258} & 0.190 & 0.276 & 0.184 & 0.270 & 0.183 & 0.265 & 0.195 & 0.295 \\
& 192 & \second{0.239} & \best{0.299} & 0.245 & 0.306 & \best{0.237} & \second{0.300} & 0.256 & 0.315 & 0.251 & 0.312 & 0.246 & 0.308 & 0.282 & 0.359 \\
& 336 & \best{0.296} & \best{0.337} & 0.304 & 0.344 & \best{0.296} & \second{0.338} & 0.331 & 0.365 & 0.314 & 0.351 & \second{0.312} & 0.350 & 0.363 & 0.414 \\
& 720 & \best{0.394} & \best{0.394} & 0.405 & 0.402 & \best{0.394} & \second{0.396} & 0.429 & 0.420 & 0.412 & 0.406 & \second{0.419} & 0.412 & 0.547 & 0.519 \\
& Avg & \second{0.276} & \best{0.322} & 0.284 & 0.329 & \best{0.275} & \second{0.323} & 0.302 & 0.344 & 0.290 & 0.335 & 0.290 & 0.334 & 0.347 & 0.397 \\
\bottomrule
\end{tabular}
}
\end{table*}

\begin{table}[t]
\centering
\caption{Short-term forecasting results on the M4 benchmark across seven methods: \method, TimePro, TimeFilter, TimeMixer++, iTransformer, PatchTST, and DLinear. Red and blue denote the best and second-best results, respectively. \method achieves the best performance in 13 out of 15 evaluation cases, demonstrating consistently strong accuracy across temporal granularities and forecasting horizons.}
\label{tab:non-multi-short}
\setlength{\tabcolsep}{2pt}
\small
\renewcommand{\arraystretch}{0.95}
\newcolumntype{L}[1]{>{\raggedright\arraybackslash}p{#1}}
\newcolumntype{R}[1]{>{\raggedleft\arraybackslash}p{#1}}

\resizebox{0.9\textwidth}{!}{%
\begin{tabular}{ll|r|r|r|r|r|r|r}

\midrule
 &
\multicolumn{1}{|c|}{\parbox{0.85cm}{\centering Metric}} &
{\parbox{1.8cm}{\centering \method\\(Ours)}} &
{\parbox{1.8cm}{\centering TimePro\\\citeyearpar{ma2025timepro}}} &
{\parbox{1.8cm}{\centering TimeFilter\\\citeyearpar{hu2025timefilter}}} &
{\parbox{1.8cm}{\centering TimeMixer++\\\citeyearpar{wang2025timemixer}}} &
{\parbox{1.8cm}{\centering iTransformer\\\citeyearpar{liu2024itransformer}}} &
{\parbox{1.8cm}{\centering PatchTST\\\citeyearpar{nie2023a}}} &
{\parbox{1.8cm}{\centering DLinear\\\citeyearpar{zeng2023transformers}}}  \\
\specialrule{0.3pt}{1pt}{2pt}

\multirow{3}{*}{\rotatebox[origin=c]{90}{\scalebox{0.8}{Yearly}}}
& \multicolumn{1}{|l|}{SMAPE} & \best{13.314} & \second{13.368} & 18.836 & 13.368 & 14.091 & 14.311 & 14.343 \\
& \multicolumn{1}{|l|}{MASE}  & \best{2.989}  & \second{3.002}  & 4.153  & 3.002  & 3.133  & 3.240  & 3.123 \\
& \multicolumn{1}{|l|}{OWA}   & \best{0.783}  & \second{0.787}  & 1.099  & 0.787  & 0.825  & 0.845  & 0.832 \\
\specialrule{0.3pt}{1pt}{2pt}

\multirow{3}{*}{\rotatebox[origin=c]{90}{\scalebox{0.8}{Quarterly}}}
& \multicolumn{1}{|l|}{SMAPE} & \best{10.060} & \second{10.121} & 10.660 & 10.269 & 11.775 & 10.242 & 10.502 \\
& \multicolumn{1}{|l|}{MASE}  & \best{1.177}  & \second{1.179}  & 1.228  & 1.211  & 1.449  & 1.211  & 1.240 \\
& \multicolumn{1}{|l|}{OWA}   & \best{0.886}  & \second{0.890}  & 0.932  & 0.908  & 1.063  & 0.907  & 0.929 \\
\specialrule{0.3pt}{1pt}{2pt}

\multirow{3}{*}{\rotatebox[origin=c]{90}{\scalebox{0.8}{Monthly}}}
& \multicolumn{1}{|l|}{SMAPE} & \best{12.750} & \second{12.806} & 13.477 & 13.387 & 15.623 & 12.889 & 13.373 \\
& \multicolumn{1}{|l|}{MASE}  & \best{0.936}  & \second{0.940}  & 1.025  & 1.019  & 1.273  & 0.955  & 1.004 \\
& \multicolumn{1}{|l|}{OWA}   & \best{0.882}  & \second{0.886}  & 0.949  & 0.943  & 1.140  & 0.896  & 0.935 \\
\specialrule{0.3pt}{1pt}{2pt}

\multirow{3}{*}{\rotatebox[origin=c]{90}{\scalebox{0.8}{Others}}}
& \multicolumn{1}{|l|}{SMAPE} & \best{4.867}  & 5.326 & 6.136 & 5.535 & 5.407 & \second{4.986} & 5.110 \\
& \multicolumn{1}{|l|}{MASE}  & \second{3.316} & 3.458 & 4.042 & 3.661 & 3.768 & \best{3.231} & 3.655 \\
& \multicolumn{1}{|l|}{OWA}   & \second{1.037} & 1.081 & 1.257 & 1.160 & 1.163 & \best{1.023} & 1.132 \\
\specialrule{0.3pt}{1pt}{2pt}

\multirow{3}{*}{\rotatebox[origin=c]{90}{\scalebox{0.8}{\shortstack{Weighted\\[-0.5ex]Average}}}}
& \multicolumn{1}{|l|}{SMAPE} & \best{11.840} & \second{11.947} & 13.666 & 12.241 & 13.836 & 12.186 & 12.494 \\
& \multicolumn{1}{|l|}{MASE}  & \best{1.585}  & \second{1.614}  & 1.944  & 1.653  & 1.868  & 1.656  & 1.681 \\
& \multicolumn{1}{|l|}{OWA}   & \best{0.868}  & \second{0.877}  & 0.995  & 0.884  & 0.998  & 0.893  & 0.920 \\
\midrule

\label{tab:short-term-4methods-clean}
\end{tabular}
}
\renewcommand{\arraystretch}{1.0}
\vspace{-10pt}
\end{table}


\section{Comparison Under Varying Look-Back Windows}

To examine the effect of the look-back window beyond the fixed setting, we compare SiGMA and PatchTST over $L \in \{24,48,96,192,336,512,720\}$. Table~\ref{tab:appendix-sigma-patchtst-l96} reports the best result for each dataset together with its selected configuration, while Tables~\ref{tab:appendix-sigma-diffl} and~\ref{tab:appendix-patchtst-diffl} provide the full results across all look-back windows.
Overall, \method outperforms PatchTST on most datasets. In particular, on the Traffic dataset, \method improves over PatchTST by 9.7\% in MSE and 4.6\% in MAE under the best-performing configurations. It also consistently outperforms PatchTST on all ETT benchmarks, suggesting that distance-aware multi-scale representations provide robust modeling capacity.

PatchTST also benefits from a broader search over context lengths. For example, it achieves its best performance on Weather with a relatively large input length ($L=336$), whereas Exchange favors a much shorter context length ($L=48$). These results suggest that the optimal look-back window is dataset-dependent rather than monotonically increasing with larger contexts. 
Moreover, increasing the look-back window incurs significantly higher cost for PatchTST. On Traffic with $L=720$ under the predict-720 setting, one training epoch requires 4.0 minutes for \method, compared to 23.3 minutes for PatchTST, making PatchTST approximately 5.8$\times$ slower. Overall, \method achieves robust forecasting performance across input lengths while maintaining substantially higher efficiency.

\begin{table}[h]
\centering
\caption{Experimental evaluation against PatchTST over $L\in\{24,48,96,192,336,512,720\}$. For each dataset and method, we report the best performance across look-back windows, together with the look-back window $L$ that achieves it. Overall, the results show that \method achieves robust forecasting performance across input lengths.}
\label{tab:appendix-sigma-patchtst-l96}
\setlength{\tabcolsep}{5pt}
\renewcommand{\arraystretch}{1.0}
\resizebox{0.48\textwidth}{!}{%
\begin{tabular}{l|ccc|ccc}
\toprule
\multirow{2}{*}{Models} 
& \multicolumn{3}{c|}{\method} 
& \multicolumn{3}{c}{PatchTST} \\
\cmidrule(lr){2-4} \cmidrule(lr){5-7}
& MSE & MAE & $L$ & MSE & MAE & $L$ \\
\midrule
Weather     & \second{0.243} & \second{0.277} & 336 & \best{0.234} & \best{0.270} & 336 \\
Electricity & \best{0.163} & \best{0.259} & 192 & \second{0.165} & \second{0.266} & 720  \\
Traffic     & \best{0.362} & \best{0.269} & 720 & \second{0.401} & \second{0.282} & 720  \\
Exchange    & \second{0.353} & \second{0.400} & 96  & \best{0.351} & \best{0.397} & 48  \\
ETTh1       & \best{0.422} & \best{0.431} & 336 & \second{0.451} & \second{0.450} & 96  \\
ETTh2       & \best{0.366} & \best{0.400} & 336 & \second{0.386} & \second{0.415} & 192 \\
ETTm1       & \best{0.361} & \best{0.384} & 192 & \second{0.379} & \second{0.403} & 192 \\
ETTm2       & \best{0.266} & \best{0.320} & 336 & \second{0.287} & \second{0.337} & 192 \\
\bottomrule
\end{tabular}
}
\end{table}

\begin{table*}[t]
\centering
\caption[SiGMA across look-back windows.]{Long-term forecasting results \method with different look-back windows.}
\label{tab:appendix-sigma-diffl}
\resizebox{\columnwidth}{!}{
\begin{tabular}{cc|cc|cc|cc|cc|cc|cc|cc}
\toprule
\multicolumn{2}{c|}{$L$}
 & \multicolumn{2}{c|}{$L=24$}
 & \multicolumn{2}{c|}{$L=48$}
 & \multicolumn{2}{c|}{$L=96$}
 & \multicolumn{2}{c|}{$L=192$}
 & \multicolumn{2}{c|}{$L=336$}
 & \multicolumn{2}{c|}{$L=512$}
 & \multicolumn{2}{c}{$L=720$}
 \\
\midrule
\multicolumn{2}{c|}{Metric}
 & MSE & MAE
 & MSE & MAE
 & MSE & MAE
 & MSE & MAE
 & MSE & MAE
 & MSE & MAE
 & MSE & MAE \\
\midrule

\multirow{5}*{\rotatebox{90}{Weather}} & 96 & 0.197 & 0.229 & 0.185 & 0.223 & 0.160 & 0.204 & 0.164 & 0.213 & 0.162 & 0.216 & 0.161 & 0.216 & 0.161 & 0.214 \\
  & 192 & 0.236 & 0.262 & 0.225 & 0.258 & 0.209 & 0.248 & 0.210 & 0.254 & 0.203 & 0.250 & 0.213 & 0.262 & 0.217 & 0.265 \\
  & 336 & 0.299 & 0.307 & 0.282 & 0.299 & 0.270 & 0.293 & 0.262 & 0.291 & 0.263 & 0.295 & 0.277 & 0.306 & 0.291 & 0.318 \\
  & 720 & 0.388 & 0.364 & 0.367 & 0.354 & 0.348 & 0.345 & 0.340 & 0.344 & 0.346 & 0.349 & 0.358 & 0.359 & 0.388 & 0.378 \\
 & Avg & 0.280 & 0.290 & 0.265 & 0.283 & 0.247 & 0.272 & 0.244 & 0.276 & 0.243 & 0.277 & 0.252 & 0.286 & 0.264 & 0.294 \\
\midrule

\multirow{5}*{\rotatebox{90}{Electricity}} & 96 & 0.231 & 0.303 & 0.177 & 0.268 & 0.146 & 0.241 & 0.132 & 0.228 & 0.133 & 0.231 & 0.133 & 0.233 & 0.135 & 0.238 \\
  & 192 & 0.233 & 0.307 & 0.189 & 0.277 & 0.163 & 0.257 & 0.154 & 0.249 & 0.160 & 0.259 & 0.158 & 0.258 & 0.158 & 0.259 \\
  & 336 & 0.256 & 0.327 & 0.208 & 0.295 & 0.179 & 0.273 & 0.170 & 0.267 & 0.177 & 0.276 & 0.173 & 0.272 & 0.176 & 0.278 \\
  & 720 & 0.300 & 0.358 & 0.242 & 0.325 & 0.213 & 0.304 & 0.195 & 0.291 & 0.190 & 0.289 & 0.198 & 0.297 & 0.198 & 0.299 \\
 & Avg & 0.255 & 0.324 & 0.204 & 0.291 & 0.175 & 0.269 & 0.163 & 0.259 & 0.165 & 0.263 & 0.166 & 0.265 & 0.167 & 0.268 \\
\midrule

\multirow{5}*{\rotatebox{90}{Traffic}} & 96 & 0.599 & 0.371 & 0.514 & 0.332 & 0.431 & 0.288 & 0.380 & 0.264 & 0.357 & 0.259 & 0.341 & 0.255 & 0.337 & 0.254 \\
  & 192 & 0.610 & 0.379 & 0.520 & 0.336 & 0.444 & 0.296 & 0.401 & 0.272 & 0.378 & 0.268 & 0.351 & 0.261 & 0.353 & 0.265 \\
  & 336 & 0.635 & 0.387 & 0.544 & 0.347 & 0.461 & 0.303 & 0.416 & 0.280 & 0.381 & 0.275 & 0.366 & 0.268 & 0.360 & 0.270 \\
  & 720 & 0.659 & 0.400 & 0.568 & 0.358 & 0.494 & 0.320 & 0.441 & 0.296 & 0.416 & 0.289 & 0.402 & 0.284 & 0.398 & 0.289 \\
 & Avg & 0.626 & 0.384 & 0.537 & 0.343 & 0.458 & 0.302 & 0.410 & 0.278 & 0.383 & 0.272 & 0.365 & 0.267 & 0.362 & 0.269 \\
\midrule

\multirow{5}*{\rotatebox{90}{Exchange}} & 96 & 0.085 & 0.202 & 0.086 & 0.205 & 0.084 & 0.204 & 0.094 & 0.216 & 0.094 & 0.218 & 0.100 & 0.224 & 0.111 & 0.238 \\
  & 192 & 0.187 & 0.306 & 0.187 & 0.308 & 0.174 & 0.297 & 0.204 & 0.323 & 0.195 & 0.316 & 0.210 & 0.329 & 0.224 & 0.343 \\
  & 336 & 0.339 & 0.418 & 0.345 & 0.424 & 0.322 & 0.411 & 0.336 & 0.422 & 0.362 & 0.436 & 0.444 & 0.490 & 0.399 & 0.467 \\
  & 720 & 0.888 & 0.710 & 0.915 & 0.720 & 0.833 & 0.687 & 1.068 & 0.765 & 1.043 & 0.766 & 1.554 & 0.896 & 1.375 & 0.881 \\
 & Avg & 0.375 & 0.409 & 0.383 & 0.414 & 0.353 & 0.400 & 0.426 & 0.432 & 0.423 & 0.434 & 0.577 & 0.485 & 0.527 & 0.482 \\
\midrule

\multirow{5}*{\rotatebox{90}{ETTh1}} & 96 & 0.430 & 0.416 & 0.385 & 0.395 & 0.379 & 0.393 & 0.388 & 0.403 & 0.378 & 0.398 & 0.374 & 0.401 & 0.386 & 0.412 \\
  & 192 & 0.485 & 0.448 & 0.439 & 0.428 & 0.430 & 0.425 & 0.430 & 0.423 & 0.413 & 0.419 & 0.405 & 0.422 & 0.418 & 0.434 \\
  & 336 & 0.538 & 0.476 & 0.490 & 0.453 & 0.481 & 0.446 & 0.458 & 0.438 & 0.437 & 0.436 & 0.434 & 0.440 & 0.454 & 0.456 \\
  & 720 & 0.535 & 0.489 & 0.505 & 0.480 & 0.480 & 0.468 & 0.479 & 0.471 & 0.459 & 0.469 & 0.508 & 0.494 & 0.659 & 0.554 \\
 & Avg & 0.497 & 0.457 & 0.455 & 0.439 & 0.443 & 0.433 & 0.439 & 0.434 & 0.422 & 0.431 & 0.430 & 0.439 & 0.479 & 0.464 \\
\midrule

\multirow{5}*{\rotatebox{90}{ETTh2}} & 96 & 0.325 & 0.356 & 0.297 & 0.342 & 0.289 & 0.339 & 0.289 & 0.341 & 0.295 & 0.350 & 0.290 & 0.349 & 0.292 & 0.351 \\
  & 192 & 0.416 & 0.407 & 0.383 & 0.394 & 0.369 & 0.394 & 0.367 & 0.394 & 0.355 & 0.389 & 0.364 & 0.394 & 0.373 & 0.407 \\
  & 336 & 0.472 & 0.449 & 0.436 & 0.436 & 0.415 & 0.427 & 0.398 & 0.419 & 0.384 & 0.413 & 0.394 & 0.422 & 0.431 & 0.449 \\
  & 720 & 0.477 & 0.461 & 0.454 & 0.455 & 0.431 & 0.446 & 0.435 & 0.451 & 0.431 & 0.450 & 0.449 & 0.464 & 0.472 & 0.478 \\
 & Avg & 0.422 & 0.418 & 0.393 & 0.407 & 0.376 & 0.402 & 0.372 & 0.401 & 0.366 & 0.400 & 0.374 & 0.407 & 0.392 & 0.421 \\
\midrule

\multirow{5}*{\rotatebox{90}{ETTm1}} & 96 & 0.677 & 0.502 & 0.467 & 0.426 & 0.323 & 0.359 & 0.300 & 0.346 & 0.304 & 0.350 & 0.306 & 0.357 & 0.319 & 0.364 \\
  & 192 & 0.717 & 0.526 & 0.508 & 0.447 & 0.360 & 0.381 & 0.340 & 0.370 & 0.335 & 0.371 & 0.353 & 0.386 & 0.368 & 0.396 \\
  & 336 & 0.761 & 0.551 & 0.551 & 0.473 & 0.392 & 0.404 & 0.368 & 0.391 & 0.374 & 0.397 & 0.377 & 0.402 & 0.414 & 0.425 \\
  & 720 & 0.795 & 0.573 & 0.591 & 0.497 & 0.455 & 0.442 & 0.435 & 0.429 & 0.434 & 0.434 & 0.447 & 0.444 & 0.480 & 0.461 \\
 & Avg & 0.737 & 0.538 & 0.529 & 0.461 & 0.383 & 0.397 & 0.361 & 0.384 & 0.362 & 0.388 & 0.370 & 0.397 & 0.395 & 0.411 \\
\midrule

\multirow{5}*{\rotatebox{90}{ETTm2}} & 96 & 0.211 & 0.289 & 0.191 & 0.273 & 0.174 & 0.257 & 0.171 & 0.254 & 0.172 & 0.258 & 0.171 & 0.257 & 0.176 & 0.262 \\
  & 192 & 0.280 & 0.331 & 0.259 & 0.316 & 0.239 & 0.299 & 0.234 & 0.296 & 0.239 & 0.301 & 0.246 & 0.306 & 0.240 & 0.303 \\
  & 336 & 0.352 & 0.374 & 0.325 & 0.355 & 0.296 & 0.337 & 0.290 & 0.334 & 0.289 & 0.335 & 0.297 & 0.341 & 0.311 & 0.349 \\
  & 720 & 0.458 & 0.431 & 0.428 & 0.412 & 0.394 & 0.394 & 0.381 & 0.392 & 0.365 & 0.386 & 0.374 & 0.392 & 0.385 & 0.401 \\
 & Avg & 0.325 & 0.356 & 0.301 & 0.339 & 0.276 & 0.322 & 0.269 & 0.319 & 0.266 & 0.320 & 0.272 & 0.324 & 0.278 & 0.329 \\
 
\bottomrule
\end{tabular}
}
\end{table*}
\begin{table}[t]
\centering
\caption[PatchTST across look-back windows.]{
Long-term forecasting results of PatchTST under different look-back windows. 
$L=336$ and $L=512$ correspond to PatchTST/42 and PatchTST/64, respectively. 
All settings use patch length $P=16$ and stride $S=8$.
}
\label{tab:appendix-patchtst-diffl}
\resizebox{\columnwidth}{!}{
\begin{tabular}{cc|cc|cc|cc|cc|cc|cc|cc}
\toprule
\multicolumn{2}{c|}{$L$}
 & \multicolumn{2}{c|}{$L=24$}
 & \multicolumn{2}{c|}{$L=48$}
 & \multicolumn{2}{c|}{$L=96$}
 & \multicolumn{2}{c|}{$L=192$}
 & \multicolumn{2}{c|}{$L=336$}
 & \multicolumn{2}{c|}{$L=512$}
 & \multicolumn{2}{c}{$L=720$}
 \\
\midrule
\multicolumn{2}{c|}{Metric}
 & MSE & MAE
 & MSE & MAE
 & MSE & MAE
 & MSE & MAE
 & MSE & MAE
 & MSE & MAE
 & MSE & MAE \\
\midrule

\multirow{5}*{\rotatebox{90}{Weather}} & 96 & 0.220 & 0.247 & 0.208 & 0.242 & 0.174 & 0.216 & 0.157 & 0.204 & 0.153 & 0.203 & 0.151 & 0.204 & 0.158 & 0.214 \\
  & 192 & 0.262 & 0.279 & 0.251 & 0.275 & 0.220 & 0.257 & 0.202 & 0.245 & 0.198 & 0.246 & 0.202 & 0.252 & 0.203 & 0.256 \\
  & 336 & 0.323 & 0.322 & 0.306 & 0.313 & 0.278 & 0.298 & 0.259 & 0.288 & 0.254 & 0.289 & 0.262 & 0.301 & 0.260 & 0.301 \\
  & 720 & 0.404 & 0.374 & 0.383 & 0.363 & 0.354 & 0.347 & 0.337 & 0.341 & 0.330 & 0.341 & 0.325 & 0.341 & 0.331 & 0.350 \\
 & Avg & 0.302 & 0.305 & 0.287 & 0.298 & 0.257 & 0.279 & 0.239 & 0.270 & 0.234 & 0.270 & 0.235 & 0.274 & 0.238 & 0.280 \\
\midrule
\multirow{5}*{\rotatebox{90}{Electricity}} 
 & 96  & 0.276 & 0.328 & 0.232 & 0.306 & 0.180 & 0.273 & 0.145 & 0.246 & 0.139 & 0.242 & 0.137 & 0.240 & 0.136 & 0.240 \\
 & 192 & 0.268 & 0.329 & 0.227 & 0.306 & 0.187 & 0.279 & 0.161 & 0.261 & 0.156 & 0.257 & 0.153 & 0.255 & 0.153 & 0.255 \\
 & 336 & 0.292 & 0.347 & 0.247 & 0.323 & 0.204 & 0.295 & 0.179 & 0.278 & 0.173 & 0.273 & 0.170 & 0.271 & 0.168 & 0.270 \\
 & 720 & 0.330 & 0.377 & 0.287 & 0.354 & 0.246 & 0.328 & 0.217 & 0.310 & 0.210 & 0.304 & 0.205 & 0.300 & 0.202 & 0.298 \\
 & Avg & 0.291 & 0.345 & 0.248 & 0.322 & 0.204 & 0.294 & 0.175 & 0.274 & 0.170 & 0.270 & 0.166 & 0.267 & 0.165 & 0.266 \\
\midrule

\multirow{5}*{\rotatebox{90}{Traffic}} 
 & 96  & 0.753 & 0.416 & 0.651 & 0.375 & 0.460 & 0.298 & 0.400 & 0.277 & 0.386 & 0.273 & 0.379 & 0.273 & 0.374 & 0.270 \\
 & 192 & 0.715 & 0.401 & 0.607 & 0.357 & 0.467 & 0.302 & 0.416 & 0.285 & 0.402 & 0.281 & 0.393 & 0.278 & 0.387 & 0.275 \\
 & 336 & 0.748 & 0.414 & 0.631 & 0.366 & 0.483 & 0.308 & 0.430 & 0.290 & 0.413 & 0.286 & 0.404 & 0.283 & 0.402 & 0.282 \\
 & 720 & 0.781 & 0.429 & 0.666 & 0.382 & 0.516 & 0.325 & 0.460 & 0.307 & 0.444 & 0.302 & 0.440 & 0.302 & 0.440 & 0.302 \\
 & Avg & 0.749 & 0.415 & 0.639 & 0.370 & 0.482 & 0.308 & 0.427 & 0.290 & 0.411 & 0.286 & 0.404 & 0.284 & 0.401 & 0.282 \\
 \midrule

\multirow{5}*{\rotatebox{90}{Exchange}} & 96 & 0.082 & 0.198 & 0.083 & 0.200 & 0.087 & 0.204 & 0.095 & 0.216 & 0.100 & 0.227 & 0.104 & 0.236 & 0.120 & 0.255 \\
  & 192 & 0.170 & 0.292 & 0.174 & 0.294 & 0.188 & 0.308 & 0.212 & 0.329 & 0.236 & 0.351 & 0.249 & 0.365 & 0.267 & 0.378 \\
  & 336 & 0.317 & 0.405 & 0.319 & 0.407 & 0.341 & 0.423 & 0.356 & 0.435 & 0.413 & 0.472 & 0.428 & 0.486 & 0.554 & 0.550 \\
  & 720 & 0.841 & 0.690 & 0.827 & 0.685 & 0.921 & 0.720 & 0.890 & 0.699 & 1.058 & 0.757 & 1.110 & 0.779 & 1.361 & 0.863 \\
 & Avg & 0.352 & 0.396 & 0.351 & 0.397 & 0.384 & 0.414 & 0.388 & 0.420 & 0.452 & 0.452 & 0.473 & 0.466 & 0.575 & 0.512 \\
\midrule

\multirow{5}*{\rotatebox{90}{ETTh1}} & 96 & 0.415 & 0.416 & 0.384 & 0.402 & 0.379 & 0.398 & 0.423 & 0.437 & 0.451 & 0.452 & 0.455 & 0.464 & 0.470 & 0.470 \\
  & 192 & 0.475 & 0.449 & 0.443 & 0.438 & 0.430 & 0.434 & 0.504 & 0.480 & 0.511 & 0.489 & 0.506 & 0.491 & 0.502 & 0.496 \\
  & 336 & 0.530 & 0.481 & 0.498 & 0.471 & 0.473 & 0.460 & 0.531 & 0.498 & 0.552 & 0.513 & 0.520 & 0.509 & 0.560 & 0.535 \\
  & 720 & 0.536 & 0.503 & 0.505 & 0.494 & 0.523 & 0.506 & 0.594 & 0.544 & 0.591 & 0.541 & 0.610 & 0.556 & 0.673 & 0.594 \\
 & Avg & 0.489 & 0.462 & 0.457 & 0.451 & 0.451 & 0.450 & 0.513 & 0.490 & 0.526 & 0.499 & 0.523 & 0.505 & 0.551 & 0.524 \\
\midrule

\multirow{5}*{\rotatebox{90}{ETTh2}} & 96 & 0.319 & 0.352 & 0.306 & 0.348 & 0.300 & 0.351 & 0.324 & 0.371 & 0.346 & 0.389 & 0.339 & 0.388 & 0.362 & 0.407 \\
  & 192 & 0.413 & 0.406 & 0.391 & 0.399 & 0.380 & 0.400 & 0.384 & 0.408 & 0.452 & 0.450 & 0.443 & 0.443 & 0.441 & 0.443 \\
  & 336 & 0.465 & 0.446 & 0.440 & 0.439 & 0.431 & 0.441 & 0.412 & 0.434 & 0.466 & 0.468 & 0.452 & 0.457 & 0.475 & 0.472 \\
  & 720 & 0.470 & 0.461 & 0.448 & 0.456 & 0.447 & 0.461 & 0.425 & 0.445 & 0.457 & 0.473 & 0.508 & 0.500 & 0.491 & 0.490 \\
 & Avg & 0.417 & 0.417 & 0.396 & 0.411 & 0.390 & 0.413 & 0.386 & 0.415 & 0.430 & 0.445 & 0.436 & 0.447 & 0.442 & 0.453 \\
\midrule

\multirow{5}*{\rotatebox{90}{ETTm1}} & 96 & 0.635 & 0.486 & 0.424 & 0.404 & 0.330 & 0.368 & 0.324 & 0.369 & 0.324 & 0.374 & 0.336 & 0.385 & 0.329 & 0.384 \\
  & 192 & 0.681 & 0.513 & 0.467 & 0.428 & 0.370 & 0.391 & 0.358 & 0.388 & 0.371 & 0.398 & 0.379 & 0.405 & 0.383 & 0.411 \\
  & 336 & 0.727 & 0.539 & 0.506 & 0.453 & 0.402 & 0.411 & 0.389 & 0.412 & 0.400 & 0.420 & 0.417 & 0.429 & 0.454 & 0.459 \\
  & 720 & 0.764 & 0.563 & 0.559 & 0.484 & 0.459 & 0.446 & 0.445 & 0.444 & 0.477 & 0.467 & 0.492 & 0.473 & 0.515 & 0.497 \\
 & Avg & 0.702 & 0.525 & 0.489 & 0.442 & 0.390 & 0.404 & 0.379 & 0.403 & 0.393 & 0.415 & 0.406 & 0.423 & 0.420 & 0.438 \\
\midrule

\multirow{5}*{\rotatebox{90}{ETTm2}} & 96 & 0.210 & 0.289 & 0.191 & 0.274 & 0.183 & 0.265 & 0.182 & 0.269 & 0.190 & 0.277 & 0.196 & 0.288 & 0.197 & 0.284 \\
  & 192 & 0.281 & 0.333 & 0.261 & 0.319 & 0.246 & 0.308 & 0.244 & 0.312 & 0.249 & 0.319 & 0.253 & 0.325 & 0.264 & 0.331 \\
  & 336 & 0.353 & 0.375 & 0.332 & 0.362 & 0.312 & 0.350 & 0.317 & 0.356 & 0.320 & 0.365 & 0.323 & 0.371 & 0.322 & 0.371 \\
  & 720 & 0.457 & 0.432 & 0.432 & 0.418 & 0.419 & 0.412 & 0.404 & 0.410 & 0.420 & 0.428 & 0.423 & 0.427 & 0.406 & 0.422 \\
 & Avg & 0.325 & 0.357 & 0.304 & 0.343 & 0.290 & 0.334 & 0.287 & 0.337 & 0.295 & 0.347 & 0.299 & 0.353 & 0.297 & 0.352 \\
 
\bottomrule
\end{tabular}
}
\end{table}

\newpage
\clearpage

\section{Workflow Diagram and Algorithms}
\label{sec:workflow}

For completeness, we provide the overall workflow of \method together with detailed algorithms for the end-to-end forward pass and the LDG kernel computations.

\vspace{10pt}

\begin{figure}[h]
\centering
\small
\resizebox{0.72\linewidth}{!}{
\begin{tikzpicture}[
    font=\small,
    >=Latex,
    node distance=3.2mm,
    box/.style={draw, rounded corners=2pt, thick, align=center, minimum width=42mm, minimum height=5.2mm, inner sep=1.4pt},
    smallbox/.style={
    draw,
    rounded corners=2pt,
    thick,
    align=center,
    minimum width=28mm,
    minimum height=5.2mm,
    inner sep=1.4pt},
    io/.style={draw, rounded corners=2pt, thick, fill={rgb,255:red,0;green,76;blue,152}, text=white, align=center, minimum width=42mm, minimum height=5.2mm, inner sep=1.4pt},
    op/.style={draw, rounded corners=2pt, thick, fill=gray!10, align=center, minimum width=42mm, minimum height=5.2mm, inner sep=1.4pt},
    proj/.style={draw, rounded corners=2pt, thick, align=center, minimum width=42mm, minimum height=5.2mm, inner sep=1.4pt},
    param/.style={draw, rounded corners=2pt, thick, fill=orange!14, align=center, minimum width=36mm, minimum height=5.2mm, inner sep=1.4pt},
    sum/.style={circle, draw, thick, minimum size=4.8mm, inner sep=0pt},
    line/.style={->, thick}
]

\node[io] (input) {Input $\mathbf{x}$};

\node[op, below=of input] (norm) {Normalization (RevIN)};

\node[op, below=of norm] (embed) {Embedding $\mathbf{X}$};

\node[op, below=of embed] (ldg) {LDG decomposition $\mathbf{K}(\mathbf{s})$};

\node[smallbox, below left=3.2mm and 4mm of ldg, xshift=13mm] (smooth)
{$\mathbf{K}(\mathbf{s})\mathbf{X}$};

\node[smallbox, below right=3.2mm and 4mm of ldg, xshift=-13mm] (resid)
{$(I-\mathbf{K}(\mathbf{s}))\mathbf{X}$};

\node[op, below=3.8mm of $(smooth)!0.5!(resid)$] (concat) {Concatenation};

\node[box, below=of concat] (H)
{$\mathbf{H} = \mathbf{K}(\mathbf{s})\mathbf{X} \,\|\, (I-\mathbf{K}(\mathbf{s}))\mathbf{X}$};

\node[op, below=of H] (mlp) {$\mathrm{MLP}(\mathbf{H})$};

\node[sum, below=of mlp] (sum) {$+$};

\node[box, below=of sum] (Hres) {$\mathbf{H} + \mathrm{MLP}(\mathbf{H})$};

\node[proj, below=of Hres] (proj)
{$\mathbf{W}_1\bigl(\mathbf{H}+\mathrm{MLP}(\mathbf{H})\bigr)\mathbf{W}_2$};

\node[op, below=of proj] (denorm) {De-normalization};

\node[io, below=of denorm] (output) {Output $\hat{\mathbf{y}}$};

\node[param, right=8mm of ldg] (scale)
{Scale parameter $\mathbf{s}$};

\node[param, right=8mm of proj, yshift=5mm] (W1)
{Temporal projection $\mathbf{W}_1$};

\node[param, right=8mm of proj, yshift=-5mm] (W2)
{Feature projection $\mathbf{W}_2$};

\draw[line] (input) -- (norm);
\draw[line] (norm) -- (embed);
\draw[line] (embed) -- (ldg);

\draw[line] (scale.west) -- (ldg.east);

\draw[line] (ldg) -- (smooth);
\draw[line] (ldg) -- (resid);

\draw[line] (smooth) |- (concat);
\draw[line] (resid) |- (concat);

\draw[line] (concat) -- (H);
\draw[line] (H) -- (mlp);

\draw[line] (mlp) -- (sum);
\draw[line] (H.east) -- ++(5mm,0) |- (sum.east);

\draw[line] (sum) -- (Hres);
\draw[line] (Hres) -- (proj);

\draw[line] (W1.west) -- (proj.east);
\draw[line] (W2.west) -- (proj.east);

\draw[line] (proj) -- (denorm);
\draw[line] (denorm) -- (output);

\end{tikzpicture}
}

\vspace{10pt}
\caption[Workflow of SiGMA.]{
Workflow of \method. The input $\mathbf{x}$ is normalized and embedded into $\mathbf{X}$. The LDG operator parameterized by $\mathbf{s}$ decomposes $\mathbf{X}$ into smoothed and residual components, which are concatenated into $\mathbf{H}$ and processed with a residual MLP. The resulting representation is projected to the prediction horizon via $\mathbf{W}_1$ and $\mathbf{W}_2$, followed by de-normalization to obtain $\hat{\mathbf{y}}$.
}
\label{fig:sigma_workflow_vertical_compact}
\end{figure}
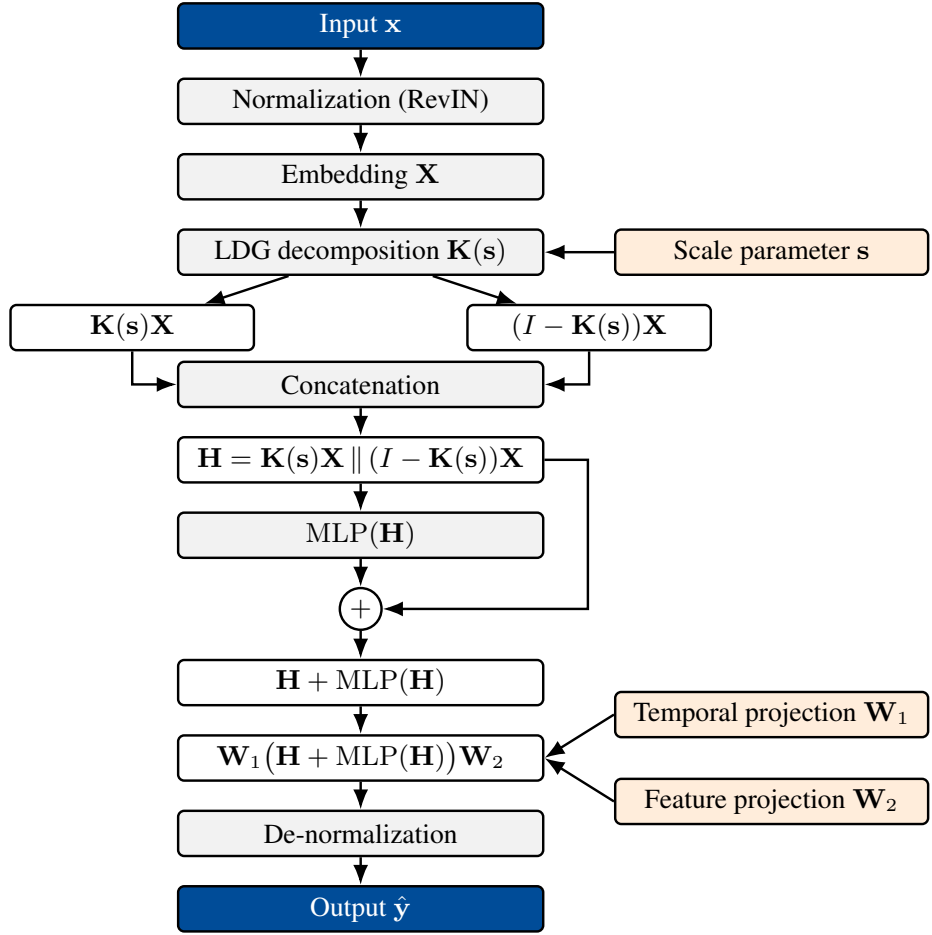

\newpage
\clearpage

\begin{algorithm}[H]
\caption{Forward pass of \method}
\label{alg:sigma}
\begin{algorithmic}[1]

\STATE \textbf{Input:} $\mathbf{x}\in\mathbb{R}^{B\times L\times C}$ with batch size $B$, look-back length $L$, and number of variables $C$
\STATE \textbf{Output:} $\hat{\mathbf y}\in\mathbb{R}^{B\times T\times C}$ with prediction horizon $T$

\STATE $\mathbf{x}\gets \mathrm{Normalize}(\mathbf{x})$

\STATE reshape $\mathbf{x}\in\mathbb{R}^{B\times L\times C}$ to $\mathbf{x}\in\mathbb{R}^{(BC)\times L\times 1}$

\STATE $\mathbf{H}\gets \mathrm{Embed}(\mathbf{x})$

\STATE $\mathbf{s}\gets \mathrm{softplus}(\boldsymbol{\theta})$
\STATE $(\mathbf{H}_{\mathrm{s}}, \mathbf{H}_{\mathrm{r}})\gets \mathrm{LDGForward}(\mathbf{H}, \mathbf{s})$

\STATE $\bar{\mathbf{H}}\gets \mathbf{H}_{\mathrm{s}} \,\|\, \mathbf{H}_{\mathrm{r}}$

\STATE $\mathbf{U}\gets \bar{\mathbf{H}}+\mathrm{MLP}(\bar{\mathbf{H}})$

\STATE split $\mathbf{U}$ into $\tilde{\mathbf{H}}_{\mathrm{s}}, \tilde{\mathbf{H}}_{\mathrm{r}}\in\mathbb{R}^{BC\times L\times d}$

\STATE $\mathbf{H}\gets \tilde{\mathbf{H}}_{\mathrm{s}}+\tilde{\mathbf{H}}_{\mathrm{r}}$

\FOR{each branch $\mathbf{U}\in\{\tilde{\mathbf{H}}_{\mathrm{s}}, \tilde{\mathbf{H}}_{\mathrm{r}}\}$}
    \STATE $\mathbf{Z}\gets \mathbf{W}_{1}\mathbf{U}$
    \STATE $\tilde{\mathbf{Y}}^{(\mathbf{U})}\gets \mathbf{Z}\mathbf{W}_{2}$
    \STATE reshape $\tilde{\mathbf{Y}}^{(\mathbf{U})}$ back to $\mathbb{R}^{B\times T\times C}$
\ENDFOR

\STATE $\tilde{\mathbf y}\gets \tilde{\mathbf{Y}}^{(\tilde{\mathbf{H}}_{\mathrm{s}})}
+\tilde{\mathbf{Y}}^{(\tilde{\mathbf{H}}_{\mathrm{r}})}$

\STATE $\hat{\mathbf y}\gets \mathrm{Denormalize}(\tilde{\mathbf y})$
\STATE \textbf{return} $\hat{\mathbf y}$

\end{algorithmic}
\end{algorithm}

\begin{algorithm}[h]
\caption{LDG kernel forward computation}
\label{alg:ldg-forward}
\begin{algorithmic}[1]

\STATE \textbf{Input:} hidden representation $\mathbf{H}\in\mathbb{R}^{BC\times L\times d}$, scale parameters $\mathbf{s}\in\mathbb{R}_{+}^{L}$
\STATE \textbf{Output:} smoothed representation $\mathbf{H}_{\mathrm{s}}\in\mathbb{R}^{BC\times L\times d}$, residual representation $\mathbf{H}_{\mathrm{r}}\in\mathbb{R}^{BC\times L\times d}$

\STATE construct the distance matrix $\mathbf{D}\in\mathbb{N}^{L\times L}$ with $D_{i,j}=|i-j|$

\FOR{$d=0,\dots,L-1$}
    \STATE $k_d \gets e^{-s_d} I_d(s_d)$
\ENDFOR

\STATE form the kernel matrix $\mathbf{K}\in\mathbb{R}^{L\times L}$ by $K_{i,j}=k_{D_{i,j}}$

\STATE $\mathbf{H}_{\mathrm{s}}\gets \mathbf{K}\mathbf{H}$

\STATE $\mathbf{H}_{\mathrm{r}}\gets \mathbf{H}-\mathbf{H}_{\mathrm{s}}$

\STATE \textbf{return} $\mathbf{H}_{\mathrm{s}}, \mathbf{H}_{\mathrm{r}}$

\end{algorithmic}
\end{algorithm}

\begin{algorithm}[h]
\caption{LDG kernel backward computation}
\label{alg:ldg-backward}
\begin{algorithmic}[1]

\STATE \textbf{Input:} upstream gradient $\frac{\partial \mathcal{L}}{\partial \mathbf{K}}\in\mathbb{R}^{L\times L}$, scale parameters $\mathbf{s}\in\mathbb{R}_{+}^{L}$
\STATE \textbf{Output:} gradient w.r.t. $\mathbf{s}$, i.e., $\frac{\partial \mathcal{L}}{\partial \mathbf{s}}\in\mathbb{R}^{L}$

\STATE construct the distance matrix $\mathbf{D}\in\mathbb{N}^{L\times L}$ with $D_{i,j}=|i-j|$

\FOR{$d=0,\dots,L-1$}
    \STATE compute the derivative
    \[
    \frac{\partial k_d}{\partial s_d}
    =
    e^{-s_d}
    \left(
    \frac{I_{d-1}(s_d)+I_{d+1}(s_d)}{2}
    -
    I_d(s_d)
    \right)
    \]

    \STATE aggregate gradients for distance $d$:
    \[
    g_d \gets
    \sum_{i,j:D_{i,j}=d}
    \frac{\partial \mathcal{L}}{\partial K_{i,j}}
    \]

    \STATE $\frac{\partial \mathcal{L}}{\partial s_d}
    \gets
    g_d\frac{\partial k_d}{\partial s_d}$
\ENDFOR

\STATE \textbf{return} $\frac{\partial \mathcal{L}}{\partial \mathbf{s}}$

\end{algorithmic}
\end{algorithm}

\section{Relation to Deformable Convolution or Pooling}

\method introduces adaptivity into multi-scale modeling by learning distance-aware scale parameters, rather than relying on fixed discrete scaling operators. Because \method learns adaptive scale parameters, it is useful to distinguish it from other adaptive operators such as deformable convolution~\citep{dai2017deformable}. While both approaches move beyond rigid uniform operators, their objectives are fundamentally different. Deformable convolution improves representational flexibility by adapting where information is sampled through learned input-dependent offsets, whereas \method constructs a valid continuous-scale operator family for multi-scale representation learning.

Rather than modifying the sampling geometry, \method controls how information is aggregated across temporal distances through a structured kernel. This distinction is important because the proposed operator satisfies principled properties, including non-expansiveness and energy reduction, thereby enabling  consistent multi-scale behavior. Thus, \method is best understood not as a deformable sampling mechanism, but as a structured continuous-scale operator for principled multi-scale time-series forecasting.

\end{document}